\documentclass[11pt]{article}
\usepackage[preprint]{jmlr2e}
\usepackage{lipsum,booktabs}
\usepackage{amsmath,mathrsfs,amssymb,amsfonts,bm,enumitem}
\usepackage{xcolor}
\usepackage[T1]{fontenc}
\allowdisplaybreaks
\usepackage{appendix}
\usepackage{mathtools}
\usepackage{multirow}
\usepackage{algorithmic,algorithm}
\usepackage{caption2}

\renewcommand{\tilde}{\widetilde}
\renewcommand{\hat}{\widehat}

\newcommand{\ie}{i.e., }
\newcommand{\wrt}{\textit{w}.\textit{r}.\textit{t}. }
\newcommand{\iid}{\textit{i}.\textit{i}.\textit{d}. }

\def \A {\mathcal{A}}
\def \B {\mathcal{B}}
\def \C {\mathtt{c}}

\def \E {\mathbb{E}}
\def \Ecal {\mathcal{E}}
\def \F {\mathcal{F}}

\def \H {\mathcal{H}}

\def \K {\mathcal{K}}

\def \O {\mathcal{O}}

\def \R {\mathbb{R}}

\def \T {\top}
\def \U {\mathcal{U}}

\def \X {\mathcal{X}}

\def \Ot {\tilde{\O}}
\def \Pb {\bar{P}}
\def \Vb {\bar{V}}
\def \Dp {D_{\psi}}
\def \Dpb {D_{\psib}}
\def \Dw {D_{\bar{\psi}}}
\def \psib {\bar{\psi}}

\def \qtc {q^{\pi_t^\C}}
\def \qk {q^{\pi_k}}
\def \qkc {q^{\pi_k^\C}}

\def \mp {m^\prime}

\def \sumi {\sum_{i=1}^N}
\def \suml {\sum_{l=0}^{H-1}}
\def \sumk {\sum_{k=1}^K}
\def \sumt {\sum_{t=1}^T}
\def \sumd {\sum_{i=1}^d}

\def \ph {\hat{p}}
\def \qh {\hat{q}}

\def \epsilon {\varepsilon}

\newcommand{\propers}{\ensuremath{\Pi_{\textnormal{proper}}}}
\newcommand{\Pip}{\propers}

\DeclareMathOperator*{\argmax}{arg\,max}
\DeclareMathOperator*{\argmin}{arg\,min}
\def \base {\mathtt{base}\mbox{-}\mathtt{regret}}
\def \meta {\mathtt{meta}\mbox{-}\mathtt{regret}}
\def \baseswitching {\mathtt{base}\mbox{-}\mathtt{switching}\mbox{-}\mathtt{cost}}
\def \metaswitching {\mathtt{meta}\mbox{-}\mathtt{switching}\mbox{-}\mathtt{cost}}
\newcommand{\term}[1]{\mathtt{term~(#1)}}

\let\proof\relax

\usepackage{amsthm}
\let\norm\undefined 
\DeclarePairedDelimiter\norm{\lVert}{\rVert}
\DeclarePairedDelimiter\abs{\lvert}{\rvert}
\DeclarePairedDelimiter\ceil{\lceil}{\rceil}
\newcommand\inner[2]{\langle #1, #2 \rangle}

\newtheorem{myThm}{Theorem}

\newtheorem{myLemma}{Lemma}

\newtheorem{myRemark}{Remark}
\newtheorem{myDef}{Definition}

\usepackage{graphicx,color}

\definecolor{wine_red}{RGB}{228,48,64}
\definecolor{DSgray}{cmyk}{0,1,0,0}

\usepackage{prettyref}
\newcommand{\pref}[1]{\prettyref{#1}}

\newcommand{\savehyperref}[2]{\texorpdfstring{\hyperref[#1]{#2}}{#2}}
\newrefformat{eq}{\savehyperref{#1}{Eq.~\textup{(\ref*{#1})}}}
\newrefformat{eqn}{\savehyperref{#1}{Equation~\ref*{#1}}}
\newrefformat{lem}{\savehyperref{#1}{Lemma~\ref*{#1}}}
\newrefformat{lemma}{\savehyperref{#1}{Lemma~\ref*{#1}}}
\newrefformat{def}{\savehyperref{#1}{Definition~\ref*{#1}}}
\newrefformat{line}{\savehyperref{#1}{Line~\ref*{#1}}}
\newrefformat{thm}{\savehyperref{#1}{Theorem~\ref*{#1}}}
\newrefformat{corr}{\savehyperref{#1}{Corollary~\ref*{#1}}}
\newrefformat{cor}{\savehyperref{#1}{Corollary~\ref*{#1}}}
\newrefformat{sec}{\savehyperref{#1}{{\color{black}Section}~\ref*{#1}}}
\newrefformat{app}{\savehyperref{#1}{{\color{black}Appendix}~\ref*{#1}}}
\newrefformat{appendix}{\savehyperref{#1}{{\color{black}Appendix}~\ref*{#1}}}
\newrefformat{assum}{\savehyperref{#1}{Assumption~\ref*{#1}}}
\newrefformat{ex}{\savehyperref{#1}{Example~\ref*{#1}}}
\newrefformat{fig}{\savehyperref{#1}{Figure~\ref*{#1}}}
\newrefformat{alg}{\savehyperref{#1}{Algorithm~\ref*{#1}}}
\newrefformat{rem}{\savehyperref{#1}{Remark~\ref*{#1}}}
\newrefformat{conj}{\savehyperref{#1}{Conjecture~\ref*{#1}}}
\newrefformat{prop}{\savehyperref{#1}{Proposition~\ref*{#1}}}
\newrefformat{proto}{\savehyperref{#1}{Protocol~\ref*{#1}}}
\newrefformat{prob}{\savehyperref{#1}{Problem~\ref*{#1}}}
\newrefformat{claim}{\savehyperref{#1}{Claim~\ref*{#1}}}
\newrefformat{que}{\savehyperref{#1}{Question~\ref*{#1}}}
\newrefformat{op}{\savehyperref{#1}{Open Problem~\ref*{#1}}}
\newrefformat{fn}{\savehyperref{#1}{Footnote~\ref*{#1}}}

\usepackage{hyperref,url,dsfont,nicefrac}
\usepackage{bbm}
\hypersetup{
    colorlinks,
    breaklinks,
    linkcolor = blue,
    citecolor = blue,
    urlcolor  = blue,
}
\def \A {A}
\def \X {X}
\def \mdp {MDPs~}
\newcommand{\ind}{\mathbbm{1}}
\def \piC {\pi^{\C}}

\DeclareMathOperator*{\Reg}{\textnormal{Regret}}

\DeclareMathOperator*{\DReg}{\textnormal{D-Regret}}

\usepackage{lastpage}
\jmlrheading{1}{2022}{1-\pageref{LastPage}}{4/00}{10/00}{xx}{Peng Zhao, Long-Fei Li, and Zhi-Hua Zhou}

\ShortHeadings{Dynamic Regret of Online Markov Decision Processes}{Zhao, Li, and Zhou}
\firstpageno{1}

\begin{document}

\title{Dynamic Regret of Online Markov Decision Processes}

\author{\name Peng Zhao \email zhaop@lamda.nju.edu.cn \\
    \name Long-Fei Li \email lilf@lamda.nju.edu.cn \\
    \name Zhi-Hua Zhou \email zhouzh@lamda.nju.edu.cn \\
    \addr National Key Laboratory for Novel Software Technology\\
    Nanjing University, Nanjing 210023, China}
\editor{Kevin Murphy and Bernhard Sch{\"o}lkopf}

\maketitle

\begin{abstract}
We investigate online Markov Decision Processes~(MDPs) with adversarially changing loss functions and known transitions. We choose \emph{dynamic regret} as the performance measure, defined as the performance difference between the learner and any sequence of feasible \emph{changing} policies. The measure is strictly stronger than the standard static regret that benchmarks the learner's performance with a fixed compared policy. We consider three foundational models of online MDPs, including episodic loop-free Stochastic Shortest Path (SSP), episodic SSP, and infinite-horizon MDPs. For these three models, we propose novel online ensemble algorithms and establish their dynamic regret guarantees respectively, in which the results for episodic (loop-free) SSP are provably minimax optimal in terms of time horizon and certain non-stationarity measure. Furthermore, when the online environments encountered by the learner are predictable, we design improved algorithms and achieve better dynamic regret bounds for the episodic (loop-free) SSP; and moreover, we demonstrate impossibility results for the infinite-horizon MDPs.
\end{abstract}


\section{Introduction}
\label{sec:introduction}
Markov Decision Processes (MDPs) are widely used to model decision-making problems, where a learner interacts with the environments sequentially and aims to improve the learned strategy over time. The MDPs model is very general and encompasses a variety of applications, including games~\citep{Nature:Go}, robotic control~\citep{ICML'15:TRPO}, autonomous driving~\citep{ICRA'19:Drive}, etc.

In this paper, we focus on the online MDPs framework with adversarially changing loss functions and known transitions, which has attracted increasing attention in recent years due to its generality~\citep{MatOR'09:online-MDP,NIPS'13:MDP-Neu,ICML'19:unknown-transition-Rosenberg,ICML'20:bandit-unknown-chijin,COLT'21:MDP-aggregate,COLT'21:SSP-minimax}. Let $T$ be the total time horizon. The general procedures of the online MDPs are as follows: at each round $t \in [T]$, the learner observes the current state $x_t$ and decides a policy $\pi_t : \X \times \A \rightarrow [0,1]$, where $\pi_t(a | x)$ is the probability of taking action $a \in \A$ at state $x \in \X$. Then, the learner draws and executes an action $a_t$ from $\pi_t(\cdot | x_t)$ and suffers a loss $\ell_t(x_t, a_t)$. The environments subsequently transit to the next state $x_{t+1}$ according to the transition kernel $P(\cdot | x_t, a_t)$. We focus on the full-information loss (reward) feedback setting where the entire loss function is revealed to the learner. The standard measure for online \mdp is the \emph{regret} defined as the performance difference between learner's policy and that of the best fixed policy in hindsight, namely,
\begin{equation}
    \label{eq:static-regret}
    {\Reg}_T = \sum_{t=1}^T \ell_t\left(x_t, \pi_t(x_t)\right) - \min_{\pi \in \Pi} \sum_{t=1}^T \ell_t\left(x_t, \pi(x_t)\right),
\end{equation}
where $\Pi$ is a certain policy class. There are many efforts devoted to optimizing the measure, yielding fruitful results~\citep{MatOR'09:online-MDP,colt'10:loop-free,aistats'12:unknown-neu,NIPS'13:MDP-Neu,IEEE'Control:bandit-MDP-Neu,ICML'19:unknown-transition-Rosenberg,IJCAI'21:SSP-Rosenberg,COLT'21:SSP-minimax}.
However, one caveat in the performance measure in~\pref{eq:static-regret} is that the measure only benchmarks the learner's performance with a \emph{fixed} strategy, so it is usually called the \emph{static regret} in the literature. The fact makes the static regret metric not suitable to guide the algorithm design for online decision making in open and non-stationary environments~\citep{NSR'22:Zhou-OpenML}, which is often the case in many real-world applications such as online recommendations and autonomous driving~\citep{grzywaczewski2017training,AAAI'19:virtualTaobao,KDD'18:dynamicRL,TKDE'21:DFOP}. In particular, in online MDPs model the loss functions encountered by the learner can be adversarially changing, it is thus unrealistic to assume the existence of a single fixed strategy in the policy class that can perform well over the horizon in such scenarios. To this end, in this paper we introduce the \emph{dynamic regret} as the performance measure to guide the algorithm design for online MDPs, which competes the learner's performance against a sequence of changing policies, defined as
\begin{equation}
    \label{eq:dynamic-regret}
    {\DReg}_T(\pi_{1:T}^{\C}) = \sum_{t=1}^T \ell_t\left(x_t, \pi_t(x_t)\right) - \sum_{t=1}^T \ell_t \left(x_t, \pi_t^{\C}(x_t) \right),
\end{equation}
where $\pi_1^{\C}, \ldots, \pi_T^{\C} \in \Pi$ is any sequence of compared policies in the policy class $\Pi$, which can be chosen with the complete foreknowledge of all the online loss functions. We use $\pi_{1:T}^{\C}$ as a shorthand of the compared policies. An upper bound of dynamic regret usually scales with a certain variation quantity of the compared policies denoted by $P_T(\pi_1^\C, \ldots, \pi_T^\C)$ that can reflect the non-stationarity of environments.

We note that the dynamic regret measure in~\pref{eq:dynamic-regret} is in fact very general due to the flexibility of compared policies. For example, it immediately recovers the standard regret notion defined in~\pref{eq:static-regret} when choosing the single best compared policy in hindsight, namely, choosing $\pi_{1:T}^{\C} = \pi^* \in \argmin_{\pi \in \Pi} \sum_{t=1}^T \ell_t(x_t, \pi(x_t))$. Hence, any dynamic regret upper bound directly implies a static regret upper bound by substituting a fixed compared policy. Another typical choice is setting the compared policies as the sequence of the best policy of each round, namely, choosing $\pi_t^{\C} = \pi_t^* \in \argmin_{\pi \in \Pi} \ell_t(x_t, \pi(x_t))$, and the resulting dynamic regret measure is sometimes referred to as the \emph{worst-case} dynamic regret in the literature~\citep{NIPS'18:Zhang-Ader}. It is noteworthy to emphasize that the dynamic regret measure in~\pref{eq:dynamic-regret} does not assume prior information of the compared policies, which is certainly also unknown to the online algorithms. As a result, the measure is also called \emph{universal} dynamic regret (or \emph{general} dynamic regret) in the sense that the regret bound holds for any feasible compared policies. Both static regret and the aforementioned worst-case dynamic regret are two special cases of the universal dynamic regret by configuring different choices of compared policies.

In this paper, focusing on the dynamic regret measure presented in~\pref{eq:dynamic-regret}, we investigate three foundational and well-studied models of online MDPs: (\romannumeral1) episodic loop-free Stochastic Shortest Path (SSP)~\citep{NIPS'13:MDP-Neu}, (\romannumeral2) episodic SSP~\citep{IJCAI'21:SSP-Rosenberg,COLT'21:SSP-minimax}, and (\romannumeral3) infinite-horizon MDPs~\citep{MatOR'09:online-MDP}. The first two SSP models belong to episodic MDPs, in which the learner interacts with environments in episodes and aims to reach a goal state with minimum total loss. The distinction lies in that the learner is guaranteed to reach the goal state within a fixed number of steps in the loop-free SSP model; by contrast, the horizon length in general SSP model depends on the learner's policies, which could potentially be infinite (if the goal is not reached). In infinite-horizon MDPs, there is no goal state and the horizon can be never end and the goal of the learner is to minimize the average loss over time. For all those three models, we propose novel online algorithms and provide the corresponding expected dynamic regret guarantees. We also establish several lower bound results and show that the obtained upper bounds for episodic loop-free SSP and general SSP are \emph{minimax optimal} in terms of time horizon and non-stationarity measure. Furthermore, when the online environments are not fully adversarial and have some patterns that are predictable, we develop optimistic variants for episodic (loop-free) SSP and prove that the enhanced algorithms enjoy problem-dependent dynamic regret bounds, which scale with the variation of online functions and thereby achieve better result than the minimax rate. We also demonstrate impossibility results on  attaining similar problem-dependent guarantees for infinite-horizon MDPs. Notably, all our algorithms are \emph{parameter-free} in the sense that they do not require knowing the non-stationarity quantity or the variation quantity of online functions ahead of time. Table~\ref{table:results} summarizes our main results.

\begin{table}[!t]
\centering
\caption{\small{Summary of our main results. For three models of online MDPs (episodic loop-free SSP, episodic SSP, and infinite-horizon MDPs), we establish their dynamic regret guarantees, and better rates can be achieved for the episodic (loop-free) SSP when the environments are predicable. Our obtained dynamic regret bounds immediately recover the best known static regret presented in the last column, when choosing a fixed compared policy and the non-stationarity measure $P_T$ or $\Pb_K$ then equals to zero. Besides, $V_K$ measures the variation of loss functions and reflects the predictability of environments. Note that all our results are achieved by \emph{parameter-free} algorithms in the sense that they do not require the knowledge of unknown quantities related to the environmental non-stationarity or adaptivity.}} \vspace{2mm}
\label{table:results}
    \renewcommand*{\arraystretch}{1.6}
    \resizebox{\textwidth}{!}{
        \begin{tabular}{c|r|r}
            \hline

            \hline
            MDP Model                               & \multicolumn{1}{c|}{Ours Result (dynamic regret)}                                  & \multicolumn{1}{c}{Previous Work (static regret)}                  \\ \hline
            \multirow{2}{*}{Episodic loop-free SSP} & $\Ot(H\sqrt{K (1 + P_T)})$ [\pref{thm:loop-free-non-stationarity}]                 & \multirow{2}{*}{$\Ot(H\sqrt{K  })$~\citep{NIPS'13:MDP-Neu}}        \\
                                                    & $\Ot(H\sqrt{V_K (1 + P_T)})$ [\pref{thm:loop-free-adaptivity}]                     &                                                                    \\ \hline
            \multirow{2}{*}{Episodic SSP}           & $\Ot(\sqrt{B_K(H_*+\Pb_K)}+\Pb_K)$  [\pref{thm:ssp}]                               & \multirow{2}{*}{$\Ot(\sqrt{H^{\pi^*}DK})$~\citep{COLT'21:SSP-minimax}} \\
                                                    & $\Ot(\sqrt{V_K(H_*+\Pb_K)}+\Pb_K)$ [\pref{thm:ssp-adaptivity}]                     &                                                                    \\ \hline
            Infinite-horizon MDPs                   & $\Ot(\sqrt{\tau T (1 + \tau P_T)} + \tau^2 P_T)$ [\pref{thm:total-regret-general}] & $\Ot(\sqrt{\tau T } )$~\citep{NIPS'13:MDP-Neu}                     \\
            \hline

            \hline
        \end{tabular}}
\end{table}

\paragraph{Technical contributions.} Similar to prior studies of non-stationary online learning~\citep{ICML'09:Hazan-adaptive,ICML'15:Daniely-adaptive,NIPS'18:Zhang-Ader,NeurIPS'20:Sword}, our proposed algorithms fall into the online ensemble framework with a meta-base two-layer structure. While the general framework is standard in modern online learning, several important new ingredients are required to achieve minimax and adaptive dynamic regret guarantees for online MDPs. We highlight the main technical challenges and contributions as follows.
\begin{itemize}
    \item For all three models, algorithms are performed over the ``occupancy measure'' space, so dynamic regret inevitably scales with the variation of occupancy measures induced by compared policies, making it necessary to establish relationships between the variation of occupancy measures and that of compared policies.
    \item Achieving minimax and adaptive dynamic regret bounds for episodic (non-loop-free) SSP is one of the most challenging parts of this paper due to the complicated structure of this model and also the requirement of handling dual uncertainties of unknown horizon length and unknown non-stationarity. This motivates a novel groupwise scheduling for base-learners and a new weighted entropy regularizer for the meta-algorithm. Additionally, appropriate correction terms in the feedback loss and carefully designed step sizes for both base-algorithm and meta-algorithm are also important.
    \item For learning in infinite-horizon MDPs, we reduce it to the problem of switching-cost penalized prediction with expert advice (or simply called switching-cost expert problem). We prove an impossibility result for problem-dependent (static/dynamic) regret of switching-cost expert problem, which might be of independent interest.
\end{itemize}

\paragraph{Notations.} We present several general notations used throughout the paper. We use $\ell \in \R_{[a,b]}^d$ to denote a vector whose each element satisfies $\ell_i \in [a,b]$ for $i \in [d]$. For a vector $a \in \R^d$, $a^2$ denotes the vector $(a_1^2, \ldots, a_d^2)^\T \in \R^d$. Besides, $e_{i} \in \R^d$ denotes the $i$-th standard basis vector. For a convex function $\psi$, its induced Bregman divergence is defined as $\Dp(u,w) = \psi(u)-\psi(w)-\inner{\nabla \psi(w)}{u-w}$. Given two policies $\pi$ and $\pi^\prime$, $\norm{\pi - \pi^\prime}_{1,\infty}=\max_x \norm{\pi(\cdot | x)-\pi^\prime(\cdot | x)}_1$. $\Ot(\cdot)$ omits the logarithmic factors on horizon length $T$.

\paragraph{Organization.} The rest of the paper is organized as follows. Section~\ref{sec:related-work} reviews the related work. In Section~\ref{sec:loop-free-SSP} and Section~\ref{sec:episodic-ssp}, we establish the minimax dynamic regret and present problem-dependent adaptive results for episodic loop-free and general (non-loop-free) SSP respectively. In Section~\ref{sec:infinite-horizon}, we provide dynamic regret upper bounds and impossibility results for the infinite-horizon online MDPs. Section~\ref{sec:experiment} presents the empirical studies. Section~\ref{sec:conclusion} concludes the paper and discusses the future work. We defer all the proofs to the appendices.
\section{Related Work}
\label{sec:related-work}
This section presents discussions on several topics related to this work. The first part is about the development of static regret for online adversarial MDPs, and the second part reviews related advance of dynamic regret minimization in non-stationary online learning.

\subsection{Online Adversarial MDPs}
Learning with adversarial MDPs has attracted much attention in recent years. We briefly discuss related works on three models of online MDPs studied in this paper, including episodic loop-free SSP, episodic (non-loop-free) SSP, and infinite-horizon MDPs.

\paragraph{Episodic loop-free SSP.} \citet{colt'10:loop-free} first study learning in the episodic SSP with a loop-free structure and known transition, where an $\Ot(H^2\sqrt{K})$ regret is achieved in the full information setting and $K$ is the number of the episodes and $H$ is the horizon length in each episode. Later~\citet{NIPS'13:MDP-Neu} propose the \mbox{O-REPS} algorithm which applies mirror descent over occupancy measure space and achieves the optimal regret of order $\Ot(H\sqrt{K})$. \citet{colt'10:loop-free, NIPS'13:MDP-Neu} also consider the bandit feedback setting. \citet{aistats'12:unknown-neu, ICML'19:unknown-transition-Rosenberg} investigate the unknown transition kernel and full-information setting. \citet{nips'19:bandit-unknown-rosenberg} and \citet{ICML'20:bandit-unknown-chijin} further consider the harder unknown transition kernel and bandit-feedback setting. The linear function approximation setting is also studied~\citep{icml'20:OPPO-cai}. Notably, our results for episodic loop-free SSP (see Section~\ref{sec:loop-free-SSP}) focus on known transition and full-information feedback setting. Different from all mentioned results minimizing static regret, our proposed algorithm is equipped with dynamic regret guarantee, which can recover the $\Ot(H \sqrt{K})$ minimax optimal static regret when choosing compared policies as the best fixed policy in hindsight. Furthermore, when the environments are predictable, we enhance the algorithm to capture such adaptivity and hence enjoy better dynamic regret guarantees than the minimax rate.

\paragraph{Episodic SSP.} \citet{IJCAI'21:SSP-Rosenberg} first consider learning in episodic (non-loop-free) SSP with full-information loss feedback. Their algorithm achieves an $\Ot(\frac{D}{c_{\min}}\sqrt{K})$ regret for the known transition setting, where $c_{\min} \in (0,1]$ is the lower bound of the loss function and $D$ is the diameter of the MDP. They also study the zero costs case and unknown transition setting. \citet{COLT'21:SSP-minimax} develop algorithms that significantly improve the results and achieve minimax regret $\Ot(\sqrt{H^{\pi^*} DK})$ for the full information with known transition setting, where $H^{\pi^*}$ is the hitting time of the optimal policy. They also investigate the unknown transition setting. Our results for episodic SSP (see Section~\ref{sec:episodic-ssp}) focus on the known transition and full-information setting. We develop an algorithm with optimal dynamic regret guarantees. Our result immediately recovers the optimal $\Ot(\sqrt{H^{\pi^*}D K})$ static regret when setting comparators as the best fixed policy in hindsight. We further enhance our algorithm to achieve a more adaptive bound when the environments are predictable.

\paragraph{Infinite-horizon MDPs.} \citet{MatOR'09:online-MDP} consider learning in unichain MDPs with known transition and full-information feedback, they propose the algorithm MDP-E that enjoys $\Ot( \sqrt{\tau^3 T})$ regret, where $\tau$ is the mixing time. Another work~\citep{MatOR'09:MDP-Yu} achieves $\Ot(T^{2/3})$ regret in a similar setting. The \mbox{O-REPS} algorithm of~\citet{NIPS'13:MDP-Neu} achieves an $\Ot(\sqrt{\tau T})$ regret. \citet{NIPS'10:bandit-MDP-Neu, IEEE'Control:bandit-MDP-Neu} consider the known transition kernel and bandit feedback setting. These studies focus on the MDPs with uniform mixing properties, which could be strong. Recent study tries to relax the assumption by considering the larger class of communicating MDPs~\citep{arXiv'21:communicating-MDP}. Our results for infinite-horizon MDPs (see Section~\ref{sec:infinite-horizon}) focus on the known transition and full-information feedback setting and propose an algorithm that enjoys dynamic regret which can recover the best-known $\Ot(\sqrt{\tau T})$ static regret.

\paragraph{Discussion.} We note that all those works focus on the static regret minimization, and our work establishes the dynamic regret for all the three online MDPs models. In a setting most similar to ours, \citet{NIPS'20:dynamic-fei} investigate the dynamic regret of episodic loop-free SSP (with function approximation). They propose two model-free algorithms and prove the dynamic regret bound scaling with non-stationarity of environments. However, we note that their algorithms require the prior knowledge of non-stationarity measure $P_T$ as input, which is generally unavailable to the learner in practice. By contrast, our proposed algorithms are \emph{parameter-free} to those unknown quantities related to the underlying environments (including non-stationarity measure $P_T$ and adaptivity quantity $V_T$). More importantly, we also consider dynamic regret of two more challenging settings of online MDPs — episodic (non-loop-free) SSP and infinite-horizon MDPs.

\subsection{Non-stationary Online Learning}
In this part, we first discuss related works of non-stationary MDPs (whose online loss functions are stochastic, whereas our paper studies the adversarial setting) and then discuss dynamic regret of online convex optimization whose techniques are related to us.

\paragraph{Online Non-stationary MDPs.} Another related line of research is on the online non-stationary MDPs. More specifically, in contrast to learning with adversarial MDPs where the online loss functions are generated in an adversarial way, online non-stationary MDPs consider the setting where reward (loss) functions are generated in a \emph{stochastic} way according to a certain reward distribution that might be non-stationary over the time. For infinite-horizon MDPs, \citet{JMLR'10:UCRL} consider the piecewise-stationary setting where the losses and transition kernels are allowed to change a fixed number and then propose UCRL2 with restarting mechanism to handle the non-stationarity. Later, \citet{arxiv:ns-rl-sliding-window} propose an alternative approach based on the sliding-window update for the same setting, and is later generalized to more general non-stationary setting with gradual drift~\citep{UAI'19:NS-RL}. However, all above approaches require the prior knowledge on the degree of non-stationarity, either the number of piecewise changes or the tensity of gradual drift. Recently, \citet{ICML'20:NSRL-Cheung} propose the Bandit-over-RL algorithm to remove the requirement of unknown non-stationarity measure, but nevertheless can only obtain suboptimal result. Other results for non-stationary MDPs includes episodic non-stationary MDPs~\citep{ICML'21:NS-Episodic-Mao,AISTATS'21:ns-rl-kernel} and episodic non-stationary linear MDPs~\citep{arxiv:efficient-touati,arxiv:ns-rl-zhou}. The techniques in those studies are related to the thread of stochastic linear bandits~\citep{colt'20:linear-mdp-jin, icml'20:kernel-yang, AISTATS'20:restart}. A recent breakthrough is made by~\citet{colt'21:black-box-wei}, who propose a black-box approach that can turn a certain algorithm with optimal static regret in a stationary environment into another algorithm with optimal dynamic regret in a non-stationary environment, and more importantly, the overall approach does not require any prior knowledge on the degree of non-stationarity. They achieve optimal dynamic regret for episodic tabular MDPs~\citep{ICML'21:NS-Episodic-Mao,arxiv:ns-rl-zhou,arxiv:efficient-touati}. For infinite-horizon MDPs, they can achieve optimal dynamic regret when the maximum diameter of  MDP is known or the degree of non-stationarity is known~\citep{arxiv:ns-rl-sliding-window, ICML'20:NSRL-Cheung}; when none of them is know, they attain suboptimal regret but is still the best-known result.

\paragraph{Non-stationary Online Convex Optimization.} Online convex optimization (OCO) is a fundamental and versatile framework for modeling online prediction problems~\citep{book'16:Hazan-OCO}. Dynamic regret of OCO has drawn increasing attention in recent years, and techniques are highly related to ours. We here briefly review some related results and  refer the reader to the latest paper~\citep{JMLR:sword++} for a more thorough treatment. Dynamic regret ensures the online learner to be competitive with a sequence of changing comparators, and is sometimes called tracking regret or switching regret in the study of prediction with expert advice setting~\citep{conf/nips/Cesa-BianchiGLS12}. As mentioned in Section~\ref{sec:introduction}, this paper focuses on the general dynamic regret that allows the any feasible comparators in the decision set, which is also called \emph{universal} dynamic regret. A special variant is called \emph{worst-case} dynamic regret, which only competes with the sequence of minimizers of online functions and has gained much attention in the literature~\citep{OR'15:dynamic-function-VT,AISTATS'15:dynamic-optimistic,ICML'16:Yang-smooth,ICML'16:GyorgyS-shiftregret,OR'19:V_T-pq,NIPS'19:Wangyuxiang,UAI'20:simple,L4DC'21:sc_smooth}. However, the worst-case dynamic regret would be problematic or even misleading in many cases, for example, approaching the minimizer of each-round online function would lead to overfitting when the environments admit some noise~\citep{NIPS'18:Zhang-Ader}. Thus, the universal dynamic regret is generally more desired to be performance measure for algorithm design in non-stationary online learning. We now introduce the results in this regard. \citet{ICML'03:zinkvich} first considers the universal dynamic regret of OCO and shows that Online Gradient Descent (OGD) enjoys $\O(\sqrt{T}(1+P_T))$ dynamic regret, where $P_T$ is the path length of the comparators reflecting the non-stationarity of the environments. Later,~\citet{NIPS'18:Zhang-Ader} propose a novel algorithm and prove a minimax optimal $\O(\sqrt{T(1+P_T)})$ dynamic regret guarantee without requiring the knowledge of unknown $P_T$. Their proposed algorithm employs the meta-base structure, which turns out to be a key component to handle unknown non-stationarity measure $P_T$. When the environments are predictable and the loss functions are convex and smooth, \citet{NeurIPS'20:Sword,JMLR:sword++} develop an algorithm, achieving problem-dependent dynamic regret which could be much smaller than the minimax rate. \citet{colt'21:exp-concave-baby,AISTATS'22:sc-proper} consider OCO with exp-concave or strongly convex loss functions. Dynamic regret of bandit online learning is studied for adversarial linear bandits~\citep{COLT'22:Corral++} and bandit convex optimization~\citep{JMLR'21:BCO}. More discussions can be found in the latest advance~\citep{JMLR:sword++}.
\section{Episodic Loop-free Stochastic Shortest Path}
\label{sec:loop-free-SSP}
This section presents our results for episodic loop-free SSP, a foundational and conceptually simple model of online MDPs. We first introduce the problem setup, then establish the minimax dynamic regret, and finally provide the adaptive results.

\subsection{Problem Setup}
\label{sec:loop-free-preliminary}
An episodic online MDP is specified by a tuple $M = (\X, g, \A, P, \{\ell_k\}_{k=1}^K)$, where $X$ and $A$ are the finite state and action spaces, $g \notin \X$ is the goal state, $P : X \times A \times X\cup\{g\} \rightarrow [0,1]$ is the transition kernel, $K$ is the number of episodes and $\ell_k \in \R_{[0,1]}^{\abs{X} \abs{A}}$ is the loss function in episode $k \in [K]$. An episodic loop-free SSP is an instance of episodic online MDPs and further satisfies the following conditions: state space $X\cup\{g\}$ can be decomposed into $H+1$ non-intersecting layers denoted by $X_0, \ldots, X_{H-1}, g$ such that $X_0 =\{x_0\}$ and $g$ are singletons, and transitions are only possible between the consecutive layers. Notice that the total horizon is $T=KH$.

The learning protocol of episodic loop-free SSP proceeds in $K$ episodes. In each episode $k \in [K]$, environments decide a loss $\ell_k: \X \times \A \rightarrow [0,1]$, and simultaneously the learner starts from state $x_0$ and moves forward across consecutive layers until reaching the goal state $g$. 
We focus on the full-information setting, namely, the loss is revealed to the learner after the episode ends. Notably, no statistical assumption is imposed on the loss sequence, which means the online loss functions can be chosen in an adversarial manner.

\paragraph{Occupancy measure.} Existing studies reveal the importance of the concept ``occupancy measure'' in handling online MDPs~\citep{NIPS'13:MDP-Neu, ICML'19:unknown-transition-Rosenberg}, which deeply connects the problem of online MDPs with online convex optimization. Given a policy $\pi$ and transition kernel $P$, the occupancy measure $q^{\pi} \in \R_{[0,1]}^{\abs{X} \abs{A}}$ is defined as the probability of visiting state-action pair $(x,a)$ by executing policy $\pi$, \ie
$ q^\pi(x,a) = \Pr \left[ x_{l(x)}=x, a_{l(x)}=a \mid P, \pi \right]$, where $l(x)$ is the index of the layer that $x$ belongs to. For an episode loop-free SSP instance $M$, its occupancy measure space is defined as $\Delta(M) = \{ q \mid q \in \R_{\geq 0}^{|X||A|} \text{ and } q \text{ satisfies constraints \textbf{(C1)} and \textbf{(C2)}}\}$, where the two constraints are described below. First, \textbf{(C1)} requires that for all layer $l = 0,\ldots, H-1$,
\[
    \sum_{x\in X_l}\sum_{a \in A} q(x,a) = 1,
\]
and second \textbf{(C2)} requires that for every $x \in X \setminus \{x_0\}$ the following equation holds:
\[
    \sum_{a \in A} q(x,a) = \sum_{x^\prime \in X_{l(x)-1}} \sum_{a^\prime \in A} P(x | x^\prime, a^\prime)q(x^\prime, a^\prime)
\]
For any occupancy measure $q \in \Delta(M)$, it induces a policy $\pi$ such that
\begin{equation}
    \label{eq:induce-policy}
    \pi(a |x) = \frac{q(x,a)}{\sum_{b\in\X} q(x,b)}
\end{equation}
holds for all $(x, a) \in \X \times \A$. Existing study shows that there exists a unique induced policy for all measures in $\Delta(M)$ and vice versa~\citep{NIPS'13:MDP-Neu}. Then, the expected loss of any policy $\pi$ at episode $k$ can be written as
\[
    \E\left[\sum_{l=0}^{H-1}\ell_k(x_l, a_l \mid P, \pi)\right] = \sum_{l=0}^{H-1} \sum_{x \in \X_l} \sum_{a \in \A} q^{\pi}(x,a) \ell_t(x,a) = \sum_{x \in \X} \sum_{a\in \A}q^{\pi}(x,a) \ell_t(x,a) = \inner{q^{\pi}}{\ell_k},
\]
where the expectation is taken over the randomness of the policy and transition kernel. Note the total horizon $T$ of episodic loop-free SSP can be divided into $K$ episodes, each with horizon length $H$, \ie $T = KH$. Denote by $\pi_{k,l}$ the policy at layer $l \in \{0\} \cup [H-1]$ in episode $k \in [K]$, the policy sequence $\pi_1, \ldots, \pi_T$ in~\pref{eq:static-regret} can be represented by $\pi_{1,0}, \ldots, \pi_{1, H-1}, \pi_{2, 0}, \ldots, \pi_{K, H-1}$. We use the notation $\pi_k$ as a shorthand of $\pi_{k, 0:H-1}$ for notational simplicity. Then we can rewrite the expected static regret in~\pref{eq:static-regret} as follows:
\begin{align}
    \E\left[{\Reg}_K\right] \triangleq {} & \E\left[\sum_{k=1}^K \sum_{l=0}^{H-1}\ell_k(x_l, \pi_{k,l}(x_l))\right] - \min_{\pi \in \Pi} \E\left[\sum_{k=1}^K  \sum_{l=0}^{H-1}\ell_k(x_l, \pi_l(x_l))\right] \label{eq:static-rewrite-loop-free} \\
    = {}                                  & \sum_{k=1}^K \inner{q^{\pi_k}}{\ell_k} - \min_{q\in \Delta(M)} \sumk \inner{q}{\ell_k}. \nonumber
\end{align}

\paragraph{Dynamic regret.}As discussed before, the static regret metric not suitable to guide the algorithm design in non-stationary environments. To this end, we focus on the expected \emph{dynamic regret} that competes the learner's performance against any sequence of changing policies $\pi_{1:K}^\C$, as defined in~\pref{eq:dynamic-regret}. Similar to the derivation in~\pref{eq:static-rewrite-loop-free}, we can also rewrite the expected dynamic regret into a form with respect to the occupancy measure:
\begin{align}
    \E\left[{\DReg}_K(\pi_{1:K}^{\C})\right] \triangleq {} & \E\left[\sum_{k=1}^K\sum_{l=0}^{H-1}\ell_k(x_l, \pi_{k,l}(x_l))\right] -  \E\left[\sum_{k=1}^K\sum_{l=0}^{H-1}\ell_k(x_l, \pi_{k,l}^\C(x_l))\right]\nonumber \\
    = {}                                                   & \sum_{k=1}^K \inner{\qk}{\ell_k} - \sum_{k=1}^K \inner{\qkc}{\ell_k},\label{eq:dynamic-regret-episodic-loop-free}
\end{align}
where $\qkc$ is the occupancy measure of the compared policy $\piC_k$ for all $k \in [K]$. The non-stationarity measure is naturally defined as $P_T = \sum_{k=2}^K \sum_{l=0}^{H-1} \norm{\pi_{k,l}^{\C}-\pi_{k-1,l}^{\C}}_{1, \infty}$.

\subsection{Minimax Dynamic Regret}
\label{sec:non-stationarity}
Before presenting our algorithm for dynamic regret of episodic loop-free SSP, we first briefly review the {O-REPS} algorithm of~\citet{NIPS'13:MDP-Neu} developed for minimizing the static regret. The key idea of {O-REPS} is to perform the online mirror descent over the occupancy measure space $\Delta(M)$, specifically, at episode $k+1$, the learner updates the prediction by
\begin{equation*}
    {q}_{k+1} = \argmin_{q\in \Delta(M)}{\eta}\inner{q}{\ell_k} + \Dp(q , {q}_{k}),
\end{equation*}
where $\eta > 0$ is the step size, $\psi(q) = \sum_{x,a} q(x,a) \log{q(x,a)}$ is the standard negative entropy, and $\Dp(\cdot,\cdot)$ is the induced Bregman divergence.~\citet{NIPS'13:MDP-Neu} prove that {\mbox{O-REPS}} enjoys an $\O(H\sqrt{K \log{(|X||A|)}})$ static regret.

By slightly modifying the algorithm, in following lemma we show \mbox{O-REPS} over a clipped occupancy measure space can achieve dynamic regret guarantees. Specifically, define the clipped space as $\Delta(M, \alpha) = \{q \mid q \in \Delta(M), \text{ and } q(x,a)\geq \alpha, \forall x, a\}$ with $0 < \alpha < 1$ being the clipping parameter, we prove that performing \mbox{O-REPS} over $\Delta(M, \alpha)$ ensures the following dynamic regret, whose proof can be found in Appendix~\ref{sec:loop-free-ssp-proof-base}.
\begin{myLemma}
    \label{lem:loop-free-ssp-1}
    Set $q_1 = \argmin_{q \in \Delta(M,\alpha)} \psi(q)$. For any compared policies $\piC_1,\ldots,\piC_K \in \{ \pi \mid q^\pi \in \Delta(M,\alpha) \}$, \mbox{O-REPS} over a clipped occupancy measure space $\Delta(M, \alpha)$ ensures
    $$
        \sum_{k=1}^K \inner{q_k - \qkc}{\ell_k}\leq \eta T + \frac{1}{\eta}\left({H\log{\frac{|X||A|}{H}}+ \Pb_T \log{\frac{1}{\alpha}}}\right),
    $$
    where $\Pb_T = \Pb_T(\piC_1,\ldots,\piC_K) = \sum_{k=2}^K \norm{\qkc - q^{\pi_{k-1}^\C}}_1$ is the path length of occupancy measures.
\end{myLemma}

To achieve a favorable dynamic regret, we need to set the step size $\eta$ optimally to balance time horizon $T$ and the path length of occupancy measures $\Pb_T$. However, we actually do not have prior knowledge of $\Pb_T$ even after the horizon ends since the compared policies can be arbitrarily chosen in the feasible set. Thus, we cannot apply the standard adaptive step size tuning techniques such as doubling trick~\citep{JACM'97:doubling-trick} or self-confident tuning~\citep{JCSS'02:Auer-self-confident} to remove the dependence on $\Pb_T$.
To address the issue, we employ a meta-base two-layer structure to handle the uncertainty~\citep{NIPS'18:Zhang-Ader,NeurIPS'20:Sword}. Specifically, we first construct a step size pool $\H = \{\eta_1, \cdots,\eta_N\}$ ($N$ is the number of candidate step sizes and is of order $\O(\log T)$ whose configuration will be specified later) to discretize value range of the optimal step size; and then initialize multiple base-learners simultaneously, denoted by $\B_1, \cdots, \B_N$, where $\B_i$ returns her prediction $q_{k,i}$ by performing \mbox{O-REPS} with step size $\eta_i \in \H$; finally a meta-algorithm is used to combine predictions of all base-learners and yield the final output $\{q_k\}_{k=1}^K$. Below, we specify the details.

\begin{algorithm}[!t]
    \caption{DO-REPS}
    \label{alg:loop-free}
    \begin{algorithmic}[1]
        \REQUIRE step size pool $\H = \{\eta_1, \ldots, \eta_N\}$, learning rate $\varepsilon$ and clipping parameter $\alpha$.
        \STATE Define: $\psi(q) = \sum_{x,a} q(x,a)\log q(x,a), \forall q \in \Delta(M, \alpha)$.
        \STATE Initialization: set $q_{1,i} = \argmin_{q\in \Delta(M, \alpha)}\psi(q)$ and $p_{1,i} = 1/N, \forall i \in [N]$.
        \FOR{$k=1$ to $K$}
        \STATE Receive $q_{k,i}$ from base-learner $\B_i$ for $i \in [N]$.
        \STATE Compute the occupancy measure $q_k = \sum_{i=1}^N p_{k,i} q_{k,i}$.
        \STATE Play the induce policy $\pi_k(a|x) = q_k(x,a)/\sum_{b \in \X}q_k(x,b), \forall x\in \X, a\in \A$.
        \STATE Suffer losses $\{\ell_k(x_0, a_0), \ldots \ell_k(x_{H-1}, a_{H-1})\}$ and observe loss function $\ell_k$.
        \STATE Update the weight by $p_{k+1,i} \propto \exp(-\varepsilon \sum_{s=1}^k h_{s,i})$ where $h_{k,i} = \inner{q_{k,i}}{\ell_k}, \forall i \in [N]$.
        \STATE Each base-learner $\B_i$ updates ${q}_{k+1,i} = \argmin_{q\in \Delta(M, \alpha)}{\eta_i}\inner{q}{\ell_k} + \Dp(q , {q}_{k,i})$.
        \ENDFOR
    \end{algorithmic}
\end{algorithm}

At episode $k\in [K]$, the learner receives the decision $q_{k,i}$ from each base-learner $\B_i, \forall i \in [N]$ and the weight vector $p_k \in \Delta_N$ from meta-algorithm. Then the learner outputs the decisions by $q_k = \sum_{i=1}^N p_{k,i} q_{k,i}$, plays the corresponding policy $\pi_k(a|x) \propto q_k(x,a), \forall x, a$, suffers loss $\{\ell_k(x_0, a_0), \ldots \ell_k(x_{H-1}, a_{H-1})\}$ where $\{(x_0, a_0), \ldots (x_{H-1}, a_{H-1})\}$ is the traversed trajectory and observes the loss function $\ell_k$.

After that, the base-algorithm updates by performing \mbox{O-REPS} over the clipped occupancy space $\Delta(M,\alpha)$ with a customized step size in pool $\H$. Concretely, for $i \in [N]$, denote by $\{q_{k,i}\}_{k=1}^K$ the occupancy measure sequence returned by the base-learner $\B_i$, and the base-learner $\B_i$ updates according to
\begin{equation*}
    {q}_{k+1,i} = \argmin_{q\in \Delta(M,\alpha)}{\eta_i}\inner{q}{\ell_k} + \Dp(q , {q}_{k,i}),
\end{equation*}
where $\eta_i \in \H$ is the step size associated with the base-learner $\B_i$.

The meta-algorithm aims to track the unknown best base-learner. We employ the Hedge algorithm~\citep{JCSS'97:boosting} that updates the weight $p_{k+1} \in \Delta_N$ by $p_{k+1,i} \propto \exp(-\varepsilon \sum_{s=1}^k h_{s,i})$ where $\varepsilon > 0$ is the learning rate of the meta-algorithm, $h_k \in \R^N$ evaluates the performance of the base-learners and is set as $h_{k,i} = \inner{q_{k,i}}{\ell_k}$ for $i \in [N]$.

Algorithm~\ref{alg:loop-free} summarizes our proposed Dynamic \mbox{O-REPS} (DO-REPS) algorithm. In the following, we present the dynamic regret guarantee for the proposed algorithm.
\begin{myThm}
    \label{thm:loop-free-non-stationarity}
    Set the step size pool $\H = \{\eta_i = 2^{i-1} \sqrt{K^{-1}\log(|X||A|/H)} \mid i\in [N] \}$, where $N=\ceil{\frac{1}{2} \log(1+\frac{4K\log{T}}{\log(|X||A|/H)})}+1$, the learning rate of meta-algorithm as $\varepsilon  = \sqrt{(\log{N})/(HT)}$ and the clipping parameter $\alpha=1/T^2$. \mbox{DO-REPS} (Algorithm~\ref{alg:loop-free}) satisfies
    \begin{align*}
        \E[{\DReg}_K(\pi_{1:K}^\C)] \leq \O\big(\sqrt{T(H\log{{|X||A|}} + \Pb_T \log{T})}\big) \leq \O\big(H\sqrt{K(\log{{|X||A|}} + P_T \log{T})}\big),
    \end{align*}
    where $\Pb_T = \Pb_T(\piC_1,\ldots,\piC_K) = \sum_{k=2}^K \norm{\qkc - q^{\pi_{k-1}^\C}}_1$ is the path length of occupancy measures and $P_T = \sum_{k=2}^K \sum_{l=0}^{H-1} \norm{\pi_{k,l}^{\C}-\pi_{k-1,l}^{\C}}_{1, \infty}$ is the path length of the compared policies.
\end{myThm}

\begin{myRemark}
    Setting compared policies $\pi_{1:K}^\C = \pi^*$ (then $P_T = 0$ and also $\bar{P}_T = 0$), Theorem~\ref{thm:loop-free-non-stationarity}  recovers the $\O(H\sqrt{K\log{{|X||A|}}})$ minimax optimal static regret of~\citet{NIPS'13:MDP-Neu}.
\end{myRemark}
The proof can be found in Appendix~\ref{sec:loop-free-ssp-proof-overall}. Note that Theorem~\ref{thm:loop-free-non-stationarity} presents two dynamic regret bounds in terms of either the path length of occupancy measures $\Pb_T$ or the path length of compared policies $P_T$ (see definition at the end of Section~\ref{sec:loop-free-preliminary}). To achieve the latter one, we establish the relationship of path length variations between compared policies and their induced occupancy measures. Indeed, we prove that $\Pb_T \leq H P_T$ holds in the episode loop-free SSP, and a formal description can be found in Lemma~\ref{lem:occupancy-to-policy-loop-free} of Appendix~\ref{sec:appendix-loop-free-SSP-occupancy-measure}.

We finally establish the lower bound in Theorem~\ref{Thm:lower-bound-loop-free}, which indicates the minimax optimality of our attained upper bound in terms of $T$ and $\bar{P}_T$ (up to logarithmic factors). The proof of Theorem~\ref{Thm:lower-bound-loop-free} can be found in Appendix~\ref{sec-appendix:lower-bound-loop-free}.
\begin{myThm}
    \label{Thm:lower-bound-loop-free}
    For any online algorithm and any $\gamma \in [0, 2T]$, there exists an episode loop-free SSP with $K$ episodes, $H$ layers, $|X|$ states and $|A|$ actions and a sequence of compared policies $\pi_1^\C, \ldots, \pi_K^\C$ such that
    $$
        \Pb_T(\piC_1,\ldots,\piC_K) \leq \gamma \mbox{ and } \E[{\DReg}_K(\pi_{1:K}^\C)] \geq  \Omega (\sqrt{T(H+\gamma)\log|X||A|})
    $$
    under the full-information and known transition setting.
\end{myThm}

\subsection{More Adaptive Results}
\label{sec:loop-free-adaptive}
In previous subsection, online loss functions are supposed to be chosen in a possibly adversarial manner. However, in certain applications, they might have some patterns and could be predictable. In such cases, there is a chance to enhance our algorithm to enjoy an adaptive bound better than the minimax rate. Thus, we propose the Optimistic \mbox{DO-REPS} algorithm that can exploit the predictability of environments to obtain more adaptive bounds.

\begin{algorithm}[t]
    \caption{Optimistic \mbox{DO-REPS}}
    \label{alg:ODO-REPS}
    \begin{algorithmic}[1]
        \REQUIRE size pool $\H = \{\eta_1, \ldots, \eta_N\}$, learning rate $\varepsilon$ and clipping parameter $\alpha$.
        \STATE Define: $\psi(q) = \sum_{x,a} q(x,a)\log q(x,a), \forall q \in \Delta(M, \alpha)$.
        \STATE Initialization: $\qh_{1,i} = \argmin_{q \in \Delta(M, \alpha)} \psi(q)$, and $p_{1,i} = 1/N, \forall i \in [N]$.
        \FOR{$k=1$ to $K$}
        \STATE Receive optimism $m_k \in \R_{[0,1]}^{|X||A|}$ and feed it to all base-learners.
        \STATE Each base-learner $\B_i$ updates $q_{k,i} = \argmin_{q \in \Delta(M, \alpha)}\eta_i \inner{q}{m_k}+\Dp(q, \qh_{k,i})$.
        \STATE Update the weight $p_k \in \Delta_N$ by $p_{k,i} \propto \exp\big( -\varepsilon (\sum_{s=1}^{k-1} h_{s,i} + M_{k}) \big)$, where $h_{k,i} = \inner{q_{k,i}}{\ell_k}$ and $M_{k,i} = \inner{q_{k,i}}{m_k}, \forall i \in [N]$.
        \STATE Compute the occupancy measure $q_k = \sum_{i=1}^N p_{k,i} q_{k,i}$.
        \STATE Play the induce policy $\pi_k(a|x) = q_k(x,a)/\sum_{b \in \X}q_k(x,b), \forall x\in \X, a\in \A$.
        \STATE Suffer losses $\{\ell_k(x_0, a_0), \ldots \ell_k(x_{H-1}, a_{H-1})\}$ and observe loss function $\ell_k$.
        \STATE Each base-learner $\B_i$ updates $\qh_{k+1,i} = \argmin_{q \in \Delta(M, \alpha)}\eta_i \inner{q}{\ell_k}+\Dp(q, \qh_{k,i})$.\\
        \ENDFOR
    \end{algorithmic}
\end{algorithm}

The Optimistic \mbox{DO-REPS} algorithm follows the meta-base two-layer structure similar to the \mbox{DO-REPS} algorithm proposed in the last subsection. We adopt the optimistic online learning framework~\citep{COLT'12:variation-Yang,COLT'13:Optimistic-OMD} to exploit the predictability of the environments. Specifically, let $m_k$ be the prior knowledge (or called \emph{optimism}) at the beginning of episode $k$, serving as a guess of the loss $\ell_k$. Optimistic \mbox{DO-REPS} maintains $N$ base-learners denoted by $\B_1,\ldots,\B_N$, where the base-learner $\B_i$ updates by
\begin{align}
    \label{eq:optimistic-loop-free-base}
    {q}_{k,i}       = \argmin_{q\in \Delta(M, \alpha)}\eta_i\inner{q}{m_{k}} + \Dp(q, \hat{q}_{k,i}), \mbox{ and }
    \hat{q}_{k+1,i} = \argmin_{q\in \Delta(M, \alpha)}\eta_i\inner{q}{\ell_k} + \Dp(q, \hat{q}_{k,i}).
\end{align}
Here $\eta_i $ is the associated step size from step size pool $\H$. The meta-algorithm also takes the optimism into account and  updates the weight vector $p_k \in \Delta_N$ by $p_{k,i} \propto \exp\big( -\varepsilon (\sum_{s=1}^{k-1} h_{s,i} + M_{k}) \big)$, where $h_{s,i} = \inner{q_{s,i}}{\ell_k}, \forall i \in [N]$ evaluates the performance of the base-learner $\B_i$, and $M_{k}$ is set as
$
    M_{k,i} = \inner{q_{k, i}}{m_{k}}, \forall i \in [N]
$, serving as the hint of the next-round $h_k$ for meta-algorithm's update. Algorithm~\ref{alg:ODO-REPS} summarizes the procedures, and the improved algorithm enjoys the following adaptive bound.
\begin{myThm}
    \label{thm:loop-free-adaptivity}
    Set the step size pool $\H = \{\eta_i = 2^{i-1} \sqrt{K^{-1}\log(|X||A|/H)} \mid  i\in[N] \}$, where $N=\ceil{\frac{1}{2} \log(K+\frac{4K^2\log{T}}{\log(|X||A|/H)})}+1$,  learning rate $\varepsilon  = \sqrt{(\log{N})/(H^2(1+V_K))}$ and the clipping parameter $\alpha=1/T^2$. Optimistic \mbox{DO-REPS} (Algorithm~\ref{alg:ODO-REPS}) satisfies
    \begin{align*}
        \E[{\DReg}_K(\pi_{1:K}^\C)] \leq \O\big(H\sqrt{V_K(\log{|X||A|} + P_T \log{T})}\big),
    \end{align*}
    where $V_K = \sum_{k=1}^K \norm{\ell_k-m_k}_\infty^2$ measures the quality of optimism. In particular, setting the optimism as the last-round loss (namely, $m_k = \ell_{k-1}$) yields the dynamic regret scaling with variations in online loss functions, i.e., $\sum_{k=1}^K \norm{\ell_k-\ell_{k-1}}_\infty^2$.
\end{myThm}
\begin{myRemark}
    Note that the meta-algorithm's learning rate depends on $V_K$, which can be easily removed by the standard self-confident tuning~\citep{JCSS'02:Auer-self-confident}. In addition, compared with the minimax result in Theorem~\ref{thm:loop-free-non-stationarity}, Theorem~\ref{thm:loop-free-adaptivity} exhibits more adaptivity in the sense that the upper bound depends on $V_K$ rather than $K$, which can be much tighter when environments are predictable (for example, online loss functions evolve gradually and we choose $m_k = \ell_{k-1}$) and at the same time safeguards the worst-case rate due to the fact $V_K \leq \O(K)$.
\end{myRemark}
\section{Episodic Stochastic Shortest Path}
\label{sec:episodic-ssp}
In this section, we consider the episodic SSP, which does not necessarily satisfy the loop-free structure and is thus more general and difficult than the loop-free SSP studied in Section~\ref{sec:loop-free-SSP}. For this model, we first introduce the formal problem setup and then establish minimax dynamic regret and finally provide adaptive results.

\subsection{Problem Setup}
\label{sec:general-SSP-problem}
An episodic SSP instance is defined by a tuple $M=(X, g, A, P, \{\ell_k\}_{k=1}^K)$, as the same as introduced in Section~\ref{sec:loop-free-preliminary}, $x_0 \in X$ is the initial state and $g \notin X$ is the goal state. The learning protocol proceeds in $K$ episodes. In each episode $k \in [K]$, environments decide a loss $\ell_k: \X \times \A \rightarrow [0,1]$, and simultaneously the learner starts from the initial state $x_0$ and moves to the next state until reaching the goal state $g$. Thus, the horizon in each episode depends on the learner's policy and is unfixed and can be even infinite, leading to inherent difficulties compared with episodic loop-free SSP. The learner aims to reach the goal state with a cumulative loss as small as possible. Again, we focus on the full-information setting, namely, the entire loss is revealed to the learner after the episode ends. Below we introduce several key concepts and we refer the reader to the work~\citep{COLT'21:SSP-minimax} for more details.

\paragraph{Proper policy.} A policy is called \emph{proper} if playing it ensures that the goal state is reached within a finite number of steps with probability $1$ starting from any state, otherwise it is called \emph{improper}. The set of all proper policies is denoted by $\Pip$. Following prior studies~\citep{IJCAI'21:SSP-Rosenberg,COLT'21:SSP-minimax}, we assume $\Pip \neq \emptyset$.

\paragraph{Hitting time.}
Denote by $H^{\pi}(x)$ the expected hitting time of $g$ when executing policy $\pi$ and starting from state $x$. If $\pi$ is proper, $H^{\pi}(x)$ is finite for any $x \in \X$. Let $H^\pi \triangleq H^{\pi}(x_0)$ be the hitting time of policy $\pi$ from the initial state $x_0$ to simplify notation. Another useful concept in SSP is the \textit{fast policy} $\pi^f$, defined as the (deterministic) policy that achieves the minimum expected hitting time starting from any state. The diameter of the SSP is defined as $D \triangleq \max_{x\in X} \min_{\pi\in \Pip} H^{\pi}(x)=\max_{x\in X} H^{\pi^f}(x)$. Note that both $\pi^f$ and $D$ can be computed ahead of time as the transition kernel is known~\citep{MathOR'91:SSP}.

\paragraph{Cost-to-go function.} Given a loss function $\ell$ and a policy $\pi$, the induced \textit{cost-to-go function} $J^{\pi}: X \rightarrow [0,\infty)$ is defined as $J^{\pi}(x)=\E [\sum_{i=1}^{I}\ell(x_i, a_i) \mid P, \pi]$, where $I$ denotes the number of steps before reaching $g$ of policy $\pi$ and the expectation is over the randomness of the stochastic policy and transition kernel. Denote by $J_k^\pi$ the cost-to-go function for policy $\pi$ with respect to loss $\ell_k$ from the initial state $x_0$.

\paragraph{Occupancy measure.}For the episodic SSP, the occupancy measure  $q^{\pi} \in \R^{|X||A|}$ is defined as the expected number of visits to $(x, a)$ from $x_0$ to $g$ when executing $\pi$, \ie $q^{\pi}(s, a) = \E[ \sum_{t=1}^I \ind\{x_t=x, a_t=a\} \mid P, \pi, x_1=x_0 ]$. Similar to the case in loop-free SSP, the induced policy of a given occupancy measure $q: \X \times \A \rightarrow [0, \infty)$ can be calculated by $\pi(a | x) \propto q(x, a), \forall x\in \X, a\in \A$. It holds that $H^\pi = \sum_{x, a} q^\pi(x,a)$. Based on the occupancy measure, we can rewrite the cost-to-go function $J_k^\pi$ as follows:
\[
    J_k^\pi = \E\bigg[ \sum_{i=1}^{I_k} \ell_k (x_i, a_i) \mid P, \pi_k \bigg] = \sum_{x,a} q^{\pi}(x,a) \ell_k(x,a) = \inner{q^{\pi}}{\ell_k},
\]
where $I_k$ denotes the number of steps before reaching $g$ of policy $\pi$ in episode $k$. Then the expected static regret in~\pref{eq:static-regret} for episodic SSP can be written as
\[
    \E\left[{\Reg}_K\right] \triangleq \E\bigg[\sum_{k=1}^K (J_k^{\pi_k} - J_k^{\pi^*})\bigg] = \E\bigg[\sum_{k=1}^K \inner{\qk - q^{\pi^*}}{\ell_k}\bigg],
\]
where $\pi^* = \argmin_{\pi \in \Pip} \sum_{k=1}^K J_k^\pi$. Two important quantities related to $\pi^*$ are commonly used in the analysis: (i) its hitting time $H^{\pi^*}$ from initial state $x_0$; and (ii) the cumulative loss $\sum_{k=1}^K J_k^{\pi^*}$ during $K$ episodes. The cumulative loss of the best policy is smaller than the fast policy, \ie $\sum_{k=1}^K J_k^{\pi^*} \leq \sum_{k=1}^K J_k^{\pi^f} \leq DK$, where the last inequality holds due to the definition of the fast policy and the boundedness of the loss range in $[0,1]$.

\paragraph{Dynamic regret.} To handle non-stationary environments, we employ the dynamic regret as the performance measure to compete against a sequence of changing policies $\pi_1^\C, \ldots, \pi_K^\C$, as defined in~\pref{eq:dynamic-regret}. Similar to the derivation in the episodic loop-free SSP, we can also rewrite the dynamic regret of episodic (general) SSP in terms of the occupancy measure as
\[
    \E[{\DReg}_K(\pi_{1:K}^\C)] \triangleq \E\bigg[\sumk (J_k^{\pi_k} - J_k^{\pi_k^\C})\bigg] = \E\bigg[\sum_{k=1}^K \inner{\qk - \qkc}{\ell_k}\bigg].
\]

Similarly, we generalize the two crucial quantities to accommodate changing comparators: the largest hitting time starting from the initial state $H_* = \max_{k \in [K]} H^{\pi_k^\C}$ and the cumulative loss of compared policies $B_K = \sumk J_k^{\pi_k^{\C}}=\sumk \inner{\qkc}{\ell_k}$. It is clear that $B_K \leq H_*K$. Notably, the sequence of compared policies can be arbitrarily chosen in the feasible set $\Pi$, so it is \emph{unknown} to the learner even at the end of episodes. Consequently, both quantities $H_*$ and $B_K$ are unknown to the learner. In addition, we remark that the inequality of $\sum_{k=1}^K J_k^{\piC_k} \leq \sum_{k=1}^K J_k^{\pi^f}$ is \emph{not} necessarily true due to the different possibility of compared polices, in stark contrast to the analysis in the static regret that competes with the fixed optimal policy in hindsight (see the counterpart inequality at the end of last paragraph). For the episodic (non-loop-free) SSP, the non-stationarity measure is naturally defined as $P_K = \sum_{k=2}^K \norm{\pi_k^\C - \pi_{k-1}^\C}_{1, \infty}$.

\subsection{Minimax Dynamic Regret}
\label{sec:dynamic-ssp}
Before introducing our approach, we first review existing works studying static regret and then illustrate that several crucial ingredients are required to achieve dynamic regret.

To resolve episodic (non-loop-free) SSP, \citet{IJCAI'21:SSP-Rosenberg} propose to deploy Online Mirror Descent (OMD) over the \emph{parametrized} occupancy measure space. For an MDP instance $M$ and a given horizon length $H$, the parameterized space is defined as $\Delta(M, H) = \{ q \in \R_{\geq 0}^{|X||A|} \mid \sum_{x,a}q(x,a) \leq H \text{ and } \sum_a q(x,a) = \sum_{x^\prime, a^\prime} P(x | x^\prime, a^\prime)q(x^\prime, a^\prime), \forall x \in X\}$. The authors prove that OMD enjoys an $\Ot(H \sqrt{K})$ static regret as long as $q^{\pi^*} \in \Delta(M,H)$. Therefore, if the largest hitting time $H^{\pi^*}$ were known ahead of time, a simple choice of $H = H^{\pi^*}$ would attain the favorable  static regret. However, such information is in fact unavailable in advance, which motivates a two-layer approach deal with this uncertainty.

Specifically,~\citet{COLT'21:SSP-minimax} maintain multiple base-learners $\B_1, \ldots, \B_N$, where $\B_i$ works with an occupancy measure space $\Delta(M,H_i)$ and a step size $\eta_i$ and returns her individual occupancy measure $q_k^i$; and then a certain meta-algorithm is employed to combine predictions of base-learners to produce final decisions $q_k$. Let $\B_{i^*}$ be the base-learner whose space size $H_{i^*}$ well approximates the unknown $H^{\pi^*}$. Denote by $L_K = \sum_{k=1}^K \inner{q_k}{\ell_k}, L_K^{i^*} = \sum_{k=1}^K \inner{q_k^{i^*}}{\ell_k}, L_K^\C=\sum_{k=1}^K \inner{q^{\piC_k}}{\ell_k}$\footnote{Here we define $L_K^c$ in a general way to accommodate changing comparators, which will be later used in the explanation of dynamic regret analysis. For this static regret statement, it becomes $L_K^\C = \sumk \inner{q^{\pi^*}}{\ell_k}$.} the cumulative loss of final decisions, base-learner $\B_{i^*}$ and the compared policy, respectively. Then, the overall regret can be decomposed as
\begin{equation}
    \label{eq:decompose-ssp}
    \E[{\Reg}_K] =\E\bigg[\sumk \inner{q_k - \qkc}{\ell_k}\bigg] = \E\big[(L_K - L_K^{i^*})\big] + \E\big[(L_K^{i^*} - L_K^\C)\big],
\end{equation}
where the two terms are called \emph{meta-regret} (that captures the regret overhead due to the two-layer ensemble) and \emph{base-regret} (that measures the regret of the unknown best base-learner). To achieve a favorable regret, they propose two mechanisms to control base-regret and meta-regret respectively. First, they pick the base-algorithm with an $\Ot(H_{i^*}/\eta_{i^*} + \eta_{i^*} L_K^\C)$ \emph{small-loss} static regret, which ensures an $\Ot(\sqrt{H^{\pi^*}DK})$ base-regret by setting $\eta_{i} = \O(\sqrt{H_i/DK})$ as the cumulative loss of the best policy in hindsight satisfies $L_K^{\C} \leq D K$. Second, they design a small-loss type \emph{multi-scale} online algorithm (roughly, OMD with weighted entropy $\psib(p)= \sum_{i=1}^N \frac{1}{\varepsilon_i} p_i\log{p_i}$) as the meta-algorithm to make meta-regret adaptive to the individual loss range of experts, so that meta-regret is at most $\Ot(1/\varepsilon_{i^*} + \varepsilon_{i^*} H_{i^*} L_K^{i^*})$. Combining the base-regret we further have $L_K^{i^*} \leq L_K^\C + \Ot(\sqrt{H^{\pi^*}DK}) \leq DK + \Ot(\sqrt{H^{\pi^*}DK}) = \Ot(DK)$ as $H^{\pi^*} \leq DK$. So an $\Ot(\sqrt{H^{\pi^*}DK})$ meta-regret is achievable by setting $\varepsilon_i = \Ot(1/\sqrt{H_{i}DK})$, which in conjunction with the base-regret yields an $\Ot(\sqrt{H^{\pi^*}DK})$ static regret.

However, it becomes more involved for dynamic regret. First of all, in addition to the uncertainty of unknown horizon length $H_*$, the base level also needs to deal with the unknown environmental non-stationarity $P_K$. Conceptually, this can be handled by maintaining more base-learners, which will be specified later. Second and more importantly, it is  challenging to design a compatible meta-algorithm. To see this, suppose we already have an $\Ot(\sqrt{B_K(P_K+H_*)})$ \emph{small-loss dynamic regret} for the base-algorithm, where $B_K = \sumk J_k^{\pi_k^{\C}}$ is the cumulative loss of compared policies, we then continue the above recipe and see the issue in meta-regret. Indeed, the meta-regret is at most $\Ot(1/\varepsilon_{i^*} + \varepsilon_{i^*} H_{i^*} L_K^{i^*})$, and by the base-regret bound we have $L_K^{i^*} \leq L_K^\C + \base \leq B_K + \Ot(\sqrt{B_K(P_K+H_*)})$. The natural upper bound of $B_K$ depends on $H_*$ (recall that $B_K \leq H_* K$) due to the arbitrary choice of compared policies. An important technical caveat is that as mentioned earlier we cannot simply assume the cost-to-go functions of  the compared policies $\{J_k^{\piC_k}\}_{1,\ldots,K}$ are bounded by that of fast policy $J_k^{\pi^f}$, in contrast to the static regret analysis where we have $\sum_{k=1}^K J_k^{\pi^*} \leq \sum_{k=1}^K J_k^{\pi^f}$ due to the optimality of the compared offline policy. Hence, even with a multi-scale meta-algorithm, meta-regret will be $\Ot(H_*\sqrt{K})$ and become the dominating term, making the final dynamic regret linear in $H_*$ and thus suboptimal.

To address above issues in both base and meta levels, building upon the structure of~\citet{COLT'21:SSP-minimax}, we propose a novel two-layer approach to deal with the dual uncertainties of unknown horizon length and unknown non-stationarity. Specifically, we introduce three crucial ingredients: \emph{groupwise scheduling} for base-learners, injecting \emph{corrections} in feedback loss of both base and meta levels, and a new \emph{multi-scale} meta-algorithm. Below, we first describe the base-algorithm, then introduce the scheduling method that instantiates a bunch of base-learners with different parameter configurations, and finally design the meta-algorithm that adaptively combines all the base-learners.

\paragraph{Base-algorithm.} The base-algorithm performs OMD over a clipped occupancy measure space. At each episode $k \in [K]$, the base-algorithm receives the online loss $\ell_k$ and performs
\begin{equation}
    \label{eq:base-learner}
    q_{k+1}  = \argmin_{q\in\Delta(M,H, \alpha)} \eta \inner{q}{\ell_k + a_k} + \Dp(q, q_k),
\end{equation}
where $\eta > 0$ is the step size, $\Delta(M, H, \alpha)=\{q\in \Delta(M,H) \mid q(x,a)\geq \alpha, \forall x,a\}$ is the clipped space with $\alpha \in (0,1)$, $\psi$ is the standard negative-entropy regularizer. Notably, we inject a \emph{correction term} $a_k \in \R^{|X||A|}$ to the loss, set as $a_k=32\eta \ell_k^2, \forall k \in [K]$. The purpose is to ensure a small-loss dynamic regret and simultaneously introduce an additional \emph{negative term} that will be crucial to address the difficulty occurred in controlling meta-regret (as mentioned earlier). The base-algorithm enjoys the following dynamic regret.
\begin{myLemma}
    \label{lem:CDO-base}
    Set $q_1 = \argmin_{q \in \Delta(M, H, \alpha)} \psi(q)$. Suppose $\eta \leq \frac{1}{64}$, for any compared policies $\piC_1,\ldots,\piC_K \in \{ \pi \mid q^\pi \in \Delta(M, H, \alpha) \}$, the base-algorithm in~\pref{eq:base-learner} ensures
    \[
        \sumk\inner{q_k - \qkc}{\ell_k} \leq \frac{1}{\eta}\Big(\Pb_K\log{\frac{H}{\alpha}} + H\big(1+\log(|X||A|H)\big) \Big)  + 32\eta B_K - 16\eta \sumk \inner{q_k}{\ell_{k}^2},
    \]
    where $\Pb_K = \Pb_K(\piC_1,\ldots,\piC_K) = \sum_{k=2}^K \norm{\qkc - q^{\pi_{k-1}^\C}}_1$ is the path length of occupancy measure.
\end{myLemma}

\paragraph{Scheduling.} Lemma~\ref{lem:CDO-base} indicates that given a horizon length $H$, it is crucial to set step size properly to achieve tight dynamic regret. Since $H$ affects the base-learner's feasible domain (i.e., the parametrized occupancy measure space), we propose a \emph{groupwise scheduling} scheme to simultaneously adapt to unknown non-stationarity $\Pb_K$ and horizon length $H_*$. Specifically, due to $H^{\pi^f} \leq H_* \leq K$, we first construct a horizon length pool $\mathcal{H} = \{H_i = 2^{i-1} \cdot H^{\pi^f} \mid i \in [G]\}$ where $G = 1 + \ceil{\log((K+1)/H^{\pi^f})}$ to exponentially discretize the possible range; and for each $H_i$ in the pool, we further design a step size grid $\Ecal_i = \{\eta_{i,j} = {1}/({32 \cdot 2^j}) \mid j \in [N_i] \}$ where {$N_i = \ceil{\frac{1}{2}\log{(\frac{4 K}{1+\log{(|X||A|H_i)}})}}$} to search the optimal optimal step size associated with $H_i$. Overall, we maintain $N = \sum_{i=1}^G N_i$ base-learners, each of which associates with a specific space size and step size. More precisely, let $\B_{i,1:N_i}$ be a shorthand of the $i$-th group of base-learners $\B_{i,1},\ldots,\B_{i,N_i}$, in which they use the same space size $H_i$ yet different step sizes (see the configuration of $\Ecal_i$). Thus, the set of all base-learners can be denoted as $\{\B_{1,1:N_1},\ldots,\B_{i,1:N_i},\ldots,\B_{G,1:N_G}\}$. The decision of the base-learner $\B_{i,j}$ in episode $k$ is denoted by $q_{k}^{i,j}$, with $i \in [G]$ and $j \in [N_i]$.

\begin{algorithm}[!t]
    \caption{CODO-REPS}
    \label{alg:CODO-REPS}
    \begin{algorithmic}[1]
        \REQUIRE horizon length pool $\mathcal{H} = \{H_1, \ldots, H_G\}$, step size grid $\Ecal_{i}=\{\varepsilon_{i,1}, \ldots, \varepsilon_{i,N_i}\}, \forall i \in [G]$ and clipping parameter $\alpha$.
        \STATE Define $\psi(q) = \sum_{x,a} q(x,a) \log{q(x,a)}$ and $\psib(p)$ as in~\pref{eq:regularizer}.
        \STATE Initialize $q_1^{i,j}=\argmin_{q \in \Delta(M, H, \alpha)} \psi(q)$, $p_1^{i,j} \propto \varepsilon_{i,j}^2, \forall i \in [G], j\in [N_i]$.\\
        \FOR{$k=1,\ldots,K$}
        \STATE Receive $q_k^{i,j}$ from base-learner $\B_{i,j}, \forall i \in [G], j \in [N_i]$.
        \STATE Sample $(i_k,j_k) \sim p_k$, play the induced policy $\pi_k(a | x) \propto q^{i_k,j_k}_k(x, a), \forall x,a$.
        \STATE Suffer losses $\{\ell_k(x_1,a_1), \ldots, \ell_k(x_{I^k}, a_{I^k})\}$, receive $\ell_k$, and feed it to all base-learners.
        \STATE Define $h_k^{i,j} = \inner{q^{i,j}_k}{\ell_k}, b_k^{i,j} = 32\varepsilon_{i,j}(h_k^{i,j})^2, a_k^{i,j}=32\eta_{i,j} \ell_k^2, \forall i \in [G], j\in [N_i]$.
        \STATE Each base-learner $\B_{i,j}$ updates $q_{k+1}^{i,j}  = \argmin_{q\in\Delta(M,H, \alpha)} \eta_{i,j} \inner{q}{\ell_k + a_k^{i,j}} + \Dp(q, q_k^{i,j})$.\\
        \STATE Update weight by $p_{k+1} = \argmin_{p\in\Delta_N}\inner{p}{h_k+b_k} + \Dw(p, p_k)$.
        \ENDFOR
    \end{algorithmic}
\end{algorithm}

\paragraph{Meta-algorithm.} The meta-algorithm requires a careful design to achieve a favorable regret. We propose a new meta-algorithm under the standard OMD framework, where additional designs are required including a novel weighted entropy regularizer and an appropriate correction term. Specifically, the meta-algorithm updates the weight vector $p_{k+1} \in \Delta_N$ by
\begin{equation}
    \label{eq:meta-algorithm}
    p_{k+1} = \argmin_{p\in\Delta_N}\inner{p}{h_k+b_k} + \Dw(p, p_k),
\end{equation}
where $h_k \in \R^N$ is the loss of meta-algorithm, defined as $h_k^{i,j} = \inner{q^{i,j}_k}{\ell_k}, \forall i \in [G], j \in [N_i]$. Moreover, there are two important features in the design: (i) an injected correction term $b_k \in \R^N$; and (ii) a weighted entropy regularizer $\psib(p) = \sum_{i=1}^N \frac{1}{\epsilon_i} p_i \log p_i$ to realize the multi-scale online learning, where $\epsilon_i >0$ is a multi-scale learning rate for $i \in [N]$. Below we specify the details and explain the motivation behind such designs.

First, in the meta level we inject a correction term $b_k \in \R^N$ set as
\begin{equation}
    \label{eq:meta-correction}
    b_k^{i,j} = 32\varepsilon_{i,j}(h_k^{i,j})^2, ~~~\forall i\in [G], j\in[N_i].
\end{equation}
Let $\B_{i^*, j^*}$ be the base-learner whose space size $H_{i^*}$ well approximates the unknown $H_*$ and step size $\eta_{i^*, j^*}$ well approximates the unknown optimal step size. Although injecting a correction term for the meta-algorithm was also used in~\citep{COLT'21:SSP-minimax} to ensure a small-loss type meta-regret of the form $\Ot(1/\varepsilon_{i^*, j^*} + \varepsilon_{i^*, j^*}H_{i^*} L_K^{i^*, j^*})$, as aforementioned, this will not lead to an optimal meta-regret in our case due to the undesired upper bound of $L_K^{i^*, j^*}$. Asides from that, our key novelty is to \emph{simultaneously} exploit the correction term in the base level, which contributes to an additional negative term in the base-regret $\Ot((\bar{P}_K+H_{i^*})/\eta_{i^*, j^*} + \eta_{i^*, j^*} B_K - \eta_{i^*, j^*} \sumk \inner{q_k^{i^*, j^*}}{\ell_{k}^2})$. By a careful design of step size $\eta_{i,j}$ and learning rate $\varepsilon_{i,j}$, we can successfully cancel the positive term $\varepsilon_{i^*, j^*}H_{i^*} L_K^{i^*, j^*}$ in the meta-regret by the negative term in the base-regret.

Second, it is known that OMD with a weighted entropy regularizer leads to a multi-scale expert-tracking algorithm~\citep{JMLR'19:multi-scale}. In our case, we set the weighted entropy regularizer $\psib: \Delta_N \rightarrow \R$ as
\begin{equation}
    \label{eq:regularizer}
    \psib(p) = \sum_{i =1}^G \sum_{j=1}^{N_i} \frac{1}{\varepsilon_{i,j}}p_{i,j}\log p_{i,j}, \mbox{ with }\varepsilon_{i,j} = \frac{\eta_{i,j}}{2H_i}.
\end{equation}
In above, $\eta_{i,j}$ is the step size employed by the base-learner $\B_{i,j}$ as specified earlier. Note that the weighted entropy regularizer depends on both space size and step size such that the final meta-algorithm can successfully handle the groupwise scheduling over the base-learners.

Combining all above ingredients yields our COrrected \mbox{DO-REPS} (CODO-REPS) algorithm, as summarized in Algorithm~\ref{alg:CODO-REPS}. We have the following dynamic regret guarantee.

\begin{myThm}
    \label{thm:ssp}
    Set the horizon length pool $\mathcal{H} = \{H_i = 2^{i-1} \cdot H^{\pi^f} \mid i \in [G]\}$ with $G = \ceil{\log((K+1) / H^{\pi^f})}$, the step size grid $\Ecal_i = \{\eta_{i,j} = {1}/({32 \cdot 2^j}) \mid j \in [N_i] \}$ with {$N_i = \ceil{\frac{1}{2}\log{(\frac{4 K}{1+\log{(|X||A|H_i)}})}}$}, and the clipping parameter $\alpha=1/K^3$. \mbox{CODO-REPS} (Algorithm~\ref{alg:CODO-REPS}) enjoys the following dynamic regret guarantee,
    \begin{align*}
        \E[{\DReg}_K(\pi_{1:K}^\C)] \leq \Ot\Big(\sqrt{(H_{*}+ \Pb_{K})(H_{*}+\Pb_K+B_K )}\Big).
    \end{align*}
\end{myThm}
\begin{myRemark}
    Setting compared policies $\pi_{1:K}^\C = \pi^*$ (then $P_T = 0$ and $B_K = \sumk J_k^{\pi^*}$), Theorem~\ref{thm:ssp} implies an $\Ot(\sqrt{H_* B_K})$ static regret, which gives a small-loss type bound for the episodic SSP and is new to the literature to the best of our knowledge. The bound is no worse the minimax rate $\Ot(\sqrt{H_*DK})$ of~\citet{COLT'21:SSP-minimax} as $B_K = \sumk J_k^{\pi^*}\leq DK$ in the static case, and can be much better than theirs when best policy behaves well.
\end{myRemark}
Below we show that the result in Theorem~\ref{thm:ssp} is actually minimax in terms of $B_K$ and $\Pb_K$ up to logarithmic factors.
\begin{myThm}
    \label{thm:lower-bound-ssp}
    For any online algorithm and any $\gamma \in [0, 2T]$, there exists an episodic SSP instance with diameter $D$ and a sequence of compared policies $\pi_1^\C, \ldots, \pi_K^\C$ with the largest hitting time $H_*$ such that
    $$
        \Pb_K \leq \gamma \mbox{ and }  \E[{\DReg}_K(\pi_{1:K}^\C)] \geq \Omega(\sqrt{DH_* K(1+\gamma/H_*)})
    $$
    under the full-information and known transition setting.
\end{myThm}
We finally remark that the upper bound in Theorem~\ref{thm:ssp} depends on the path length of occupancy measures $\Pb_K$ rather than the path length of compared policies $P_K$. A natural question is how to upper bound $\Pb_K$ by $P_K$ (up to multiplicative dependence on $H_*$). However, we show that this is generally impossible for the episodic (non-loop-free) SSP as stated in Theorem~\ref{thm:occupancy-to-policy-ssp} of Appendix~\ref{sec:appendix-impossible-SSP}.

\subsection{More Adaptive Results}
\label{sec:ssp-adaptive}
In this part, we design more adaptive algorithm for episodic SSP to achieve better guarantees in predictable environments.

Similar to the problem setup for episodic loop-free SSPs in Section~\ref{sec:loop-free-adaptive}, the online learner receives an optimism $m_k$ at the beginning of episode $k \in [K]$ as the additional prior knowledge of loss $\ell_k$. A natural motivation is to design algorithms with regret scaling with an adaptive quantity such as $V_K = \sum_{k=1}^K \norm{\ell_k - m_k}_{\infty}^2$. However, we point it out that this quantity might be not suitable for the episodic (non-loop-free) SSP due to the complicated structure of the model. Technically, even for the static regret, incorporating the optimistic online learning into Algorithm~\ref{alg:CODO-REPS} can only attain an $\Ot(H^{\pi^*}\sqrt{V_K})$ adaptive bound, which leads to an $\Ot(H^{\pi^*}\sqrt{K})$ \emph{suboptimal} bound when $m_k = 0, \forall k \in [K]$; recall that the minimax (near-)optimal regret is of order $\Ot(\sqrt{H^{\pi^*}DK})$, and $H^{\pi^*}$ can be much larger than $D$.

To this end, we introduce a novel problem-dependent quantity defined in the following way to measure the quality of optimism $m_1, \ldots, m_K$,
\begin{equation}
    \label{eq:adaptive-quantity}
    V_K = \min\Bigg\{\sum_{k=1}^K {\inner{\qkc}{\ell_k^2}}, \sum_{k=1}^K {\inner{\qkc}{(\ell_k-m_k)^2}}\Bigg\}.
\end{equation}
We illustrate the advantages of this quantity. On one hand, it safeguards the small-loss behavior as $V_K \leq \sum_{k=1}^K {\inner{\qkc}{\ell_k^2}} \leq \sum_{k=1}^K {\inner{\qkc}{\ell_k}} = B_K$ holds for any optimism, which is crucial in establishing the minimax bound for episodic SSP as presented in Section~\ref{sec:dynamic-ssp}. On the other hand, the quantity $V_K$ can be much smaller if the optimism sequence $m_1, \ldots, m_K$ is of high quality, for example $V_K = 0$ when $m_k$ predicts $\ell_k$ perfectly at each episode.

To establish dynamic regret scaling with the desired quantity in~\pref{eq:adaptive-quantity}, we substantially modify the \mbox{CODO-REPS} algorithm to leverage the predictability of the environments. First, both base-algorithm and meta-algorithm now use optimistic OMD to incorporate the optimism, which also requires modifying the injected correction terms correspondingly. Second, the quantity $V_K$ in fact depends on two optimistic sequences, \ie $\{0\}_{k\in K}$ and $\{m_k\}_{k\in K}$. To handle multiple optimistic sequences, a natural idea is to learn the best one via another expert-tracking algorithm~\citep{COLT'13:Optimistic-OMD}. However, the critical challenge is that the quantity $V_K$ depends on the \emph{unknown} compared policies $\pi_1^\C, \ldots, \pi_K^\C$, making it impossible to evaluate the quality of the sequence even after the horizon ends; and thus, learning the best optimism does not apply to our problem. To address the difficulty, we propose to maintain two sets of base-learners with different optimism sequences $\{0\}_{k\in K}$ and $\{m_k\}_{k\in K}$, in order to ensure the worst-case robustness; and importantly, an appropriate meta-algorithm is needed to adaptively combine those heterogenous base-learners and ensure the \emph{best-of-both-world} result. In what follows, we introduce the details.

\paragraph{Base-algorithm.}To exploit the predictability of the environments, the base-algorithm employs Optimistic OMD over a clipped occupancy measure space. To obtain an appropriate base-regret, we incorporate the optimism in the construction of correction terms. Specifically, let $m'_k$ be the optimism received by the base-algorithm at the beginning of episode $k$, the base-algorithm performs
\begin{equation}
    \label{eq:ssp-adaptive-base}
    \begin{split}
        q_{k} & = \argmin_{q \in \Delta(M, H, \alpha)}\eta \inner{q}{m'_k}+\Dp(q, \qh_{k}), \\
        \qh_{k+1} & = \argmin_{q \in \Delta(M,H, \alpha)}\eta \inner{q}{\ell_k+a_k}+\Dp(q, \qh_{k}) \mbox{ with } a_k = 32\eta(\ell_k-m'_k)^2,
    \end{split}
\end{equation}
where $\eta > 0$ is the step size, $\Delta(M, H, \alpha)=\{q\in \Delta(M,H) \mid q(x,a)\geq \alpha, \forall x,a\}$ is the clipped space with $\alpha \in (0,1)$, $\psi$ is the negative-entropy regularizer. We incorporate the optimism in the construction of correction term, namely, $a_k=32\eta (\ell_k-m'_k)^2, \forall k \in [K]$. This aims to ensure a small-loss dynamic regret and simultaneously introduce an additional \emph{negative term} that would be crucial to address the difficulty occurred in controlling meta-regret similar to that in Section~\ref{sec:dynamic-ssp}. Formally, the base-algorithm enjoys the following dynamic regret.
\begin{myLemma}
    \label{lem:ssp}
    Set $q_1 = \argmin_{q \in \Delta(M, H, \alpha)} \psi(q)$. Suppose $\eta \leq \frac{1}{64}$, for any compared policies $\piC_1,\ldots,\piC_K \in \{ \pi \mid q^\pi \in \Delta(M, H, \alpha) \}$, base-algorithm~\eqref{eq:ssp-adaptive-base} ensures  is upper bounded by
    \begin{align*}
        \sumk\inner{q_{k} - \qkc}{\ell_k} & \leq \frac{1}{\eta}\Big(\Pb_K\log{\frac{H}{\alpha}} + H(1+\log(|X||A|H))\Big)                                     \\
                                          & \qquad \qquad + 32\eta \sumk \inner{\qkc}{(\ell_{k}-\mp_{k})^2} - 16\eta \sumk \inner{q_k}{(\ell_{k}-\mp_{k})^2}.
    \end{align*}
\end{myLemma}

\paragraph{Scheduling.}Similar to that in Section~\ref{sec:dynamic-ssp}, we use the \emph{groupwise scheduling} scheme to simultaneously adapt to unknown non-stationary measure $\Pb_K$ and horizon length $H_*$. The constructions of the space pool and the step size pool are the same as that in Section~\ref{sec:dynamic-ssp}. We maintain \emph{two} sets of base-learners with different settings of their optimism, in order to ensure the worst-case robustness. Specifically, due to $H^{\pi^f} \leq H_* \leq K$, we first construct a horizon length pool $\mathcal{H} = \{H_i = 2^{i-1} \cdot H^{\pi^f} \mid i \in [G]\}$ with $G = 1 + \ceil{\log((K+1)/H^{\pi^f})}$ to exponentially discretize the possible range; and for each $H_i$ in the pool, we further design a step size grid $\Ecal_i = \{\eta_{i,j} = {1}/({32 \cdot 2^j}) \mid j \in [N_i] \}$ with $N_i = \ceil{\frac{1}{2}\log{(\frac{4 K}{1+\log{(|X||A|H_i)}})}}$ to search the optimal step size associated with $H_i$. Hence, we maintain $N = 2\sum_{i=1}^G N_i$ base-learners in total, each of which associates with a specific space size and step size and optimism. Let $\B_{i,1:2N_i}$ be a shorthand of the $i$-th group of base-learners $\B_{i,1},\ldots,\B_{i,2N_i}$. The configuration of all those base-learners are as follows.
\begin{itemize}
    \item For $j\leq N_i$, the base-learner $\B_{i,j}$ uses the space size $H_i \in \H$, the step size $\eta_{i,j} \in \Ecal_i$ and the optimism $m'_k=0$.
    \item For $j > N_i$, the base-learner $\B_{i,j}$ uses the space size $H_{i} \in \H$, the step size $\eta_{i,j-{N_i}} \in \Ecal_i$ and the optimism $m'_k=m_k$.
\end{itemize}
Then the set of all the base-learners can be denoted as $\{\B_{1,1:2N_1},\ldots,\B_{i,1:2N_i},\ldots,\B_{G,1:2N_G}\}$. The decision returned by the base-learner $\B_{i,j}$ on episode $k \in [K]$ is denoted by $q_{k}^{i,j}$, with $i \in [G]$ and $j \in [2N_i]$.

\begin{algorithm}[t]
    \caption{Optimistic \mbox{CODO-REPS}}
    \label{alg:optimisticCDO-REPS}
    \begin{algorithmic}[1]
        \REQUIRE horizon length pool $\mathcal{H} = \{H_1, \ldots, H_G\}$, step size grid $\Ecal_{i}=\{\varepsilon_{i,1}, \ldots, \varepsilon_{i,N_i}\}, \forall i \in [G]$ and clipping parameter $\alpha$.
        \STATE Define: $\psi(q) = \sum_{x,a} q(x,a) \log{q(x,a)}$ and $\psib(p)$ as in~\eqref{eq:weighted-entropy}.
        \STATE Initialization: $\qh_1^{\ i,j} = \argmin_{q \in \Delta(M, H_i, \alpha)} \psi(q)$ and $\ph_1^{\ i,j} \propto \varepsilon_{i,j}^2, \forall i \in [G], j \in [2N_i]$.
        \FOR{$k=1$ to $K$}
        \STATE Receive optimism $m_k \in \R_{[0,1]}^{|X||A|}$ and feed it to all base-learners.
        \STATE Set optimism for each base-learner $\B_{i,j}$: $m'_k=0$ if $j\leq N_i$ and $m'_k=m_k$ otherwise.
        \STATE Each base-learner $\B_{i,j}$ updates $q_{k}^{i,j} = \argmin_{q \in \Delta(M,H_i, \alpha)}\eta_{i,j} \inner{q}{m'_k}+\Dp(q, \qh_{k}^{\ i,j})$.
        \STATE Update $p_k = \argmin_{p \in \Delta_N} \inner{p}{M_k}+\Dpb(p, \ph_k)$ where $M_{k}^{i,j} = \inner{q_k^{i,j}}{m'_k}, \forall i, j$.
        \STATE Sample $(i_k,j_k) \sim p_k$, play the induced policy $\pi_k(a | x) \propto q^{i_k,j_k}_k(x, a), \forall x,a$.
        \STATE Suffer losses $\{\ell_k(x_1,a_1), \ldots, \ell_k(x_{I^k}, a_{I^k})\}$, receive $\ell_k$, and feed it to all base-learners.
        \STATE Define $h_{k}^{i,j} = \inner{q_k^{i,j}}{\ell_k}, b_k^{i,j} = 32 \varepsilon_{i,j} (h_k^{i,j} - M_k^{i,j})^2, a_k^{i,j} = 32\eta_{i,j}(\ell_k - m'_k)^2, \forall i, j$.
        \STATE Each base-learner $\B_{i,j}$ updates $\qh_{k+1}^{\ i,j} = \argmin_{q \in \Delta(M,H_i, \alpha)}\eta_{i,j} \inner{q}{\ell_k+a_k^{i,j}}+\Dp(q, \qh_{k}^{\ i,j})$.\\
        \STATE Update weight by $\ph_{k+1} = \argmin_{p \in \Delta_N} \inner{p}{h_k+b_k}+\Dpb(p, \ph_k)$.
        \ENDFOR
    \end{algorithmic}
\end{algorithm}

\paragraph{Meta-algorithm.}Similar to the base-algorithm, we incorporate the optimism in the updates and the construction of the correction term. Specifically, the meta-algorithm updates the weight vector $p_{k} \in \Delta_N$ by
\begin{equation*}
    p_{k} = \argmin_{p\in\Delta_N}\inner{p}{M_k} + \Dw(p, \ph_k), \quad \ph_{k+1} = \argmin_{p\in\Delta_N}\inner{p}{h_k+b_k} + \Dw(p, \ph_k)
\end{equation*}
with $b_{k}^{i,j} = 32 \varepsilon_{i,j} (h_{k}^{i,j}-M_{k}^{i,j})^2$, $\forall i \in [G], j \in [2N_i]$. Moreover, $h_k, M_K \in \R^N$ is the loss and optimism of meta-algorithm, defined as $h_k^{i,j} = \inner{q^{i,j}_k}{\ell_k}$ and $M_k^{i,j} = \inner{q^{i,j}_k}{m'_k}, \forall i \in [G], j \in [2N_i]$. Moreover, $\psib: \Delta_N \rightarrow \R$ is the weighted entropy regularizer defined as
\begin{equation}
    \label{eq:weighted-entropy}
    \psib(p) = \sum_{i =1}^G \sum_{j=1}^{2N_i} \frac{1}{\varepsilon_{i,j}}p_{i,j}\log p_{i,j}, \mbox{ with }\varepsilon_{i,j} = \frac{\eta_{i,j}}{2H_i}.
\end{equation}
In above, $\eta_{i,j}$ and $H_i$ are the step size and space size employed by the base-learner $\B_{i,j}$ as specified earlier. The weighted entropy regularizer depends on both space size and step size such that the final meta-algorithm can successfully handle the groupwise scheduling over the base-learners. The algorithm is shown in Algorithm~\ref{alg:optimisticCDO-REPS}, which enjoys the following guarantee.
\begin{myThm}
    \label{thm:ssp-adaptivity}
    Under the same setting as Theorem~\ref{thm:ssp}, Optimistic \mbox{CODO-REPS} (Algorithm~\ref{alg:optimisticCDO-REPS}) enjoys the following adaptive dynamic regret guarantee,
    \begin{equation}
        \label{eq:ssp-adaptivity-bound}
        \E[{\DReg}_K(\pi_{1:K}^\C)] \leq \Ot\Big(\sqrt{(H_{*}+ \Pb_{K})(H_{*}+\Pb_K+V_K )}\Big).
    \end{equation}
    where $\Pb_K = \Pb_K(\piC_1,\ldots,\piC_K) = \sum_{k=2}^K \norm{\qkc - q^{\pi_{k-1}^\C}}_1$ is the path length of occupancy measure, and $V_K = \min\{\sum_{k=1}^K {\inner{\qkc}{\ell_k^2}}, \sum_{k=1}^K {\inner{\qkc}{(\ell_k-m_k)^2}}\}$ is the adaptive quantity.
\end{myThm}
\begin{myRemark}
    This adaptive bound in Theorem~\ref{thm:ssp-adaptivity} is strictly no worse than the minimax result in Theorem~\ref{thm:ssp} since $V_K \leq B_K$ holds for any optimism sequence, and can become much tighter if the optimism sequence $m_1, \ldots, m_K$ is of high quality (for example $V_K = 0$ when $m_k$ predicts $\ell_k$ perfects at each episode).
\end{myRemark}
\section{Infinite-horizon MDPs}
\label{sec:infinite-horizon}
This section studies infinite-horizon MDPs. We begin with the problem setup and then present our main results, including a reduction to the switching-cost expert problem, the dynamic regret bound, and an impossibility result for the adaptive bound.

\subsection{Problem Setup}
An infinite-horizon MDP instance is specified by a tuple $M = (X, A, P, \{\ell_t\}_{t=1}^\infty)$, where $X, A, P$ are the same as introduced in Section~\ref{sec:loop-free-SSP}, $\ell_t \in \R_{[0,1]}^{\abs{X} \abs{A}}$ is the loss function at time $t \in [T]$. Unlike episodic MDPs studied in previous two sections, infinite-horizon MDPs have no goal state. The learner aims to minimize the cumulative loss over a $T$-step horizon in the MDP. We investigate the uniform mixing MDPs~\citep{MatOR'09:online-MDP, NIPS'10:bandit-MDP-Neu}.
\begin{myDef}[Uniform Mixing]
    \label{def:mixing}
    There exists a constant $\tau \geq 0$ such that for any policy $\pi$ and any pair of distributions $\mu$ and $\mu^\prime$ over $X$, we have $\norm{(\mu-\mu^\prime)P^\pi}_1 \leq e^{-1/\tau} \norm{\mu - \mu^\prime}_1$. The smallest constant $\tau$ is called the mixing time.
\end{myDef}
We note that the uniform mixing assumption is standard and widely adopted in online MDPs studies~\citep{MatOR'09:online-MDP, NIPS'10:bandit-MDP-Neu,IEEE'Control:bandit-MDP-Neu}. Nevertheless, the assumption could be strong in some sense, and recent study trying to relax the assumption by considering a larger class of communicating MDPs~\citep{arXiv'21:communicating-MDP}. It would be interesting to see whether our results can be extended to the communicating MDPs, and we leave this as future work to investigate.

\paragraph{Occupancy measure.}For an infinite-horizon MDP, the occupancy measure $q^\pi \in \R_{[0, 1]}^{|X||A|}$ is defined as the stationary state-action distribution when executing policy $\pi$, \ie $q^\pi(x,a) = \lim_{T \rightarrow \infty} \frac{1}{T} \sum_{t=1}^T \ind\{x_t = x, a_t=a\}$. For an infinite-horizon MDP instance $M$, its occupancy measure space is defined as $\Delta(M) = \{q \in \R_{[0,1]}^{|X||A|} \mid \sum_{x,a}q(x,a) =1 \text{ and } \sum_a q(x,a) = \sum_{x^\prime, a^\prime} P(x | x^\prime, a^\prime)q(x^\prime, a^\prime), \forall x \in X\}$. For any occupancy measure $q \in \Delta(M)$, its induced policy $\pi$ can be obtained by $\pi(a|x) \propto q(x,a), \forall x \in \X, a \in \A$.

\paragraph{Dynamic regret.} As defined in~\pref{eq:dynamic-regret}, the dynamic regret of infinite-horizon MDPs benchmarks the learner's performance against a sequence of compared policies $\pi_{1:T}^\C$, namely,
\begin{equation}
    \label{eq:dynamic-regret-infinite-horizon}
    \E[{\DReg}_T(\pi_{1:T}^\C)] = \E\bigg[\sum_{t=1}^{T} \ell_t\big(x_t, \pi_t(x_t)\big) - \sum_{t=1}^{T} \ell_t\big(x_t, \pi_t^\C(x_t)\big)\bigg],
\end{equation}
The non-stationarity measure is naturally defined as $P_T = \sum_{t=2}^T \norm{\pi_{t}^{\C}-\pi_{t-1}^{\C}}_{1, \infty}$.

\subsection{Reduction to Switching-cost Expert Problem}
\label{sec:reduction}
In this part, we present a reduction to the switching-cost expert problem for infinite-horizon MDPs. In fact, we have the following theorem.
\begin{myThm}
    \label{thm:reduction}
    For infinite-horizon MDPs with a mixing time $\tau \geq 0$, the expected dynamic regret against any compared policies $\pi_{1:T}^\C$ satisfies
    \begin{align}
        \label{eq:mdp-sc}
        \E[{\DReg}_T(\pi_{1:T}^\C)]  \leq \sum_{t=1}^T \inner{q_t-\qtc}{\ell_t} + (\tau+1) \sum_{t=2}^T \norm{q_t - q_{t-1}}_1 + (\tau+1)^2 P_T  + 4(\tau+1).
    \end{align}
    where $q_t = q^{\pi_t}$ denotes the occupancy measure of the policy $\pi_t$, $q^{\piC_t}$ denotes the occupancy measure of the compared policy $\piC_t$, and $P_T = \sum_{t=2}^T \norm{\pi_{t}^{\C}-\pi_{t-1}^{\C}}_{1, \infty}$ is the path length of the sequence of compared policies.
\end{myThm}
Therefore, it suffices to design an algorithm to minimize the first two terms on the right-hand side of~\eqref{eq:mdp-sc}, as the last two terms $(\tau+1)^2 P_T  + 4(\tau+1)$ are not related to the algorithm. This essentially provides a generic regret reduction from infinite-horizon MDPs to the \emph{switching-cost expert problem}~\citep{TIT'02:OCOwithMemory}. Specifically, for the expert problem, at each round $t \in [T]$, the learner chooses a decision $q_t \in \Delta_N$ as a weight over all $N$ experts, then receives the loss $\ell_t \in \R^N$ and suffers an instantaneous loss $\inner{q_t}{\ell_t}$. In addition to the cumulative loss $\sum_{t=1}^T\inner{q_t}{\ell_t}$, the switching-cost expert problem further takes the actions' switch into account by adding $\lambda \sum_{t=2}^T \norm{q_t- q_{t-1}}_1$ as the penalty, $\lambda > 0$ is the coefficient.

Our reduction also holds for the static regret (simply choosing all compared policies as a fixed one), perhaps surprisingly, there is no explicit reduction in the literature to the best of our knowledge, though proof of Theorem~\ref{thm:reduction} is simple and all the ingredients are already in the pioneering work~\citep{MatOR'09:online-MDP} (see Appendix~\ref{sec:appendix-proof-reduction}). As another note,~\citet{ICML'19:online-control} study online non-stochastic control and give a reduction to the switching-cost online learning problem (or called online convex optimization with memory), while their reduction does not apply to infinite-horizon MDPs.

\subsection{Dynamic Regret}
\label{sec:dynamic-infinite-horizon}

With the reduction on hand, we now consider the design of a two-layer approach to optimize the dynamic regret of the switching-cost expert problem. It turns out that a recent result~\citep{aistats'22:scream} has resolved that expert problem, building upon which we propose our REgularized \mbox{DO-REPS} (REDO-REPS) algorithm for infinite-horizon MDPs.

As discussed before, it suffices to design an algorithm to minimize the first two terms in~\eqref{eq:mdp-sc}, namely, the dynamic regret in terms of the occupancy measure and a switching cost term. Notice that the first term also appears in optimizing dynamic regret of the episodic loop-free SSP (see~\pref{eq:dynamic-regret-episodic-loop-free}). Thus, a natural idea is to run \mbox{DO-REPS} (Algorithm~\ref{alg:loop-free}) over the occupancy measure space $\Delta(M, \alpha) = \{ q \in \Delta(M) \mid q(x,a)\geq \alpha, \forall x,a \}$. Specifically, we maintain $N$ base-learners denoted by $\B_1,\ldots,\B_N$, where $\B_i$ generates the prediction $q_{t,i}$ by performing \mbox{O-REPS} with a particular step size $\eta_i$ in the step size pool $\H$; then a meta-algorithm combines predictions to produce the final decision $q_t = \sum_{i=1}^N p_{t,i}q_{t,i}$ and updates the weight $p_t$.
However, \mbox{DO-REPS} does not take the switching cost into account, leading to undesired behavior in this problem. To see the reason, we decompose the switching cost as
\begin{align}
    \sum_{t=2}^T \norm{q_t - q_{t-1}}_1 & = \sum_{t=2}^T \Big \lVert\sum_{i=1}^N p_{t,i}q_{t,i} - \sum_{i=1}^N p_{t-1,i}q_{t-1,i}\Big \rVert_1 \nonumber                                                                                                        \\
                                        & \leq \sum_{t=2}^T \Big \lVert\sum_{i=1}^N p_{t,i}q_{t,i} - \sum_{i=1}^N p_{t,i}q_{t-1,i}\Big \rVert_1 + \sum_{t=2}^T \Big \lVert\sum_{i=1}^N p_{t,i}q_{t-1,i} - \sum_{i=1}^N p_{t-1,i}q_{t-1,i}\Big \rVert_1\nonumber \\
                                        & = \sum_{t=2}^T \Big \lVert\sum_{i=2}^N p_{t,i}(q_{t,i}-q_{t-1,i})\Big \rVert_1 + \sum_{t=2}^T \Big \lVert\sum_{i=1}^N (p_{t,i}-p_{t-1,i})q_{t-1,i}\Big \rVert_1 \nonumber                                             \\
                                        & \leq \sum_{t=2}^T \sum_{i=1}^N p_{t,i}\norm{ q_{t,i}-q_{t-1,i}}_1 + \sum_{t=2}^T \norm{ p_t - p_{t-1}}_1\label{eq:switching-cost}.
\end{align}
The second term in the right-hand side of~\eqref{eq:switching-cost} is the meta-algorithm's switching cost, which can be easily bounded by $\O(\sqrt{T})$ for common expert-tracking algorithms. However, the first term is the weighted switching cost of all base-learners, which could be very large and even grow linearly with iterations due to the base-learners with large step sizes. For example, when employing OMD as the base-algorithm, the switching cost of $\B_i$ is of order $\O(\eta_i T)$. Then, the construction of step size pool requires that $\eta_N = \O(1)$, leading to an $\O(T)$ switching cost of the base-learner $\B_N$, which ruins the overall regret bound. To address this, inspired by the recent progress on OCO with memory~\citep{aistats'22:scream}, we add a switching-cost regularization in evaluating each base-learner, namely, the feedback loss of the meta-algorithm $h_t \in \R^N$ is constructed as
\begin{equation}
    \label{eq:regularized-loss}
    h_{t,i} = \inner{q_{t,i}}{\ell_t} + \lambda \norm{q_{t,i} - q_{t-1,i}}_1.
\end{equation}
Set $\lambda=\tau+1$, it can be verified that the first two terms in~\eqref{eq:mdp-sc} can be written as
\begin{align}
      & \sum_{t=1}^T \inner{q_t-\qtc}{\ell_t} + (\tau+1) \sum_{t=2}^T \norm{q_t - q_{t-1}}_1 \nonumber \\
    = & \underbrace{\sum_{t=1}^T (\inner{p_t}{h_t} - h_{t,i}) + \lambda \sum_{t=2}^T \norm{p_t - p_{t-1}}_1}_{\meta} + \underbrace{\sum_{t=1}^T \inner{q_{t,i}-\qtc}{\ell_t} + \lambda \sum_{t=2}^T \norm{q_{t,i}-q_{t-1,i}}_1}_{\base}\label{eq:sc-decomposition}.
\end{align}
As a result, we have decomposed the switching-cost dynamic regret into two parts --- the first part is the meta-regret regarding the regularized loss $h_t$ that measures the regret overhead of the meta-algorithm penalized by the corresponding switching cost, and the second part is the base-regret of a specific base-learner $\B_i$ taking her switching cost into account. Consequently, by slightly modifying \mbox{DO-REPS} (Algorithm~\ref{alg:loop-free}), we get REgularized \mbox{DO-REPS} (REDO-REPS) algorithm as shown in Algorithm \ref{alg:infinite}. The key difference is the designed switching-cost-regularized loss for the meta-algorithm's update in Lines~\ref{line:sc-meta}--\ref{line:update-meta}, such that the overall two-layer approach can achieve a desired dynamic regret with switching cost as shown below.

\begin{algorithm}[t]
    \caption{REDO-REPS}
    \label{alg:infinite}
    \begin{algorithmic}[1]
        \REQUIRE step size pool $\H = \{\eta_1, \ldots, \eta_N\}$, learning rate $\varepsilon$ and clipping parameter $\alpha$.
        \STATE Define: $\psi(q) = \sum_{x,a} q(x,a) \log{q(x,a)}$.
        \STATE Initialization: set $q_{1,i} = \argmin_{q\in \Delta(M, \alpha)}\psi(q)$ and $p_{1,i} = 1/N$ for $\forall i \in [N]$.
        \FOR{$t=1$ to $T$}
        \STATE Receive $q_{t,i}$ from base-learner $\B_i, \forall i \in [N]$.
        \STATE Compute $q_t = \sum_{i=1}^N p_{t,i} q_{t,i}$, play the induced policy $\pi_t(a|x_t) \propto q_t(x_t,a), \forall a \in \A$.
        \STATE Suffer loss $\ell_t(x_t, a_t)$ and observe loss function $\ell_t$.
        \STATE Define the switching-cost-regularized loss as $$h_{t,i} = \inner{q_{t,i}}{\ell_t} + (\tau+1) \norm{q_{t,i}-q_{t-1,i}}_1, \forall i \in [N].$$\\ \label{line:sc-meta}
        \STATE Update weight by $p_{t+1,i} \propto \exp(-\varepsilon \sum_{s=1}^t h_{s,i}), \forall i \in [N]$. \label{line:update-meta}
        \STATE Each base-learner $\B_i$ updates ${q}_{t+1,i} = \argmin_{q\in \Delta(M, \alpha)}{\eta_i}\inner{q}{\ell_t} + \Dp(q , {q}_{t,i})$.
        \ENDFOR
    \end{algorithmic}
\end{algorithm}

\begin{myThm}
    \label{thm:total-regret-general}
    Set the step size pool $\H = \Big\{2^{i-1}\sqrt{T^{-1}\log |X||A|} \mid i \in [N]\Big\}$ where $N=\ceil{\frac{1}{2}\log(1+\frac{4T \log T}{\log|X||A|})}+1$, the learning rate $\varepsilon = (2\tau+3)^{-1}\sqrt{(\log{N})/2T}$ and the clipping parameter $\alpha=1/T^2$. REDO-REPS (Algorithm~\ref{alg:infinite}) ensures
    \begin{align*}
        \E[{\DReg}_T] \leq \O\left( \sqrt{\tau T \left(\log{|X||A| + \tau P_T \log{T}}\right)} + \tau^2 P_T \right).
    \end{align*}
\end{myThm}
\begin{myRemark}
    Set $\pi_{1:T}^\C = \pi^* \in \argmin_{\pi \in \Pi} \sum_{t=1}^T \ell_t(x_t, \pi(x_t))$ (then $P_T = 0$), Theorem \ref{thm:total-regret-general} recovers the best known static regret $\O (\sqrt{\tau T\log{|X||A|}} )$~\citep{NIPS'13:MDP-Neu}. Note that the dynamic regret scales with the path length of compared policies rather than the path length of occupancy measures. To achieve so, we establish relationships of path length variations between compared policies and their induced occupancy measures, which can be found in Lemma~\ref{lem:infinite-horizon-2} of Appendix~\ref{sec:appendix-infinite-horizon-occupancy-measure}.
\end{myRemark}

\subsection{More Adaptive Results}
\label{sec:adaptive-infinite-horizon}
We finally consider whether it is possible to enhance our algorithm to achieve adaptive dynamic regret bounds.

Based on the reduction in Theorem~\ref{thm:reduction}, we naturally consider how to design adaptive algorithms for the switching-cost problem that can exploit the predictability of environments. Unfortunately, we only have a negative result in this regard.
\begin{myThm}
    \label{thm:impossibility}
    For any online algorithm, there exists a loss sequence $\ell_1,\ldots,\ell_T$ such that the returned decision sequence $q_1,\ldots,q_T$ satisfies
    \[
        \sum_{t=1}^T \inner{q_t}{\ell_t} - \min_{q \in \Delta_N} \sum_{t=1}^T \inner{q}{\ell_t} + \sum_{t=2}^T \norm{q_t - q_{t-1}}_1 \geq \Omega \Big(\sqrt{1 + \sum_{t=2}^T \norm{\ell_t-\ell_{t-1}}_\infty^2} \Big).
    \]
\end{myThm}
Theorem~\ref{thm:impossibility} indicates that even for the static regret (a special case of our concerned dynamic regret) of the switching-cost expert problem, it is impossible to obtain variation-type adaptive bound, not to mention a general optimistic bound.

We note that a similar trade-off between adaptivity and switching cost was also considered by the prior work of~\citet{COLT'14:higher-order}, who show that it is impossible to achieve an adaptive bound scaling with the variance of  gradients in online linear optimization with switching cost. The caveat is that their lower bound argument crucially relies on the two factors: (i) they require an \emph{adaptive} adversary; (ii) the loss is required to a signed game with the loss range in $[-1,1]$. Unfortunately, this result does not apply to our case, in that the online problem reduced from the online infinite-horizon MDPs does not satisfy the two conditions --- the loss functions are chosen in an \emph{oblivious} way and lie in the range of $[0,1]$. Furthermore, the proof techniques of their result and our Theorem~\ref{thm:impossibility} exhibit salient difference.~\citet{COLT'14:higher-order} rely on the adaptive adversary and use some flipping operation to construct a hard instance to constitute a lower bound. By contrast, we only have an oblivious adversary and thus the proof of lower bound is more challenging. To address this, we connect the problem of adaptive OCO with switching cost to the problem of OCO with switching budget~\citep{COLT'18:switch-cost} and thus establish the desired lower bound.

We finally emphasize that Theorem~\ref{thm:impossibility} does not constitute a direct lower bound for dynamic regret of infinite-horizon MDPs. However, this suggests that significant new analyses are required to obtain adaptive bounds for this problem, though it seems to be pessimistic.
\section{Experiment}
\label{sec:experiment}
In this section, we present empirical studies to examine the performance of our algorithms.

\subsection{Episodic loop-free SSP}
\label{sec:experiment-loop-free-ssp}

\begin{figure}[t]
    \centering
    \begin{minipage}[t]{0.48\textwidth}
        \centering
        \includegraphics[width=0.9\linewidth, trim=100 0 100 30,clip]{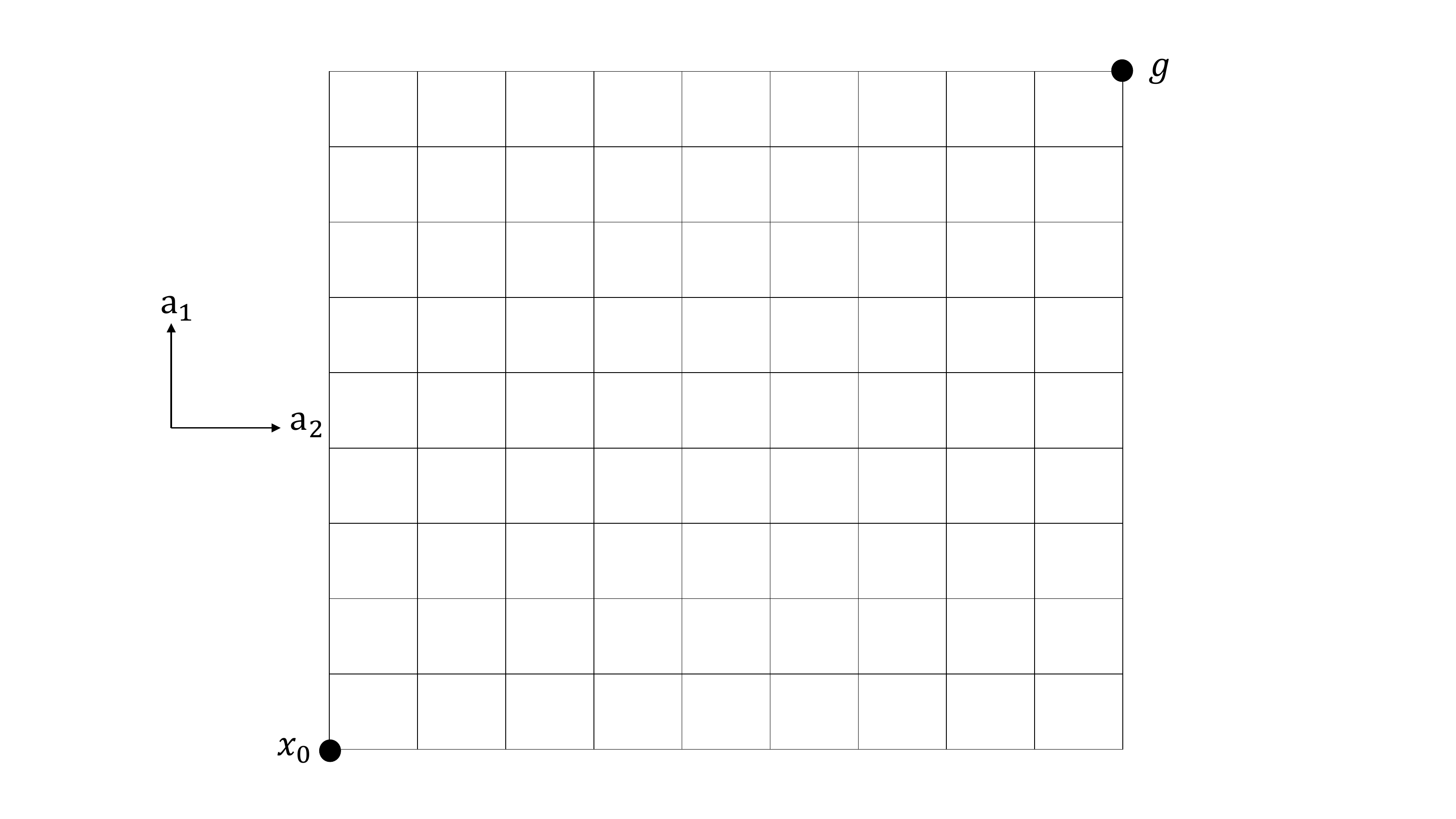}
        \caption{GridWorld environment}
        \label{fig:loop-free-env}
    \end{minipage}
    \begin{minipage}[t]{0.48\textwidth}
        \centering
        \includegraphics[width=0.9\linewidth]{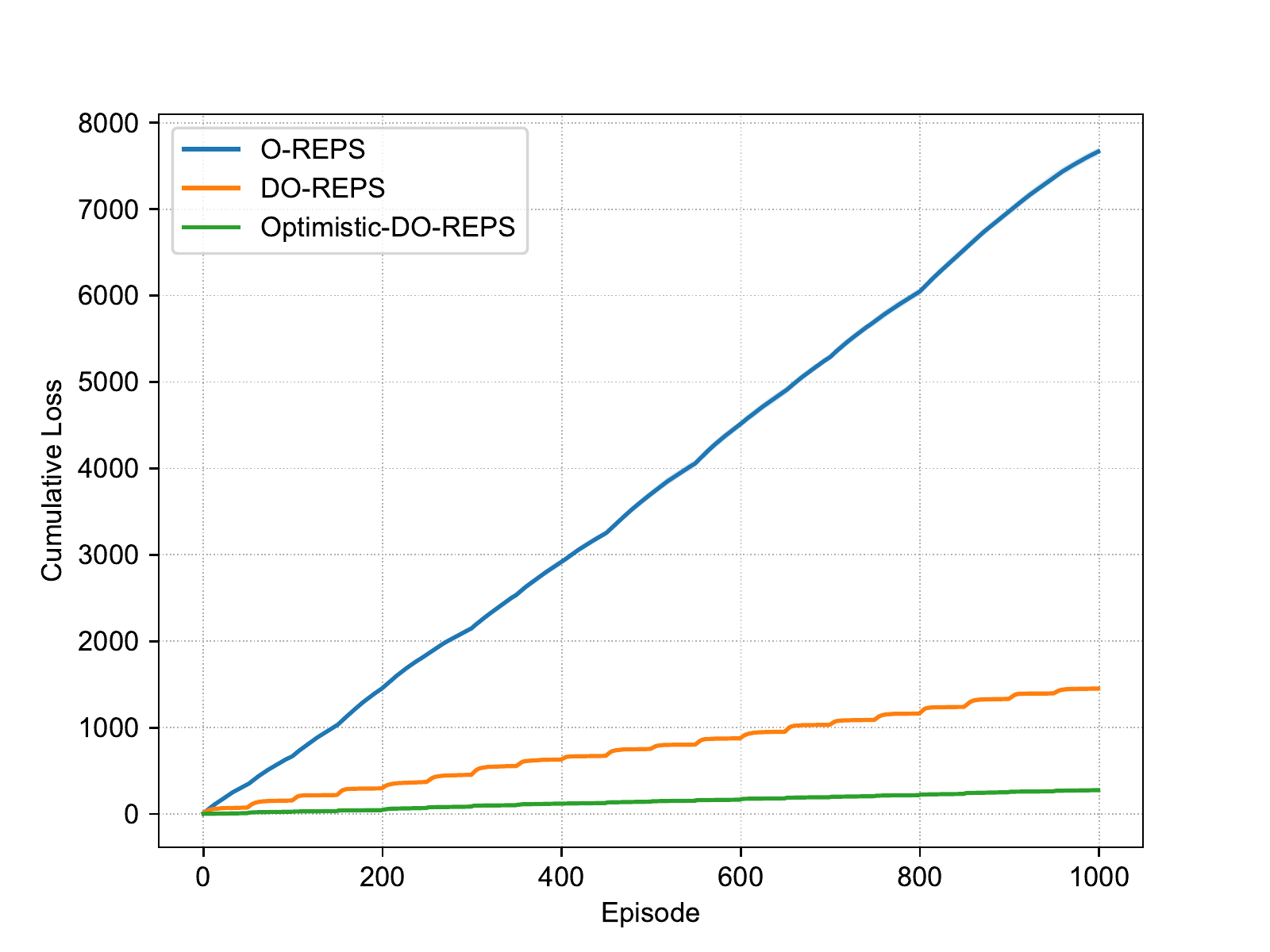}
        \caption{Cumulative loss of algorithms}
        \label{fig:loop-free-result}
    \end{minipage}
\end{figure}

We consider a GridWorld environment of size $10 \times 10$, where in each episode the learner starts from the lower left corner $x_0$ to the upper right corner $g$~\citep{colt'10:loop-free}, as shown in Figure~\ref{fig:loop-free-env}. The learner has $2$ actions in each state: moving either up or right. Taking any action leads to the corresponding direction with probability $0.9$ and the other direction with probability $0.1$. If the learner tries to move out of the boundary, the attempt will not succeed and the learner will move forward in the other direction. Thus, the problem satisfies the requirements of episodic loop-free SSPs, where the horizon length in each episode is $H=20$, the state number is $|\X|=99$ and the action number is $|\A|=2$. The number of episodes is set to $K = 1000$. The loss function $\ell_k \in \R_{[0,1]}^{|\X||\A|}$ is forced to be piecewise stationary and will change every $50$ episodes to simulate the non-stationary environments with abrupt changes. In each piece, we randomly choose an action $a \in \A$ for each state $x \in \X$ and set the loss as $\ell(x,a)=0$ and $\ell(x, a')=1$ for the other action $a'$.

We compare the performance of our proposed \mbox{DO-REPS} algorithm (Algorithm~\ref{alg:loop-free}) and Optimistic-DO-REPS algorithm (Algorithm~\ref{alg:ODO-REPS}) against \mbox{O-REPS} algorithm~\citep{NIPS'13:MDP-Neu} for this problem. For Optimistic-DO-REPS, we set the optimism $m_k = \ell_k, \forall k \in [K]$ to show the benefit of the optimism with high quality.

Figure~\ref{fig:loop-free-result} shows the cumulative loss of all algorithms. The performance of the algorithms are evaluated by running the corresponding policies in the environment $10$ times and we report the mean and the standard deviation to show the effect of the stochastic policies and transition kernel. We can observe that \mbox{O-REPS} incurs a large cumulative loss over the episodes and can not effectively learn from the non-stationary environments. By contrast, \mbox{DO-REPS} and Optimistic-DO-REPS achieve small cumulative loss in the changing environment and outperforms \mbox{O-REPS} significantly. Moreover, Optimistic-DO-REPS can exploit the knowledge of the optimism and achieve extremely small cumulative loss with high quality optimism, which supports our theoretical findings.

\subsection{Episodic SSP}
\label{sec:experiment-non-loop-free-ssp}

To simulate an episodic (non-loop-free) SSP, we again use the GridWorld environment of size $10 \times 10$. In addition, we make slight modifications to enforce larger differences in the hitting time of different polices, as shown in Figure~\ref{fig:ssp-env}. All states form a circle and transitions are possible only along the circle and the learner starts from the lower left corner $x_0$ to the upper right corner $g$ in each episode. The learner has $2$ actions in each state: moving either in a clockwise direction or the opposite. Taking any action leads to the corresponding direction with probability $0.9$ and the other one with probability $0.1$. Thus, the state number is $|\X|=99$ and the action number is $|\A|=2$. The number of episodes is set to $K = 1000$. The loss function $\ell_k \in \R_{[0,1]}^{|\X||\A|}$ is forced to be piecewise stationary and will change every $50$ episodes to simulate the non-stationary environments with abrupt changes. We randomly choose an action $a \in \A$ and set the losses $\ell(x,a)=0$ and $\ell(x, a')=1$ for the other action $a'$ for all states $x \in \X$ in the first piece and swap losses for two actions after each piece ends.

We compare the performance of our proposed \mbox{CODO-REPS} (Algorithm~\ref{alg:CODO-REPS}) and Optimistic-CODO-REPS algorithm (Algorithm~\ref{alg:optimisticCDO-REPS}) against the following two contenders designed for optimizing the static regret: (\romannumeral1) SSP-O-REPS algorithm~\citep{IJCAI'21:SSP-Rosenberg}, online mirror descent on largest possible occupancy measure space with $H=\frac{D}{c_{\min}}$ for $c_{\min} > 0$ and $H=DK^{\frac{3}{4}}$ for $c_{\min} = 0$ where $c_{\min}$ is minimum loss. (\romannumeral2) Ada-O-REPS algorithm~\citep{COLT'21:SSP-minimax}, which employs a meta-base two-layer structure to learn the unknown hitting time of the optimal policy $H^{\pi^*}$ on the fly. For Optimistic-DO-REPS, we set the optimism as $m_k = 2\ell_k, \forall k \in [K]$ to make sure the correction terms of Optimistic-CODO-REPS are the same as \mbox{CODO-REPS} (note that the correction term $a_k = 32 \eta (\ell_k - m_k)^2, b_k = 32 \varepsilon (h_k - M_k)^2, \forall k \in [K]$) while ensuring the high quality of the optimism.

\begin{figure}[t]
    \centering
    \begin{minipage}[t]{0.48\textwidth}
        \centering
        \includegraphics[width=0.9\linewidth, trim=100 0 100 30,clip]{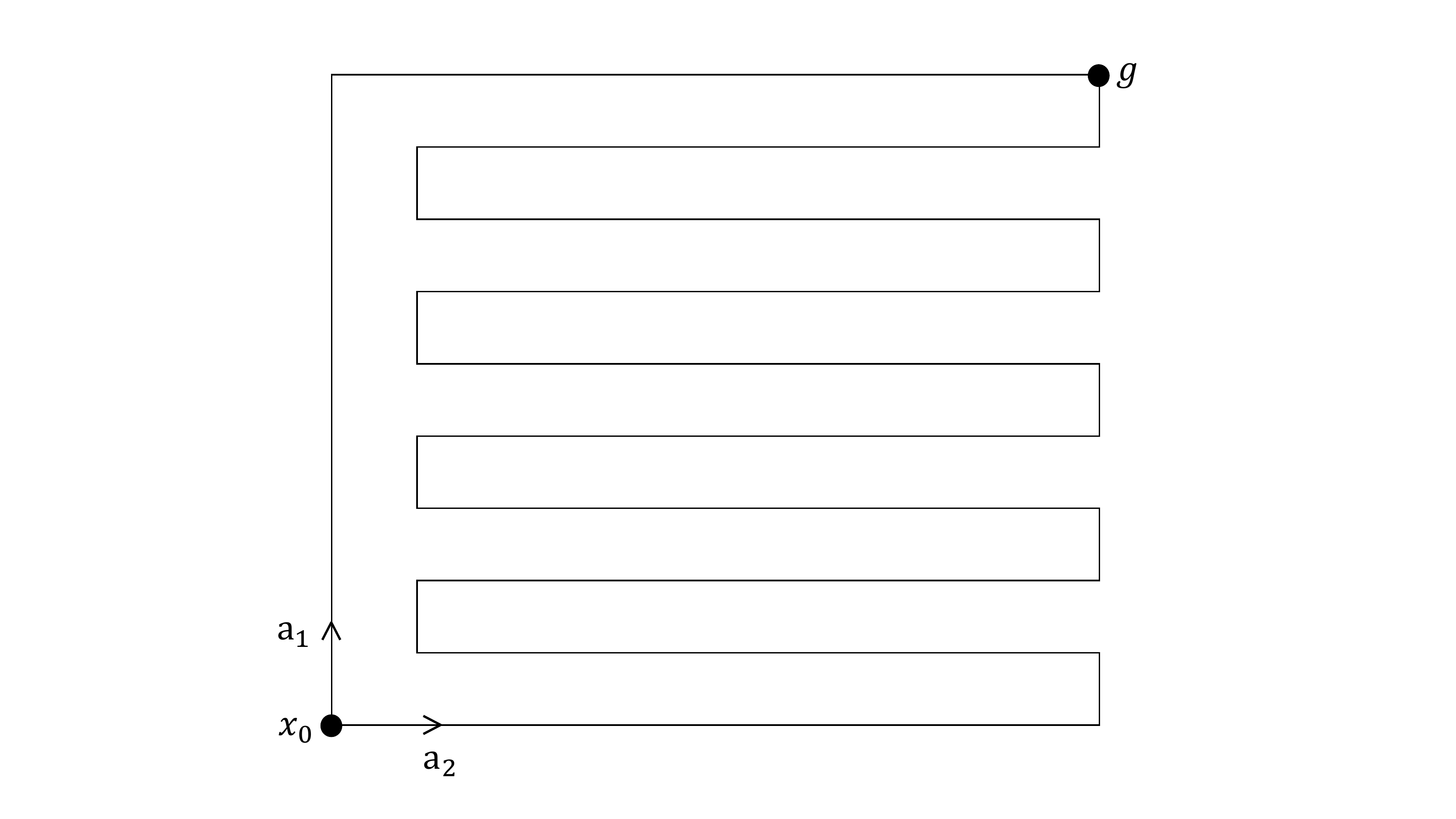}
        \caption{GridWorld environment}
        \label{fig:ssp-env}
    \end{minipage}
    \begin{minipage}[t]{0.48\textwidth}
        \centering
        \includegraphics[width=0.9\linewidth]{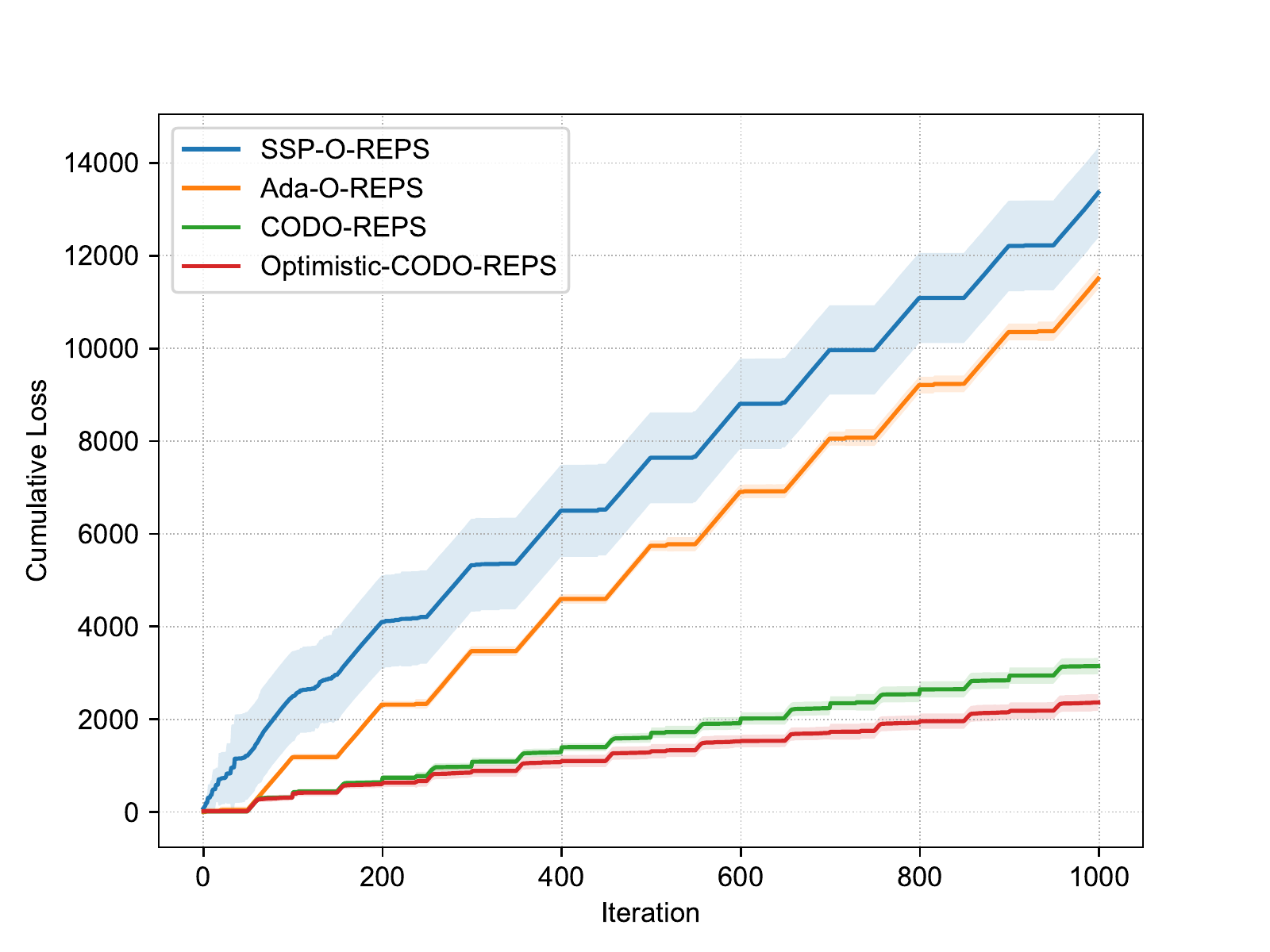}
        \caption{Cumulative loss of algorithms}
        \label{fig:ssp-result}
    \end{minipage}
\end{figure}

Figure~\ref{fig:loop-free-result} shows the cumulative loss of all algorithms. The performance of the algorithms are evaluated by running the corresponding policies in the environment $10$ times and we report the mean and the standard deviation to show the effect of the stochastic policies and transition kernel. It can be observed that our algorithms outperform the existing algorithms significantly.Moreover, the cumulative losses of SSP-O-REPS and Ada-O-REPS remain almost the same for odd pieces and grows linearly for even pieces. This is due to these two algorithms fail to learn in the non-stationary environments and keep one direction all the episodes. On the contrary, both our \mbox{CODO-REPS} and Optimistic-CODO-REPS can adapt to the changes of the environments quickly and suffer almost zero losses after a little episodes in each piece. Moreover, Optimistic-CODO-REPS can exploit the optimism and suffer smaller cumulative loss than \mbox{CODO-REPS} if the optimism is of high quality.

\subsection{Infinite-horizon MDPs}
\label{sec:experiment-infinite-horizon-mdps}

We consider the same GridWorld environment as that in Section~\ref{sec:experiment-loop-free-ssp}, where the learner starts from the initial state $x_0$. The difference is that now there is not goal state $g$ in infinite-horizon MDPs and the learner keeps moving in the MDP to minimize the cumulative loss over a $T$ step horizon. The learner has $4$ actions in each state: moving up, down, left or right. Taking any action leads to the corresponding direction with probability $0.9$ and other undesired directions uniformly at random with probability $0.1$.  Any action leads the learner to go out of the boundary will fail and the learner will not move. The number of steps is set to $T = 5000$. The loss function $\ell_k \in \R_{[0,1]}^{|\X||\A|}$ is forced to be piecewise stationary and will change every $500$ steps to simulate the non-stationary environments with abrupt changes. In each piece, we randomly choose an action $a \in \A$ for each state $x \in \X$ and set the loss as $\ell(x,a)=0$ and $\ell(x, a')=1$ for the other three actions $a' \in \A$.

\begin{figure}[t]
    \centering
    \begin{minipage}[t]{0.48\textwidth}
        \centering
        \includegraphics[width=0.9\linewidth]{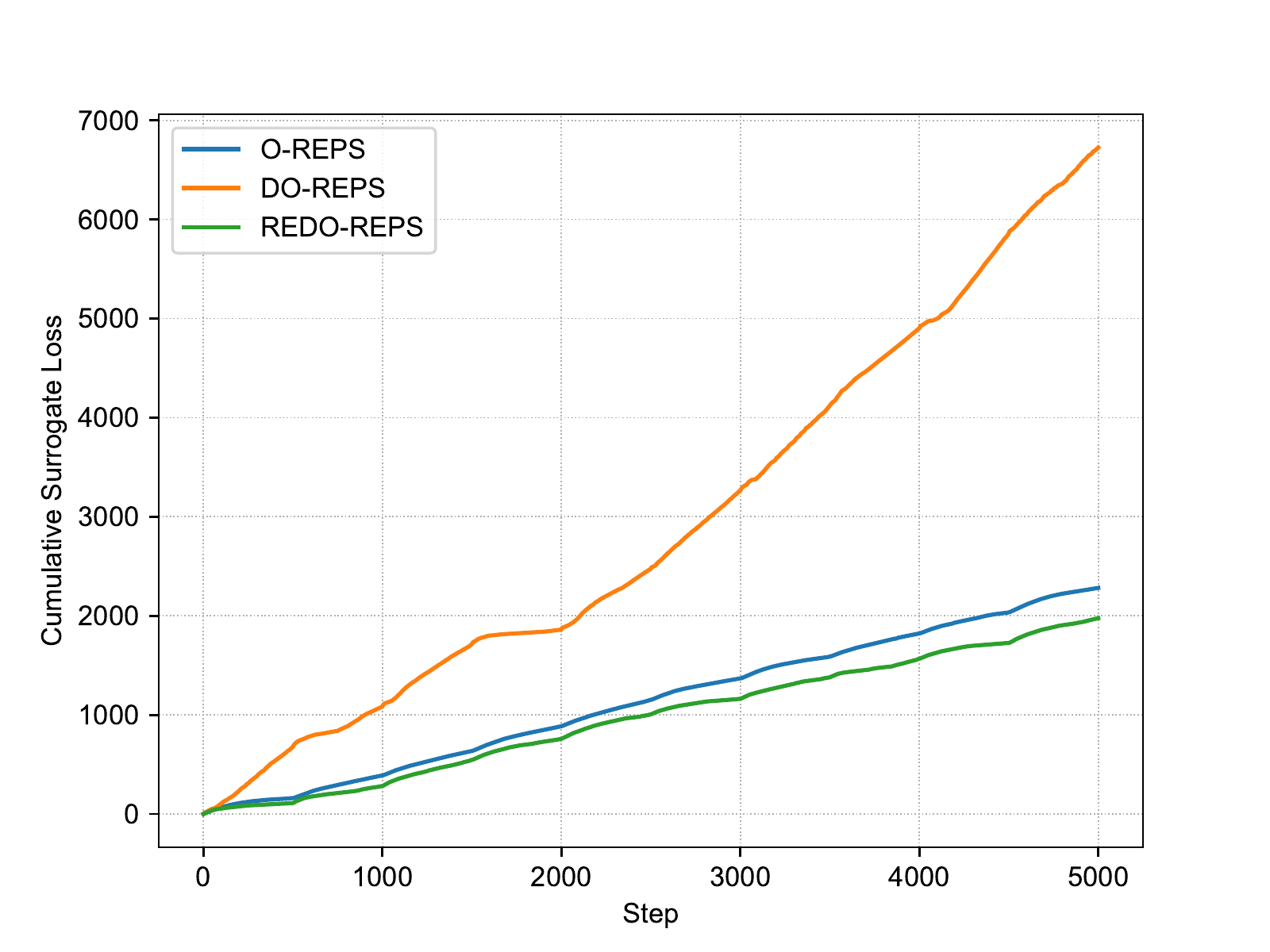}
        \caption{Cumulative surrogate loss}
        \label{fig:infinite-switching}
    \end{minipage}
    \begin{minipage}[t]{0.48\textwidth}
        \centering
        \includegraphics[width=0.9\linewidth]{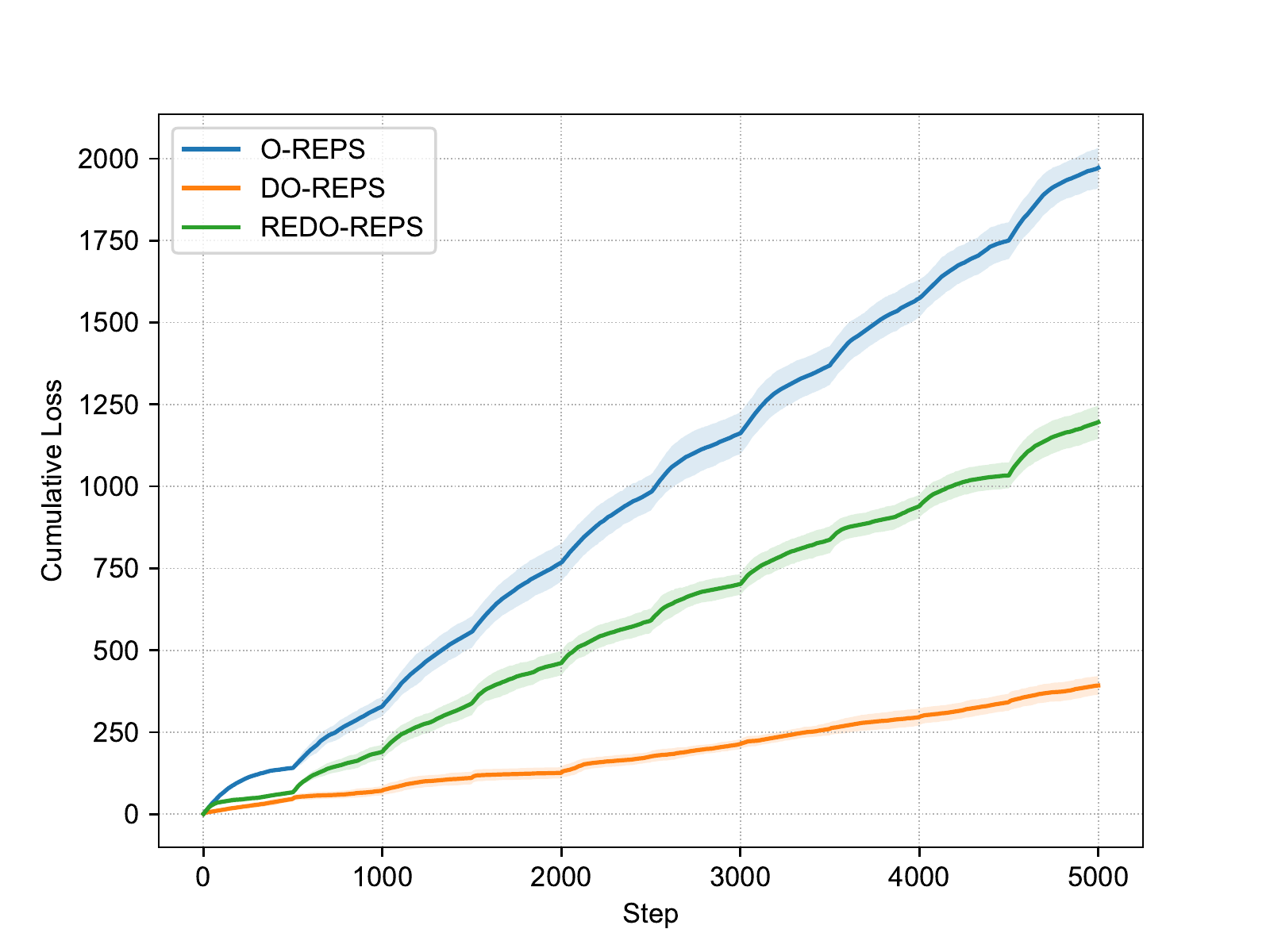}
        \caption{Cumulative loss}
        \label{fig:infinite-result}
    \end{minipage}
\end{figure}

We compare the performance of our proposed REDO-REPS (Algorithm~\ref{alg:infinite}) against the following two contenders: (\romannumeral1) \mbox{O-REPS}~\citep{NIPS'13:MDP-Neu}, online mirror descent over the occupancy measure space, which is design for optimizing static regret; and (\romannumeral2) \mbox{DO-REPS} (Algorithm~\ref{alg:loop-free}), which is designed to optimize the dynamic regret in a meta-base tow-layer structure yet does not take the switching cost into consideration.

Since our algorithm REDO-REPS is designed to optimize the dynamic regret with switching cost $ \sum_{t=1}^T \inner{q_t-\qtc}{\ell_t} + (\tau+1) \sum_{t=2}^T \norm{q_t - q_{t-1}}_1 $ in~Theorem~\ref{thm:reduction}, we define the cumulative surrogate loss as $ \sum_{t=1}^T \inner{q_t}{\ell_t} + (\tau+1) \sum_{t=2}^T \norm{q_t - q_{t-1}}_1$ and report the cumulative surrogate loss of difficult algorithms in Figure~\ref{fig:infinite-switching}. Under this measure, we can see that REDO-REPS clearly achieves the best, \mbox{O-REPS} is comparable, while \mbox{DO-REPS} is not well-behaved. Furthermore, we run the policies of different algorithms in the MDP $10$ times and Figure~\ref{fig:infinite-result} show the the mean and the standard deviation of the cumulative loss. The result reveals that though \mbox{DO-REPS} achieves largest stationary loss with switching cost, it performs best in this problem. This is due to REDO-REPS is designed to optimize the worst case upper bound of the true dynamic regret, which may be overly pessimistic and perform poorly in some situations. How to give a refined and smoothed analysis beyond the worst-case upper bound for infinite-horizon MDP is an interesting question and we leave this as the future work. We note that though REDO-REPS and \mbox{O-REPS} perform close with respect to the stationary loss with switching cost, the true performance of REDO-REPS in the MDP is much better than \mbox{O-REPS}, which shows the effectiveness of our algorithm to some extent.

\section{Conclusion}
\label{sec:conclusion}
In this paper we investigate learning in three foundational online MDPs with adversarially changing loss functions and known transition kernel. We propose novel online ensemble algorithms and establish their dynamic regret guarantees. In particular, the results for episodic (loop-free) SSP are provably minimax optimal in terms of time horizon and certain non-stationarity measure. Furthermore, when the environments are predictable, we enhance our algorithms to achieve better regret for episodic (loop-free) SSP and present impossibility results for infinite-horizon MDPs. 

Our results present an initial resolution for dynamic regret of online MDPs, and there remain many interesting open problems. For example, it remains open whether it is possible to obtain adaptive dynamic regret bound for infinite-horizon MDPs, as discussed in Section~\ref{sec:adaptive-infinite-horizon}. Moreover, extending our results to the bandit feedback and unknown transition setting is an important and challenging future work.

\bibliography{mdp}
\newpage
\appendix
\section{Useful Lemmas Related to Online Mirror Descent}
\label{sec:appendix-lemmas}
In this section, we present several important lemmas used frequently in the analysis of (optimistic) online mirror descent.

\begin{myLemma}[Lemma 3.2 of~\citet{SIAM'93:Bregman}]
    \label{lem:omd}
    Define $q^* = \argmin_{q \in \K}\eta \inner{q}{\ell} + \Dp(q, \qh)$ for some compact set $\K \subseteq \R^d$, convex function $\psi$, an arbitrary point $\ell \in \R^d$, and a point $\qh \in \K$. Then for any $u \in \K$,
    \begin{equation*}
        \inner{q^* - u}{\ell} \leq \frac{1}{\eta}(\Dp(u, \qh) - \Dp(u, q^*) - \Dp(q^*, \qh)).
    \end{equation*}
\end{myLemma}

\begin{myLemma}[Lemma 5 and Proposition 7 of~\citet{COLT'12:variation-Yang}]
    \label{lem:optimistic-omd}
    Let $q_t = \argmin_{q \in \K} {\eta}\inner{q}{m_t} + D_{\psi}(q, \qh_t)$ and $\qh_{t+1} = \argmin_{q \in \K} {\eta}\inner{q}{\ell_t} + D_{\psi}(q, \qh_t)$ for some compact convex set $\K \subseteq \R^d$, convex function $\psi$, arbitrary points $\ell_t, m_t\in \R^d$, and a point $\qh_1 \in \K$. Then, for any $u \in \K$,
    \begin{equation*}
        \inner{q_t - u}{\ell_t}\leq \inner{q_t - \qh_{t+1}}{\ell_t - m_t} + \frac{1}{\eta}(D_{\psi}(u, \qh_t)- D_{\psi}(u, \qh_{t+1}) - D_{\psi}(\qh_{t+1}, q_t)-D_{\psi}(q_t, \qh_t)),
    \end{equation*}
    and when $\psi$ is $\lambda$-strongly convex function \wrt the norm $\|\cdot\|$, we have
    \begin{equation*}
        \norm{q_t - \qh_{t+1}} \leq \frac{1}{\lambda}\norm{\ell_t - m_t}_*.
    \end{equation*}
\end{myLemma}

\begin{myLemma}[Lemma 1 of~\citet{COLT'21:impossible-tuning}]
    \label{lem:impossible-tuning}
    Define $\psi(q) = \sum_{i=1}^d \frac{1}{\eta_i} q_i \log{q_i}$, where $d$ is the dimension of $q$. Let $a_t \in \R^d$ with $a_{t,i} = 32\eta_i(\ell_{t,i}-m_{t,i})^2$, $q_t = \argmin_{q \in \K} \inner{q}{m_t} + D_{\psi}(q, \qh_t)$ and $\qh_{t+1} = \argmin_{q \in \K} \inner{q}{\ell_t+a_t} + D_{\psi}(q, \qh_t)$ for some compact convex set $\K \subseteq \R^d$, arbitrary points $\ell_t, m_t\in \R^d$, and a point $\qh_t \in \K$. Suppose $32 \eta_i \abs{\ell_{t,i}-m_{t,i}}\leq 1$ holds for all $i\in[d]$. Then, for any $u \in \K$,
    \begin{equation*}
        \inner{q_t - u}{\ell_t}\leq \Dp(u, \qh_t) - \Dp(u, \qh_{t+1}) + 32\sumd \eta_i u_i (\ell_{t,i}-m_{t,i})^2 - 16\sumd \eta_i q_{t,i}(\ell_{t,i}-m_{t,i})^2.
    \end{equation*}
\end{myLemma}

\begin{proof}
    This lemma is originally proven by~\citet{COLT'21:impossible-tuning}. For self-containedness, we present their proof and adapt to our notations. By Lemma~\ref{lem:optimistic-omd}, we have
    \begin{equation*}
        \inner{q_t - u}{\ell_t+a_t}\leq \Dp(u, \qh_t) - \Dp(u, \qh_{t+1}) + \inner{q_t-\qh_{t+1}}{\ell_t-m_t+a_t} - \Dp(\qh_{t+1},q_t).
    \end{equation*}
    For the last two terms, define $q^* = \argmax_{q \in \R_{>0}^d}\inner{q_t-q}{\ell_t-m_t+a_t} + \Dp(q, q_t)$, by the optimality of $q^*$, we have: $\ell_t - m_t +a_t = \nabla \psi(q_t) - \nabla \psi(q^*)$ and thus $q_i^* = q_{t,i} e^{-\eta_i (\ell_{t,i}-m_{t,i}+a_{t,i})}$. Therefore, we have
    \begin{align*}
             & \inner{q_t-\qh_{t+1}}{\ell_t-m_t+a_t} - \Dp(\qh_{t+1}, q_t)                                                                     \\
        \leq & \ \inner{q_t-q^*}{\ell_t-m_t+a_t} - \Dp(q^*, q_t)                                                                               \\
        =    & \ \inner{q_t-q^*}{\nabla \psi(q_t) - \nabla \psi(q^*)} - \Dp(q^*, q_t)                                                          \\
        =    & \ \Dp(q_t, q^*) = \sum_{i=1}^d \frac{1}{\eta_i}\left(q_{t,i}\log{\frac{q_{t,i}}{q_i^*}} - q_{t,i}+q_i^*\right)                  \\
        =    & \ \sum_{i=1}^d \frac{q_{t,i}}{\eta_i} \left(\eta_i(\ell_{t,i}-m_{t,i}+a_{t,i})-1+e^{-\eta_i(\ell_{t,i}-m_{t,i}+a_{t,i})}\right) \\
        \leq & \ \sum_{i=1}^d \eta_i q_{t,i}(\ell_{t,i}-m_{t,i}+a_{t,i})^2,
    \end{align*}
    where the last inequality holds due to the fact $e^{-x}-1+x \leq x^2$ for $x \geq -1$ and the condition that $\eta_i \abs{\ell_{t,i}-m_{t,i}}\leq \frac{1}{32}$ such that $\eta_{i} \abs{\ell_{t,i}-m_{t,i}+a_{t,i}}\leq \eta_i \abs{\ell_{t,i}-m_{t,i}}+32\eta_i^2(\ell_{t,i}-m_{t,i})^2 \leq \frac{1}{16}$. Using the definition of $a_t$ and the condition $\eta_i \abs{\ell_{t,i}-m_{t,i}}\leq \frac{1}{32}$, we have
    \begin{align*}
        \inner{q_t-\qh_{t+1}}{\ell_t-m_t+a_t} - \Dp(\qh_{t+1}, q_t) & \leq \sum_{i=1}^d \eta_i q_{t,i}(\ell_{t,i}-m_{t,i}+32\eta_i(\ell_{t,i}-m_{t,i})^2)^2 \\
        & \leq 4 \sum_{i=1}^d \eta_i q_{t,i}(\ell_{t,i}-m_{t,i})^2.
    \end{align*}
    To sum up, we have
    \begin{equation*}
        \inner{q_t - u}{\ell_t+a_t}\leq \Dp(u, \qh_t) - \Dp(u, \qh_{t+1}) + 4 \sum_{i=1}^d \eta_i q_{t,i}(\ell_{t,i}-m_{t,i})^2.
    \end{equation*}
    Finally, moving $\inner{q_t - u}{a_t}$ to the right-hand side of the inequality and using the definition of $a_t$ finishes the proof.
\end{proof}
\section{Proofs for Section \ref{sec:loop-free-SSP} (Episodic Loop-free SSP)}
\label{sec:appendix-loop-free-SSP}
In this section, we provide the proofs for~\pref{sec:loop-free-SSP}. First, we introduce the relationship between the path length of policies and the path length of occupancy measures, and then provide proofs of the minimax dynamic regret of \mbox{DO-REPS} algorithm in Section~\ref{sec:non-stationarity}. Next we present the proofs of the dynamic regret lower bound and finally give the proofs of the enhanced Optimistic \mbox{DO-REPS} algorithm in~\pref{sec:loop-free-adaptive}.

\subsection{Path Length of Policies and Occupancy Measures}
\label{sec:appendix-loop-free-SSP-occupancy-measure}
This part introduces the relationship between the path length of policies and path length of occupancy measures.
\begin{myLemma}
    \label{lem:occupancy-to-policy-loop-free}
    For any occupancy measure sequence $q^{\pi_1}, \ldots, q^{\pi_K}$ induced by the policy sequence $\pi_1, \ldots, \pi_K$, it holds that
    \begin{equation*}
        \sum_{k=2}^K \norm{q^{\pi_k}-q^{\pi_{k-1}}}_1 \leq H \sum_{k=2}^K \suml \norm{\pi_{k,l} - \pi_{k-1,l}}_{1, \infty}.
    \end{equation*}
\end{myLemma}
\begin{proof}
    Let $d^{\pi_k}_l(x) \triangleq \sum_{a} q^{\pi_k}(x,a), \forall x \in X_l, k\in[K]$, we have
    \begin{align}
             & \sum_{x,a} \abs{q^{\pi_k}(x,a)-q^{\pi_{k-1}}(x,a)} \nonumber                                                                                                                           \\
        =    & \suml \sum_{x\in X_l}\sum_{a}\abs{q^{\pi_k}(x,a) - q^{\pi_{k-1}}(x,a)} \nonumber                                                                                                       \\
        =    & \suml \sum_{x\in X_l}\sum_{a}\abs{d^{\pi_k}_l (x)\pi_k(a|x) - d^{\pi_{k-1}}_l (x)\pi_{k-1}(a|x) } \nonumber                                                                            \\
        \leq & \suml \sum_{x\in X_l}\sum_{a}\abs{d^{\pi_k}_l (x)\pi_k(a|x) - d^{\pi_{k-1}}_l (x)\pi_{k}(a|x) } + \abs{d^{\pi_{k-1}}_l (x)\pi_{k}(a|x) - d^{\pi_{k-1}}_l (x)\pi_{k-1}(a|x) } \nonumber \\
        =    & \suml \sum_{x\in X_l}\abs{d^{\pi_k}_l (x)- d^{\pi_{k-1}}_l (x)}\sum_{a}\pi_k(a|x) + \suml \sum_{x\in X_l} d^{\pi_{k-1}}_l (x)\sum_{a}\abs{\pi_{k}(a|x)-\pi_{k-1}(a|x) }\nonumber       \\
        \leq & \suml \norm{d^{\pi_k}_l- d^{\pi_{k-1}}_l}_1 + \suml \norm{\pi_{k,l} - \pi_{k-1,l}}_{1,\infty}\label{eq:loop-free-occ-1},
    \end{align}
    where the first inequality due to the triangle inequality.
    Next, we bound the term $\norm{d^{\pi_k}_l-d^{\pi_{k-1}}_l}_1$. Since $X_0 = \{x_0\}$, we have $\norm{d^{\pi_k}_0-d^{\pi_{k-1}}_0}_1 = 0$. For $l \geq 1$, we have
    \begin{align*}
        \norm{d^{\pi_k}_l - d^{\pi_{k-1}}_l}_1 & = \norm{d^{\pi_k}_{l-1}P^{\pi_k}-d^{\pi_{k-1}}_{l-1} P^{\pi_{k-1}}}_1                                                                          \\
                                               & \leq \norm{d^{\pi_k}_{l-1}P^{\pi_k}-d^{\pi_k}_{l-1}P^{\pi_{k-1}}}_1 + \norm{d^{\pi_k}_{l-1}P^{\pi_{k-1}}-d^{\pi_{k-1}}_{l-1}P^{\pi_{k-1}}}_{1} \\
                                               & \leq \norm{\pi_{k,l-1}- \pi_{k-1,l-1}}_{1,\infty} + \norm{d^{\pi_k}_{l-1} - d^{\pi_{k-1}}_{l-1}}_1,
    \end{align*}
    where the last inequality holds due to Lemma \ref{lem:policy-difference} and Lemma \ref{lem:distribution-difference}. Summing the above inequality from $1$ to $l$, we have
    \begin{equation}
        \label{eq:loop-free-occ-2}
        \norm{d^{\pi_k}_l - d^{\pi_{k-1}}_l}_1 \leq \sum_{i=0}^{l-1}\norm{\pi_{k,i} - \pi_{k-1,i}}_{1, \infty}.
    \end{equation}
    Combining~\eqref{eq:loop-free-occ-1} and~\eqref{eq:loop-free-occ-2}, we obtain
    \begin{align*}
        \sum_{x,a} \abs{q^{\pi_k}(x,a)-q^{\pi_{k-1}}(x,a)} & \leq \sum_{l=0}^{H-1} \sum_{i=0}^{l-1}\norm{\pi_{k,i} - \pi_{k-1,i}}_{1, \infty} + \suml \norm{\pi_{k,l} - \pi_{k-1,l}}_{1, \infty} \\
                                                           & = \sum_{l=0}^{H-1} \sum_{i=0}^{l}\norm{\pi_{k,i} - \pi_{k-1,i}}_{1, \infty}                                                         \\
                                                           & \leq H \suml \norm{\pi_{k,l} - \pi_{k-1,l}}_{1, \infty}.
    \end{align*}
    We complete the proof by summing the inequality over all iterations.
\end{proof}

\subsection{Proof of Lemma \ref{lem:loop-free-ssp-1}}
\label{sec:loop-free-ssp-proof-base}
\begin{proof}
    Denote by $q_{k+1}^\prime = \argmin_{q\in \R^{|X||A|}}\eta \inner{q}{\ell_k}+\Dp(q, q_k)$, or equivalently, $q_{k+1}^\prime(x,a) =q_{k}(x,a)\exp(-\eta \ell_k(x,a))$. By standard analysis of online mirror descent, we have
    \begin{equation}
        \label{eq:omd-analysis}
        \begin{aligned}
            \sumk\inner{q_k - \qkc}{\ell_k} & = \sumk \inner{q_k-q_{k+1}^\prime}{\ell_k} +  \inner{q_{k+1}^\prime-\qkc}{\ell_k}                                           \\
                                            & \leq \sumk \inner{q_k-q_{k+1}^\prime}{\ell_k} + \frac{1}{\eta}\sumk \left(\Dp(\qkc, q_k) - \Dp(\qkc, q_{k+1}^\prime)\right) \\
                                            & \leq \sumk \inner{q_k-q_{k+1}^\prime}{\ell_k} +  \frac{1}{\eta}\sumk \left(\Dp(\qkc, q_k) - \Dp(\qkc, q_{k+1})\right),
        \end{aligned}
    \end{equation}
    where the first inequality holds due to Lemma~\ref{lem:omd} and the last inequality holds due to Pythagoras theorem. For the first term, applying the inequality $1 - e^{-x} \leq x$, we obtain
    \begin{equation}
        \label{eq:loop-free-base-1}
        \sumk \inner{q_k-q_{k+1}^\prime}{\ell_k} \leq \eta \sumk \sum_{x,a} q_k(x,a) \ell_k^2(x,a) \leq \eta \sumk \sum_{x,a} q_k(x,a)\leq \eta HK = \eta T.
    \end{equation}
    For the last term, we have
    \begin{equation}
        \label{eq:bregman-difference-1}
        \begin{aligned}
             & \quad \sumk\left(D_{\psi}(\qkc, q_k)- D_{\psi}(\qkc, q_{k+1})\right)                                                                                                        \\ & = \Dp(q^{\pi_1^\C}, q_1) + \sum_{k=2}^K \left(\Dp(\qkc, q_k) - \Dp(q^{\pi_{k-1}^\C}, q_k) \right)                                                                               \\
             & = \Dp(q^{\pi_1^\C}, q_1) + \sum_{k=2}^K\sum_{x,a}\left(\qkc(x,a)\log{\frac{\qkc(x,a)}{q_k(x,a)}} - q^{\pi_{k-1}^\C}(x,a)\log{\frac{q^{\pi_{k-1}^\C}(x,a)}{q_k(x,a)}}\right) \\
             & = \Dp(q^{\pi_1^\C}, q_1) + \sum_{k=2}^K\sum_{x,a}\left(\qkc(x,a)-q^{\pi_{k-1}^\C}(x,a)\right)\log{\frac{1}{q_k(x,a)}} + \psi(q^{\pi_K^\C})-\psi(q^{\pi_1^\C})               \\
             & \leq  \sum_{k=2}^K \norm{\qkc - q^{\pi_{k-1}^\C}}_1 \log \frac{1}{\alpha}  + \Dp(q^{\pi_1^\C}, q_1) + \psi(q^{\pi_K^\C})-\psi(q^{\pi_1^\C}),
        \end{aligned}
    \end{equation}
    where the last inequality holds due to $q_k(x,a) \geq \alpha$, since $q_k \in \Delta(M, \alpha)$ for all $k$ and $x, a$.
    It remains to bound the last two terms. Since $q_1$ minimize $\psi$ over $\Delta(M,\alpha)$, we have $\inner{\nabla \psi(q_1)}{q^{\pi_1^\C}-q_1} \geq 0$, and thus
    \begin{equation}
        \label{eq:bregman-difference-2}
        \Dp(q^{\pi_1^\C}, q_1) + \psi(q^{\pi_K^\C})-\psi(q^{\pi_1^\C}) \leq \psi(q^{\pi_K^\C})-\psi(q_1) \leq  \sum_{x,a} q_1(x,a)\log \frac{1}{q_1(x,a)} \leq H \log{\frac{|X||A|}{H}}.
    \end{equation}
    Combining\eqref{eq:loop-free-base-1},~\eqref{eq:bregman-difference-1} and~\eqref{eq:bregman-difference-2}, we obtain
    \begin{equation*}
        \sumk\inner{q_k - \qkc}{\ell_k} \leq \eta T + \frac{1}{\eta}\left( H\log{\frac{|X||A|}{H}} + \Pb_T \log{\frac{1}{\alpha}} \right),
    \end{equation*}
    where $\Pb_T =\sum_{k=2}^K \norm{\qkc - q^{\pi_{k-1}^\C}}_{1}$. This completes the proof.
\end{proof}

\subsection{Proof of Theorem~\ref{thm:loop-free-non-stationarity}}
\label{sec:loop-free-ssp-proof-overall}
\begin{proof}
    Without loss of generality, we assume that all states are reachable with positive probability under the uniform policy $\pi^u(a|x) = 1/|A|, \forall x \in \X, a\in \A$ (otherwise remove the unreachable states since they are unreachable by any policy). Assume $T$ is large enough such that the occupancy measure of $\pi^u$ satisfies $q^{\pi^u} \in \Delta(M, \frac{1}{T})$, then define $u_k = (1-\frac{1}{T})\qkc + \frac{1}{T}q^{\pi^u} \in \Delta(M,\frac{1}{T^2})$, we have
    \begin{align}
        \sumk\inner{q_k - \qkc}{\ell_k} & = \sumk \inner{q_k - u_k}{\ell_k} + \frac{1}{T}\sumk \inner{q^{\pi^u} - \qkc}{\ell_k}\nonumber                                                                    \\
                                        & \leq \sumk\inner{q_k - u_k}{\ell_k}+2\nonumber                                                                                                                    \\
                                        & \leq \underbrace{\sumk\inner{q_k - q_{k,i} }{\ell_k}}_{\meta} +\underbrace{\sumk\inner{q_{k,i}-u_k}{\ell_k}}_{\base}+2,\label{eq:decomposition-loop-free-dynamic}
    \end{align}
    where the first inequality follows from the definition that $u_k = (1-\frac{1}{T})\qkc + \frac{1}{T}q^{\pi^u}$ and the last inequality holds for any index $i$.

    \paragraph{Upper bound of base-regret.} Since $u_k \in \Delta(M, \frac{1}{T^2}), \forall k \in [K]$, from Lemma~\ref{lem:loop-free-ssp-1} we have
    \begin{align*}
        \small{\base} \leq \eta T + \frac{H\log{\frac{|X||A|}{H}} + 2\sum_{k=2}^K \norm{u_k-u_{k-1}}_1 \log{T}}{\eta}\leq\eta T + \frac{H\log{\frac{|X||A|}{H}} + 2\Pb_T \log{T}}{\eta},
    \end{align*}
    where the last inequality holds due to $\sum_{k=2}^K \norm{u_k-u_{k-1}}_1 \leq \sum_{k=2}^K \norm{\qkc-q^{\pi_{k-1}^\C}}_1 = \Pb_T$. It is clear that the optimal step size is $\eta^*  = \sqrt{(H\log{({|X||A|}/{H})} + 2\Pb_T \log{T})/T}$. From the definition of $\Pb_T=\sum_{k=2}^K \norm{\qkc-q^{\pi_{k-1}^\C}}_1$, we have $0 \leq \Pb_T \leq 2KH=2T$. Thus, the possible range of the optimal step size is
    \begin{align*}
        \eta_{\min} =  \sqrt{\frac{H\log(|X||A|/H)}{T}}, \mbox{ and } \eta_{\max} = \sqrt{\frac{H\log(|X||A|/H)}{T} + 4\log{T}}.
    \end{align*}
    By the construction of the candidate step size pool  $\H = \{\eta_i = 2^{i-1} \sqrt{K^{-1}\log(|X||A|/H)} \mid i\in [N] \}$, where $N=\ceil{\frac{1}{2} \log(1+\frac{4K\log{T}}{\log(|X||A|/H)})}+1$, we know that the step size therein is monotonically increasing with respect to the index, in particular
    \begin{align*}
        \eta_{1} = \sqrt{\frac{H\log(|X||A|/H)}{T}} = \eta_{\min}, \mbox{ and } \eta_{N} \geq \sqrt{\frac{H\log(|X||A|/H)}{T} + 4\log{T}} = \eta_{\max}.
    \end{align*}
    Thus, we ensure there exists a base-learner $i^*$ whose step size satisfies $\eta_{i^*} \leq \eta^* \leq \eta_{i^*+1} =2\eta_{i^*}$. Since the regret decomposition in~\eqref{eq:decomposition-loop-free-dynamic} holds for any $i \in [N]$, we choose the base-learner ${i^*}$ to analysis to obtain a sharp bound.
    \begin{equation}
        \label{eq:base-loop-free}
        \begin{aligned}
            \base & \leq \eta_{i^*} T + \frac{H\log(|X||A|/H) + 2\Pb_T \log{T}}{\eta_{i^*}} \\
                  & \leq \eta^* T + \frac{2(H\log(|X||A|/H) + 2\Pb_T \log{T})}{\eta^*}      \\
                  & = 3 \sqrt{T\left(H\log(|X||A|/H) + 2\Pb_T \log{T}\right)},
        \end{aligned}
    \end{equation}
    where the second inequality holds due to $\eta_{i^*} \leq \eta^* \leq \eta_{i^*+1} =2\eta_{i^*}$ and the last equality holds by substituting the optimal step size $\eta^* = \sqrt{\frac{H\log({|X||A|}/{H}) + 2\Pb_T \log{T}}{T}}$.

    \paragraph{Upper bound of meta-regret.} From the construction that $h_{k,i} = \inner{q_{k,i}}{\ell_t},\forall k\in[K], i\in[N]$, the meta-regret can be written in the following way:
    \begin{equation*}
        \meta = \sumk \inner{q_k - q_{k,i}}{\ell_k} = \sumk \inner{\sumi p_{k,i}q_{k,i}- q_{k,i}}{\ell_k} = \sumk \inner{p_k - e_i}{h_k}
    \end{equation*}
    It is known that the update $p_{k+1,i} \propto \exp(-\varepsilon \sum_{s=1}^k h_{s,i}), \forall i \in [N]$ is equal to the update $p_{k+1}= \argmin_{p \in \Delta_N} \varepsilon \inner{p}{h_k}+\Dp(p, p_k)$, where $\psi(p) = \sum_{i=1}^N p_i \log p_i$ is the standard negative entropy. This can be further reformulated solving the unconstrained optimization problem $p_{k+1}^\prime = \argmin_{p} \varepsilon\inner{p}{h_k}+\Dp(p, p_k)$ at first and then projecting $p_{k+1}^\prime$ to the simplex $\Delta_N$ as $p_{k+1} = \argmin_{p \in \Delta_N} \Dp(p, p_{k+1}^\prime)$. By standard analysis of OMD, we have
    \begin{align*}
        \sumk \inner{p_k - e_i}{h_k} & \leq \sumk \inner{p_k - p_{k+1}^\prime}{h_k} + \sumk \inner{p_{k+1}^\prime - e_i}{h_k}                                \\
                                     & \leq \sumk \inner{p_k - p_{k+1}^\prime}{h_k} + \frac{1}{\varepsilon} \sumk (\Dp(e_i, p_k) - \Dp(e_i, p_{k+1}^\prime)) \\
                                     & \leq \sumk \inner{p_k - p_{k+1}^\prime}{h_k} + \frac{1}{\varepsilon} \sumk (\Dp(e_i, p_k) - \Dp(e_i, p_{k+1}))        \\
                                     & \leq \sumk \inner{p_k - p_{k+1}^\prime}{h_k} + \frac{1}{\varepsilon}\Dp(e_i, p_1),
    \end{align*}
    where the second inequality holds due to Lemma~\ref{lem:omd} and the third inequality holds due to Pythagoras theorem. Using the fact that $1 - e^{-x} \leq x$ and the definition that $p_{1,i}=1/N, h_{k,i}\leq H, \forall k\in[K], i\in[N]$, we have
    \begin{equation*}
        \sumk \inner{p_k - p_{k+1}^\prime}{h_k} + \frac{1}{\varepsilon}\Dp(e_i, p_1) \leq \varepsilon \sumk \sumi p_{k,i} h_{k,i}^2 + \frac{\ln N}{\varepsilon} \leq \varepsilon HT + \frac{\ln N}{\varepsilon}.
    \end{equation*}
    Therefore, for any base-learner $i \in [N]$, we have
    \begin{equation*}
        \sumk \inner{q_k - q_{k,i}}{\ell_k} = \sumk \inner{p_k - e_i}{h_k} \leq \varepsilon HT + \frac{\log{N}}{\varepsilon}.
    \end{equation*}
    In particular, for the best base-learner $i^* \in [N]$, we have
    \begin{equation}
        \label{eq:meta-loop-free}
        \meta = \sumk \inner{q_k - q_{k,i^*}}{\ell_k} \leq \varepsilon HT + \frac{\log{N}}{\varepsilon} = \sqrt{HT\log{N}},
    \end{equation}
    where the last equality holds due to the setting $\varepsilon = \sqrt{(\log{N})/(HT)}$.

    \paragraph{Upper bound of overall dynamic regret.} Combining~\eqref{eq:decomposition-loop-free-dynamic},~\eqref{eq:base-loop-free} and~\eqref{eq:meta-loop-free}, we obtain
    \begin{align*}
        \sumk\inner{q_k - \qkc}{\ell_k} & \leq \base + \meta                                                                 \\
                                        & \leq  3\sqrt{T\left(H\log(|X||A|/H) + 2\Pb_T \log{T}\right)} + \sqrt{HT\log{N}} +2 \\
                                        & \leq \O\left(\sqrt{HT\left(\log(|X||A|/H) + P_T \log{T}\right)}\right),
    \end{align*}
    where the last equality is due to $\Pb_T \leq H P_T$ by Lemma~\ref{lem:occupancy-to-policy-loop-free}, $N=\ceil{\frac{1}{2} \log(1+\frac{4K\log{T}}{\log(|X||A|/H)})}+1$. This finishes the proof.
\end{proof}

\subsection{Proof of Theorem~\ref{Thm:lower-bound-loop-free}}
\label{sec-appendix:lower-bound-loop-free}
\begin{proof}
    The proof is similar to the proof of the minimax lower bound of dynamic regret for online convex optimization~\citep{NIPS'18:Zhang-Ader}. For any $\gamma \in [0,2T]$, we first construct a piecewise-stationary comparator sequence, whose path length is smaller than $\gamma$, then we split the whole time horizon into several pieces, where the comparator is fixed in each piece. By this construction, we can apply the existed minimax static regret lower bound of episodic loop-free SSP~\citep{NIPS'13:MDP-Neu} in each piece, and finally sum over all pieces to obtain the lower bound for the dynamic regret.

    Denote by $R_K(\Pi, \F, \gamma)$ the minimax dynamic regret, which is defined as
    \begin{equation*}
        R_K(\Pi, \F, \gamma) = \inf_{\pi_1 \in \Pi}\sup_{\ell_1 \in \F} \ldots \inf_{\pi_K \in \Pi}\sup_{\ell_K \in \F} \left(\sumk \inner{\qk}{\ell_k} - \min_{(\pi_1^\C, \ldots, \pi_K^\C) \in \U(\gamma)} \sumk \inner{\qkc}{\ell_k}\right)
    \end{equation*}
    where $\Pi$ denotes the set of all policies, $\F$ denotes the set of loss functions $\ell \in \R_{[0,1]}^{|X||A|}$, and $\U(\gamma) = \{(\pi_1^\C, \ldots, \pi_K^\C) \mid \forall k \in [K], \pi_k^\C \in \Pi, \mbox{ and } \Pb_T = \sum_{k=2}^K \norm{\qkc - q^{\pi_{k-1}^\C}}_1 \leq \gamma\}$ is the set of feasible policy sequences with the path length $\Pb_T$ of the occupancy measures less than $\gamma$.

    We first consider the case of $\gamma \leq 2H$. Then we can directly utilize the established lower bound of the static regret for learning in episodic loop-free SSP~\citep{NIPS'13:MDP-Neu} as a natural lower bound of dynamic regret,
    \begin{equation}
        \label{eq:lower-bound-loop-free-1}
        R_K(\Pi, \F, \gamma) \geq C_1 H\sqrt{K\log(|X||A|)},
    \end{equation}
    where $C_1=0.03$ is the constant appeared in the lower bound of static regret.

    We next deal with the case that $\gamma \geq 2H$. Without loss of generality, we assume $L= \ceil{\gamma/2H}$ divides $K$ and split the whole time horizon into $L$ pieces equally. Next, we construct a special policy sequence in $\U(\gamma)$ such that the policy sequence is fixed within each piece and only changes in the split point. Since the sequence changes at most $L-1\leq\gamma/2H$ times and the variation of the occupancy measure at each change point is at most $2H$, the path length $\Pb_T$ of the occupancy measures does not exceed $\gamma $. As a result, we have
    \begin{equation}
        \label{eq:lower-bound-loop-free-2}
        R_K(\Pi, \F, \gamma) \geq L C_1 H \sqrt{\frac{K}{L}\log(|X||A|)} \geq \frac{\sqrt{2}C_1}{2}\sqrt{HK\gamma \log(|X||A|)}.
    \end{equation}
    Combining~\eqref{eq:lower-bound-loop-free-1} and~\eqref{eq:lower-bound-loop-free-2}, we obtain the final lower bound
    \begin{equation*}
        R_K(\Pi, \F, \gamma) \geq \frac{\sqrt{2}C_1}{2} \sqrt{H K\log(|X||A|)}\max(\sqrt{2H}, \sqrt{\gamma}) \geq \Omega (\sqrt{HK(H+\gamma)\log(|X||A|)}),
    \end{equation*}
    which finishes the proof.
\end{proof}

\subsection{Proof of Theorem~\ref{thm:loop-free-adaptivity}}
\label{sec:proof-loop-free-adaptivity}
\begin{proof}
    Similar to the argument in Appendix~\ref{sec:loop-free-ssp-proof-overall}, assume all states are reachable with positive probability under the uniform policy $\pi^u(a|x) = 1/|A|, \forall x \in \X, a\in \A$ and $T$ is large enough such that the occupancy measure of $\pi^u$ satisfies $q^{\pi^u} \in \Delta(M, \frac{1}{T})$, then define $u_k = (1-\frac{1}{T})\qkc + \frac{1}{T}q^{\pi^u} \in \Delta(M,\frac{1}{T^2})$, we have
    \begin{align}
        \sumk\inner{q_k - \qkc}{\ell_k} & = \sumk \inner{q_k - u_k}{\ell_k} + \frac{1}{T}\sumk \inner{q^{\pi^u} - \qkc}{\ell_k}\nonumber                                                         \\
                                        & \leq \sumk\inner{q_k - u_k}{\ell_k}+2 \nonumber                                                                                                        \\
                                        & = \underbrace{\sumk\inner{q_k - q_{k,i} }{\ell_k}}_{\meta} +\underbrace{\sumk\inner{q_{k,i}-u_k}{\ell_k}}_{\base}+2\label{eq:decomposition-loop-free},
    \end{align}
    where the first inequality follows from the definition that $u_k = (1-\frac{1}{T})\qkc + \frac{1}{T}q^{\pi^u}$ and the last inequality holds for any index $i$.

    \paragraph{Upper bound of base-regret.} From Lemma \ref{lem:strongly-convexity}, we ensure $\psi(q)=\sum_{s,a}q(x,a)\log{q(x,a)}$ is $\frac{1}{ H}$ strongly convex for $q \in \Delta(M, \frac{1}{T^2})$. By Lemma \ref{lem:optimistic-omd}, we obtain
    \begin{equation*}
        \sumk\inner{q_k - u_k}{\ell_k} \leq \eta H \sum_{k=1}^K \norm{\ell_k-m_k}_\infty^2 + \frac{1}{\eta}\sumk \left(\Dp(u_k, \qh_k) - \Dp(u_k, \qh_{k+1})\right).
    \end{equation*}
    Similar to the argument in \eqref{eq:bregman-difference-1}, \eqref{eq:bregman-difference-2} and $\sum_{k=2}^K \norm{u_k - u_{k-1}}_1 \leq \sum_{k=2}^K \norm{\qkc - q^{\pi_{k-1}^\C}}_1=\Pb_T$, we have
    \begin{equation*}
        \sumk \left(\Dp(u_k, \qh_k) - \Dp(u_k, \qh_{k+1})\right) \leq H\log(|X||A|/H) + 2\Pb_T \log{T}.
    \end{equation*}
    Therefore, we obtain
    \begin{equation*}
        \base \leq \eta H V_K + \frac{H\log{\frac{|X||A|}{H}} + 2\Pb_T \log{T}}{\eta}\leq \eta H (1+V_K) + \frac{H\log{\frac{|X||A|}{H}} + 2\Pb_T \log{T}}{\eta}.
    \end{equation*}
    It is clear the optimal step size is $\eta^* = \sqrt{\frac{H\log{({|X||A|}/{H}}) +2 \Pb_T \log{T}}{H (1+V_K)}}$. Meanwhile, from the definition of $\Pb_T$ and $\Vb_K$, we have $0 \leq \Pb_T \leq 2T$ and $0 \leq V_K \leq K-1$. Thus, the possible range of the optimal step size is
    \begin{align*}
        \eta_{\min} = \sqrt{\frac{H\log(|X||A|/H)}{T}}, \mbox{ and } \eta_{\max} = \sqrt{{\log(|X||A|/H)} + 4K\log{T}}.
    \end{align*}
    By the construction of the candidate step size pool $\H = \{\eta_i = 2^{i-1} \sqrt{K^{-1}\log(|X||A|/H)} \mid  i\in[N] \}$, where $N=\ceil{\frac{1}{2} \log(K+\frac{4K^2\log{T}}{\log(|X||A|/H)})}+1$, we know that the step size therein is monotonically increasing with respect to the index, in particular
    \begin{align*}
        \eta_{1} = \sqrt{\frac{H\log(|X||A|/H)}{T}} = \eta_{\min}, \mbox{ and } \eta_{N} \geq \sqrt{\log(|X||A|/H) + 4K\log{T}} = \eta_{\max}.
    \end{align*}
    Thus, we ensure there exists a base-learner $\B_{i^*}$ whose step size satisfies $\eta_{i^*} \leq \eta^* \leq \eta_{i^*+1} =2\eta_{i^*}$. Since the regret decomposition in~\eqref{eq:decomposition-loop-free} holds for any $i \in [N]$, we choose the base-learner $\B_{i^*}$ to analysis to obtain a sharp bound.
    \begin{equation}
        \label{eq:base-loop-free-1}
        \begin{aligned}
            \base & \leq \eta_{i^*} H (1+V_K) + \frac{H\log(|X||A|/H) + 2\Pb_T \log{T}}{\eta_{i^*}} \\
                  & \leq \eta^* H (1+V_K) + \frac{2(H\log(|X||A|/H) + 2\Pb_T \log{T})}{ \eta^*}     \\
                  & = 3 \sqrt{H (1+V_K)\left(H\log(|X||A|/H) + 2\Pb_T \log{T}\right)},
        \end{aligned}
    \end{equation}
    where the last equality holds by substituting the step size $\eta^* = \sqrt{\frac{H\log{({|X||A|}/{H})} + 2\Pb_T \log{T}}{H (1+V_K)}}$.

    \paragraph{Upper bound of meta-regret.}It is known that the update $p_{k+1,i} \propto \exp(-\varepsilon \sum_{s=1}^k (h_{s,i}+M_{k+1, i})), \forall i \in [N]$ is equal to the update $p_{k+1}= \argmin_{p \in \Delta_N} \varepsilon\inner{p}{M_{k+1}}+\Dp(p, p_{k+1}^\prime)$ and $p_{k+1}^\prime= \argmin_{p \in \Delta_N} \varepsilon\inner{p}{h_k}+\Dp(p, p_k)$, where $\psi(p) = \sum_{i=1}^N p_i \log p_i$ is the standard negative entropy. Since $\psi(p) = \sum_{i=1}^N p_i \log{p_i}$ is $1$ strongly convex \wrt $\|\cdot\|_1$ for $p\in \Delta_N$. By Lemma \ref{lem:optimistic-omd}, we obtain
    \begin{align}
        \meta & = \sumk \inner{q_k - q_{k,i}}{\ell_k} = \sumk \inner{\sumi p_{k,i}q_{k,i}- q_{k,i}}{\ell_k} = \sumk \inner{p_k - e_i}{h_k} \nonumber                              \\
              & \leq \varepsilon \left(\sum_{k=1}^K\norm{h_k-M_k}_\infty^2\right) + \frac{\Dp(e_i, \ph_1)}{\varepsilon}\label{eq:meta-loop-free-1}                                \\
              & \leq \varepsilon \left(\sum_{k=1}^K \max_{i\in [N]} \inner{q_{k,i}}{\ell_k-m_k}^2\right) + \frac{\log{N}}{\varepsilon} \label{eq:meta-loop-free-2}                \\
              & \leq \varepsilon \left(\sum_{k=1}^K \max_{i\in [N]} \norm{q_{k,i}}_1^2 \norm{\ell_k-m_k}_\infty^2\right) + \frac{\log{N}}{\varepsilon}\label{eq:meta-loop-free-3} \\
              & \leq \varepsilon H^2 (1+V_K) + \frac{\log{N}}{\varepsilon} \label{eq:meta-loop-free-4}                                                                            \\
              & \leq H \sqrt{(1+V_K)\log{N}}\label{eq:meta-loop-free-5}
    \end{align}
    where~\eqref{eq:meta-loop-free-1} holds due to Lemma \ref{lem:optimistic-omd}, \eqref{eq:meta-loop-free-2} follows from the definition of $h_{k,i} = \inner{q_{k,i}}{\ell_k}$ and $M_k=\inner{q_{k,i}}{m_k}, \forall k\in[K], i\in[N]$, \eqref{eq:meta-loop-free-3} holds due to H\"{o}lder's inequality, \eqref{eq:meta-loop-free-4} is due to $\norm{q_{k,i}}_1 = H, \forall k\in[K], i\in[N]$, and~\eqref{eq:meta-loop-free-5} makes use of the setting $\varepsilon  = \sqrt{(\log{N})/(H^2 (1+V_K))}$.

    \paragraph{Upper bound of overall dynamic regret.} Combining~\eqref{eq:decomposition-loop-free},~\eqref{eq:base-loop-free-1} and~\eqref{eq:meta-loop-free-5}, we obtain
    \begin{align*}
        \sumk\inner{q_k - \qkc}{\ell_k} & \leq 3 \sqrt{H(1+V_K)\left(H\log(|X||A|/H) + 2\Pb_T \log{T}\right)} + H \sqrt{(1+V_K)\log{N}} +2 \\
                                        & \leq \O\left(H\sqrt{(1+V_K)\left(\log(|X||A|/H) + P_T \log{T}\right)}\right),
    \end{align*}
    where the last equality is due to $\Pb_T \leq H P_T$ by Lemma~\ref{lem:occupancy-to-policy-loop-free} and $N=\ceil{\frac{1}{2} \log(K+\frac{4K^2\log{T}}{\log(|X||A|/H)})}+1$. This finishes the proof.
\end{proof}

\subsection{Useful Lemmas}
In this part, we present some basic lemmas in episodic loop-free SSP. For any policy $\pi(a|x)$, we define $P^\pi$ to be the transition matrix induced by $\pi$, where the component $[P^\pi]_{x,x^\prime}$ is the transition probability from $x$ to $x^\prime$, \ie $[P^\pi]_{x,x^\prime} = \sum_a \pi(a|x)P_{x,x^\prime}^a$. Then, we have the following useful lemmas.
\begin{myLemma}[Lemma 6.3 of~\citet{MatOR'09:online-MDP}]
    \label{lem:policy-difference}
    For any policies $\pi$ and $\pi$ and any state distribution $d$, we have
    \begin{equation*}
        \norm{dP^\pi - dP^{\pi^\prime}}_1 \leq \norm{\pi - \pi^\prime}_{1, \infty}.
    \end{equation*}
\end{myLemma}
\begin{proof}
    Consider the case when $d$ is a delta function on $x$. The difference in the next state distributions, $\norm{dP^\pi - dP^{\pi^\prime}}_1$, is
    \begin{align*}
        \sum_{x^{\prime}}\Big|[P^{\pi}]_{x, x^{\prime}}-[P^{\pi^{\prime}}]_{x, x^{\prime}}\Big| & =\sum_{x^{\prime}} \sum_{a}\abs{P(x^\prime | x,a)\left(\pi(a | x)-\pi^{\prime}(a | x)\right)} \\
                                                                                                & \leq \sum_{x^{\prime}, a} P(x^\prime | x, a)\abs{\pi(a | x)-\pi^{\prime}(a | x)}              \\
                                                                                                & =\sum_{a}\abs{\pi(a | x)-\pi^{\prime}(a | x)}.
    \end{align*}
    Linearity of expectation leads to the result for arbitrary $d$.
\end{proof}

\begin{myLemma}
    \label{lem:distribution-difference}
    For any state distribution $d$ and $d^\prime$, and any policy $\pi$, we have
    \begin{equation}
        \norm{dP^\pi - d^\prime P^\pi}_1 \leq \norm{d - d^\prime}_1.
    \end{equation}
\end{myLemma}
\begin{proof}
    Note that the relationship that $d(x^\prime) = \sum_{x} d(x)P_{x,x^\prime}^\pi$, therefore, we have
    \begin{align*}
        \norm{dP^\pi - d^\prime P^\pi}_1 & = \sum_{x^\prime}\abs{\sum_x d(x)P_{x,x^\prime}^\pi - d^\prime(x)P_{x,x^\prime}^\pi} \leq \sum_{x^\prime}\sum_x \abs{d(x)P_{x,x^\prime}^\pi - d^\prime(x)P_{x,x^\prime}^\pi} \\&  = \sum_{x^\prime}\sum_x \abs{d(x) - d^\prime(x)}P_{x,x^\prime}^\pi = \sum_x \abs{d(x) - d^\prime(x)}\sum_{x^\prime} P_{x,x^\prime}^\pi\\
                                         & = \sum_x\abs{d(x) - d^\prime(x)}
        = \norm{d - d^\prime}_1.
    \end{align*}
    This finishes the proof.
\end{proof}
Finally, we show the lemma below which shows the strongly convexity of the regularizer.
\begin{myLemma}
    \label{lem:strongly-convexity}
    $\psi(w) = \sum_{i=1}^d w_i \log w_i $ is $\frac{1}{H}$-strongly convex \wrt $\|\cdot\|_1$ for $\{w \in \R_{\geq 0}^d \mid \sum_{i=1}^d w_i = H\}$.
\end{myLemma}
\begin{proof}
    For any  $y, z \in \{w \in \R_{\geq 0}^d \mid \sum_{i=1}^N w_i = H\}$, we have $\frac{y}{H}, \frac{z}{H} \in \{w \in \R_{\geq 0}^d \mid \sum_{i=1}^d w_i = 1\}$. Then, it holds that
    \begin{equation*}
        \psi(y) - \psi(z) - \inner{\nabla \psi(y)}{y-z} = \sum_{i=1}^d y_i\log{\frac{y_i}{z_i}} = H\sum_{i=1}^d \frac{y_i}{H}\log{\frac{y_i/H}{z_i/H}}\geq \frac{1}{2 H} \norm{y-z}_1^2,
    \end{equation*}
    where the last inequality holds due to Pinsker's Inequality. This finishes the proof.
\end{proof}
\section{Proofs for Section~\ref{sec:episodic-ssp} (Episodic SSP)}
\label{sec:appendix-general-SSP}
In this section, we first give the impossible result to bound the path length of occupancy measures by the path length of policies. Next we provide proofs of the dynamic regret of \mbox{CODO-REPS} algorithm and the lower bound of dynamic regret in Section~\ref{sec:dynamic-ssp}. Finally we give the proofs of the enhanced Optimistic \mbox{CODO-REPS} algorithm in Section~\ref{sec:ssp-adaptive}.

\subsection{Path Length of Policies and Occupancy Measures}
\label{sec:appendix-impossible-SSP}
In the following, we give the impossible result to bound the path length of the occupancy measures by the path length of the corresponding policies.
\begin{myThm}
    \label{thm:occupancy-to-policy-ssp}
    For any $H_*>1$ and any positive integer $c > 0$, there exists an SSP instance with $|X| = 2c+1$ states, $|A|=2$ actions and a policy sequence $\pi_1^\C, \ldots, \pi_K^\C$ with largest expected hitting time $H_*$ such that $\Pb_K \geq c P_K$.
\end{myThm}

\begin{proof}
    For any $H_*>1$ and any positive integer $c > 0$, we construct an episodic SSP with $n+1$ states $X = \{x_0, \ldots, x_n\}$ with $n=2c$ and two actions $A=\{a_1, a_2\}$. Let the transition kernel be deterministic and the state transitions are specified in Figure~\ref{fig:ssp}.
    \begin{figure}[!h]
        \centering
        \includegraphics[height=0.13\textwidth,trim=11cm 7.7cm 11cm 7.3cm,clip]{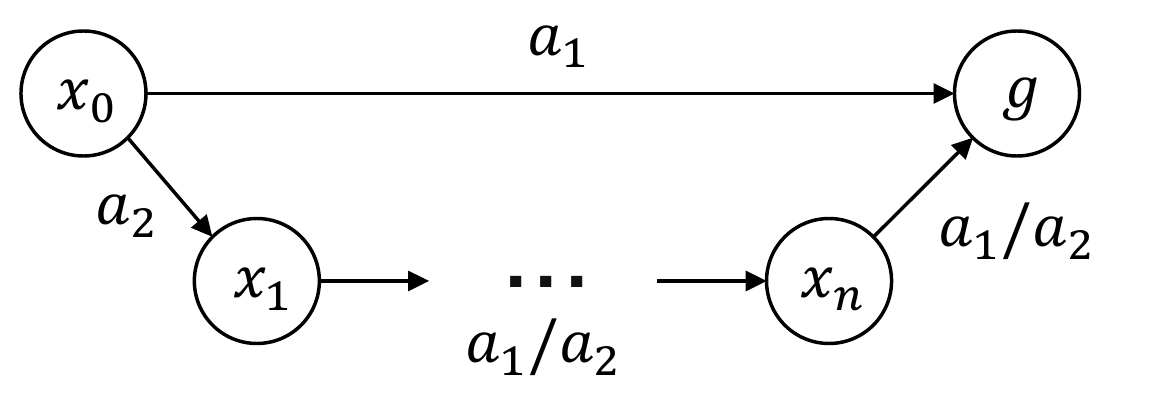}
        \caption{An illustration of state transitions of the constructed hard instance.}
        \label{fig:ssp}
    \end{figure}

    Specifically, taking $a_1$ and $a_2$ in initial state $x_0$ leads to the state $g$ and $x_1$ respectively. Taking any action in state $x_i$ leads to state $x_{i+1}, \forall i \in [n-1]$ and  taking any action in state $x_n$ leads to the goal state state $g$. Then, we consider two policies $\pi$ and $\pi^\prime$ with $\pi(a_1|x_i) = 1, \forall i \in \{0\} \cup [n]$ and $\pi^\prime(a_1|x_0) = 1-\epsilon, \pi^\prime(a_1|x_i) = 1, \forall i \in [n]$. It is clear that $\norm{\pi - \pi^\prime}_{1,\infty}=2\epsilon$ and $H^\pi(x_0) = 1, H^{\pi^\prime}(x_0) = 1+\epsilon n$.
    For any $H_*>1$ and $c>0$, let $\epsilon = (H_*-1)/n$, we have $H^{\pi^\prime} = 1+\epsilon n = H_*$, \ie the largest hitting time of $\pi$ and $\pi^\prime$ is $H_*$. Then we consider the occupancy measure discrepancy of $\pi$ and $\pi^\prime$. It is easy to verify
    \[
        \sum_{x,a}\abs{q^\pi(x,a)-q^{\pi^\prime}(x,a)} = \epsilon + \epsilon(n+1) = \epsilon(n+2) = 2\epsilon(c+1) = (c+1)\norm{\pi - \pi^\prime}_{1,\infty}.
    \]
    Therefore, we have $\norm{q^\pi - q^{\pi^\prime}}_1 \geq c \norm{\pi-\pi^\prime}_{1,\infty}$. Thus, the policy sequence $\pi, \pi^\prime, \pi, \pi^\prime, \ldots$ satisfies $\Pb_K = K \norm{q^\pi - q^{\pi^\prime}}_1 \geq cK \norm{\pi-\pi^\prime}_{1,\infty}=c P_K$, which completes the proof.
\end{proof}

\subsection{Proof of Lemma~\ref{lem:CDO-base}}
\label{sec:proof-base-ssp}
\begin{proof}
    Since $\eta \leq \frac{1}{64}$, we ensure that $32 \eta \abs{\ell_{k,i}}\leq 1, \forall k\in[K], i\in[|X||A|]$. Taking $m_k = 0$ in Lemma~\ref{lem:impossible-tuning}, we obtain
    \begin{align}
        \label{eq:ssp-lemma-1}
        \sumk\inner{q_{k} - \qkc}{\ell_k} \leq \sumk \left(\Dp(\qkc, \qh_k) - \Dp(\qkc, \qh_{k+1})\right) + 32\eta \sumk \inner{q_k}{\ell_k^2} - 16\eta \sumk \inner{\qkc}{\ell_k^2}.
    \end{align}
    For the first term, from the definition of $\Dp(q, q^\prime)$, we have
    \begin{align}
             & \sumk\left(D_{\psi}(\qkc, \qh_k)- D_{\psi}(\qkc, \qh_{k+1})\right)  \label{eq:ssp-lemma-2}                                                                                                                                           \\
        =    & \ \Dp(q^{\pi_1^\C}, \qh_1) + \sum_{k=2}^K \left(\Dp(\qkc, \qh_k) - \Dp(q^{\pi_{k-1}^\C}, \qh_k) \right)\nonumber                                                                                                                     \\
        =    & \ \Dp(q^{\pi_1^\C}, \qh_1) + \frac{1}{\eta}\sum_{k=2}^K\sum_{x,a}\left(\qkc(x,a)\log{\frac{\qkc(x,a)}{\qh_k(x,a)}} - q^{\pi_{k-1}^\C}(x,a)\log{\frac{q^{\pi_{k-1}^\C}(x,a)}{\qh_k(x,a)}}\right) \nonumber                            \\
             & \hspace{8cm} + \frac{1}{\eta}\sum_{k=2}^K\sum_{x,a}\left(q^{\pi_{k-1}^\C}(x,a)-\qkc(x,a)\right)\nonumber                                                                                                                             \\
        =    & \ \Dp(q^{\pi_1^\C}, \qh_1) + \frac{1}{\eta}\sum_{k=2}^K\sum_{x,a}\left(\qkc(x,a)-q^{\pi_{k-1}^\C}(x,a)\right)\log{\frac{1}{\qh_k(x,a)}} \nonumber                                                                                    \\
             & \hspace{6cm} + \frac{\psi(q^{\pi_K^\C})-\psi(q^{\pi_1^\C})-\sum_{x,a}(q^{\pi_K^\C}(x,a)-q^{\pi_1^\C}(x,a))}{\eta} \nonumber                                                                                                          \\
        \leq & \ \frac{1}{\eta} \log{\frac{H}{\alpha}} \sum_{k=2}^K \norm{\qkc - q^{\pi_{k-1}^\C}}_{1} + \Dp(q^{\pi_1^\C}, \qh_1) + \frac{\psi(q^{\pi_K^\C})-\psi(q^{\pi_1^\C})-\sum_{x,a}(q^{\pi_K^\C}(x,a)-q^{\pi_1^\C}(x,a))}{\eta}   \nonumber,
    \end{align}
    where the last inequality holds due to $\abs{\log \qh_k(x,a)} \leq \log\frac{H}{\alpha}$ since $\alpha \leq \qh_k(x,a) \leq H$ for $\qh_k \in \Delta(M, H, \alpha)$. For the last two term, since $\qh_1$ minimize $\psi$ over $\Delta(M, H, \alpha)$, we have $\inner{\nabla \psi(\qh_1)}{q^{\pi_1^\C} -\qh_1} \leq 0$, thus
    \begin{equation}
        \label{eq:ssp-lemma-3}
        \begin{aligned}
                 & \Dp(q^{\pi_1^\C}, \qh_1) + \frac{\psi(q^{\pi_K^\C})-\psi(q^{\pi_1^\C})-\sum_{x,a}(q^{\pi_K^\C}(x,a)-q^{\pi_1^\C}(x,a))}{\eta}                                   \\
            \leq & \ \frac{\psi(q^{\pi_1^\C}) - \psi(\qh_1)}{\eta} + \frac{\psi(q^{\pi_K^\C})-\psi(q^{\pi_1^\C})-\sum_{x,a}\left(q^{\pi_K^\C}(x,a)-q^{\pi_1^\C}(x,a)\right)}{\eta} \\
            \leq & \ \frac{\psi(q^{\pi_K^\C})-\psi(\qh_1)-\sum_{x,a}\left(q^{\pi_K^\C}(x,a)-q^{\pi_1^\C}(x,a)\right)}{\eta}                                                        \\
            \leq & \ \frac{H\log(|X||A|)+H \log{H} + H}{\eta} = \frac{H(1+\log(|X||A|H))}{\eta},
        \end{aligned}
    \end{equation}
    where the last inequality holds due to  $-H\log(|X||A|) \leq \psi(q) \leq H \log{H}$ and $0\leq \sum_{x,a} q(x,a) \\\leq H$ for any $q\in \Delta(M, H, \alpha)$ from Lemma~\ref{lem:bound-of-psi}.
    Combining~\eqref{eq:ssp-lemma-1},~\eqref{eq:ssp-lemma-2} and~\eqref{eq:ssp-lemma-3}, we have
    \begin{align*}
        \sumk\inner{q_{k} - \qkc}{\ell_k}
         & \leq \frac{H(1+\log(|X||A|H)) + \Pb_K \log(H/\alpha)}{\eta}  + 32\eta \sumk \inner{q_k}{\ell_k^2} - 16\eta \sumk \inner{\qkc}{\ell_k^2},
    \end{align*}
    where $\Pb_K = \sum_{k=2}^K \norm{\qkc - q^{\pi_{k-1}^\C}}_{1}$. This finishes the proof.
\end{proof}

\subsection{Proof of Theorem~\ref{thm:ssp}}
\label{sec:proof-dynamic-ssp}
\begin{proof}
    We only need to consider the case $H_* \leq K$ (otherwise the claimed regret bound is vacuous). Since all the compared policies are proper, they will not visit the states from which the goal state $g$ is not accessible (otherwise the hitting time will be infinite) and the states which are not accessible from initial state $x_0$. We can remove them from the SSP since we consider the known transition setting. Then, suppose $K$ is large enough such that these exists at least a policy $\pi^u$ whose occupancy measure $q^{\pi^u}$ satisfies $q^{\pi^u} \in \Delta(M, K, \frac{1}{K})$. Then, we define $u_k = (1-\frac{1}{K^2})\qkc + \frac{1}{K^2}q^{\pi^u}$ and the corresponding policy $\pi^{u_k}$. For any $k\in[K]$, we ensure that the hitting time $H^{\pi^{u_k}}\leq (1-\frac{1}{K^2}) H_* + \frac{K}{K^2} \leq H_* + 1$ and the occupancy measure $u_k(x,a)\geq \frac{1}{K^3}, \forall x,a$, \ie $u_k \in \Delta(M, H_*+1, \frac{1}{K^3})$. Thus, we have
    \begin{align}
        \E[{\DReg}_K(\pi_{1:K}^\C)] & =  \E \Bigg[\sumk \sum_{i,j} p_k^{i,j} \inner{q_k^{i,j}}{\ell_k} - \sumk \inner{\qkc}{\ell_k}\Bigg] \nonumber                                                                      \\
                                    & = \E \Bigg[\sumk \sum_{i,j} p_k^{i,j} \inner{q_k^{i,j}}{\ell_k} -  \sumk \inner{u_k}{\ell_k}\Bigg] + \frac{1}{K^2} \E \Bigg[\sumk \inner{q^{\pi^u} - \qkc}{\ell_k}\Bigg]\nonumber  \\
                                    & \leq  \E \Bigg[\sumk \sum_{i,j} p_k^{i,j} \inner{q_k^{i,j}}{\ell_k} -  \sumk \inner{u_k}{\ell_k}\Bigg] + 2,\nonumber                                                               \\
                                    & \leq \underbrace{\E \Bigg[\sumk\inner{p_k - e_{i,j}}{h_k}\Bigg]}_{\meta} +\underbrace{\E \Bigg[\sumk\inner{q_{k}^{i,j}-u_k}{\ell_k}\Bigg]}_{\base} + 2 \label{eq:ssp-clip-dynamic}
    \end{align}
    where the first inequality holds due to $\sum_{x,a} q^u(x,a) \leq K$ and $\sum_{x,a} \qkc(x,a) \leq H_* \leq K$, the last inequality holds due to the definition that $h_k^{i,j} = \inner{q_k^{i,j}}{\ell_k}, \forall i \in [G],j\in [N_i]$ and the decomposition holds for any index $i \in [G],j\in [N_i]$.

    \paragraph{Upper bound of base-regret.} Since the possible range of $H_*$ is $H^{\pi^f} \leq H_* \leq K$. From the construction of horizon length pool $\mathcal{H} = \{H_i = 2^{i-1} \cdot H^{\pi^f} \mid i \in [G]\}$ where $G = 1 + \ceil{\log((K+1)/H^{\pi^f})}$, we ensure
    \begin{equation*}
        H_1 = H^{\pi^f} \leq H_*+1 \mbox{ and } H_G = K+1 \geq H_*+1.
    \end{equation*}
    So for any unknown $H_*$, there exist an index $i$ for the space pool that satisfies $H_{i^*-1}=\frac{H_{i^*}}{2} \leq H_*+1 \leq H_{i^*}$. Then, we analysis the base-regret of the base learners in group $i^*$.

    From the construction of each step size pool, we ensure $\eta_{i,j} \leq \frac{1}{64}$, \ie $32 \eta_{i,j} \abs{\ell_{k,r}}\leq 1, \forall i\in[G],j\in[N_i], k\in[K], r\in[|X||A|]$. Since $q_k^{i^*, j}, u_k \in \Delta(M, H_{i^*}, 1/K^3), \forall j \in [N_i^*], k\in [K]$, from Lemma~\ref{lem:CDO-base}, we have
    \begin{align}
             & \ \base \nonumber                                                                                                                                                                              \\
        \leq & \ \frac{4\sum_{k=2}^K \norm{u_k-u_{k-1}}_1 \log{K} +H_{i^*}(1+\log(|X||A|H_{i^*}))}{\eta_{i^*,j}} + 16 \eta_{i^*,j} \sumk (2\inner{u_k}{\ell_{k}} - \inner{q_k^{i^*,j}}{\ell_{k}^2}) \nonumber \\
        \leq & \ \frac{4\Pb_K \log{K} +H_{i^*}(1+\log(|X||A|H_{i^*}))}{\eta_{i^*,j}} + 32\eta_{i^*,j} B_K - 16\eta_{i^*,j} \sumk \inner{q_k^{i^*,j}}{\ell_{k}^2} + 1 \label{eq:ssp-base-dynamic},
    \end{align}
    where the last inequality holds due to $\sum_{k=2}^K \norm{u_k - u_{k-1}}_1 \leq \sum_{k=2}^K \norm{\qkc - q^{\pi_{k-1}^\C}}_1=\Pb_K$ and $B_K = \sumk \inner{\qkc}{\ell_k}$.

    \paragraph{Upper bound of meta-regret.} Then, we consider the meta-regret with respect to base-learner $\B_{i^*,j}, \forall j\in N_{i^*}$. From the construction of the regularizer $\psib(p)$ in meta-algorithm, we have $32 \varepsilon_{i,j} \abs{h_{k}^{i,j} } = 32 \frac{\eta_{i,j}}{2H_i} \abs{\inner{q_k^{i,j}}{\ell_k}} \leq 1, \forall i\in[G], j\in[N_i], k\in[K]$. From the analysis of OMD in Lemma~\ref{lem:impossible-tuning}, we have
    \begin{align*}
        \meta & \leq \Dpb(e_{i^*,j}, p_1) + 32\varepsilon_{i^*,j} \sumk (h_k^{i^*,j})^2                                                                                                                                                                                                                      \\
              & = \frac{1}{\varepsilon_{i^*,j}} \log \frac{1}{p_1^{i^*,j}} + \sum_{r=1}^G \sum_{s=1}^{N_i}   \frac{p_1^{r,s}}{\varepsilon_{r,s}}+32\varepsilon_{i^*,j} \sumk (h_k^{i^*,j})^2                                                                                                                 \\
              & = \frac{1}{\varepsilon_{i^*,j}} \log{\frac{\sum_{r=1}^G \sum_{s=1}^{N_i}\varepsilon_{r,s}^2}{\varepsilon_{i^*,j}^2}} + \frac{\sum_{r=1}^G \sum_{s=1}^{N_i}\varepsilon_{r,s}}{\sum_{r=1}^G \sum_{s=1}^{N_i} \varepsilon_{r,s}^2} + 32 \varepsilon_{i^*,j}\sumk \inner{q_k^{i^*,j}}{\ell_k}^2,
    \end{align*}
    where the first equality holds due to $\Dpb(p,p')=\sum_{i,j}\frac{1}{\varepsilon_{i,j}}(p_{i,j}\log\frac{p_{i,j}}{p'_{i,j}} - p_{i,j} + p'_{i,j})$ and the last equality is due to $p_1^{i,j} \propto \varepsilon_{i,j}^2, h_k^{i,j} = \inner{q_k^{i,j}}{\ell_k}, \forall i \in[G],j\in[N_i]$. From the definition of the horizon length pool $\mathcal{H} = \{H_i = 2^{i-1} \cdot H^{\pi^f} \mid i \in [G]\}$ where $G = 1 + \ceil{\log((K+1)/H^{\pi^f})}$, the step size pools $\Ecal_i = \left\{ \frac{1}{32 \cdot 2^j} \mid j \in [N_i] \right\}, i\in[G],$ where $N_i = \ceil{\frac{1}{2}\log{(\frac{4 K}{1+\log{(|X||A|H_i)}})}}$ and learning rate $\varepsilon_{i,j} = \frac{\eta_{i,j}}{2H_i}, \forall i\in[G],j\in[N_i]$, we ensure that $\sum_{r=1}^G \sum_{s=1}^{N_i} \varepsilon_{r,s} = \Theta(1/H_1)$ and $\sum_{r=1}^G \sum_{s=1}^{N_i} \varepsilon_{r,s}^2 = \Theta(1/H_1^2)$. Thus,
    \begin{equation}
        \begin{aligned}
            \label{eq:ssp-meta-dynamic}
            \meta & \leq \Theta\left(\frac{H_{i^*}}{\eta_{i^*,j}} \log{\frac{H_{i^*}}{H_1\eta_{i^*,j}}}\right) + 16 \frac{\eta_{i^*,j}}{H_{i^*}}\sumk \inner{q_k^{i^*,j}}{\ell_k}^2 + \Theta(H_1) \\
                  & \leq \Theta\left(\frac{H_{i^*}}{\eta_{i^*,j}} \log{\frac{H_{i^*}}{H_1\eta_{i^*,j}}}\right) + 16 \eta_{i^*,j}\sumk \inner{q_k^{i^*,j}}{\ell_k^2} + \Theta(H_1),
        \end{aligned}
    \end{equation}
    where the last inequality holds due to Cauchy–Schwarz inequality.

    \paragraph{Upper bound of over all dynamic regret.} Combining~\eqref{eq:ssp-clip-dynamic},~\eqref{eq:ssp-base-dynamic} and~\eqref{eq:ssp-meta-dynamic}, we obtain
    \begin{equation}
        \label{eq:ssp-overall-dynamic}
        \E[{\DReg}_K] \leq \frac{4\Pb_T \log{K} +H_{i^*}(1+\log(|X||A|H_{i^*}))}{\eta_{i^*,j}}+ 32\eta_{i^*,j} B_K +  \Theta\left( \frac{H_{i^*}}{\eta_{i^*,j}} \log{\frac{H_{i^*}}{H_1\eta_{i^*,j}}} \right)
    \end{equation}
    holds for any index $j \in [N_{i^*}]$. Omit the last term, it is clear that the optimal step size is $\eta^* = \sqrt{\frac{H_{i^*}(1+\log(|X||A|H_{i^*}))+4\Pb_K \log K}{32B_K}}$. Meanwhile, since $\sum_{x,a} q_k(x,a) \leq H_*, \forall k \in [K]$, we have $0\leq \Pb_K=\sum_{k=2}^K \norm{\qkc - q^{\pi_{k-1}^\C}}_{1}\leq 2H_*K \leq 2H_{i^*}K$ and $B_K \leq H_*K\leq H_{i^*}K$. Therefore, we ensure that
    \begin{equation*}
        \eta^* \geq \sqrt{\frac{1+\log(|X||A|H_{i^*})}{32K}}.
    \end{equation*}
    From the construction of the candidate step size pool $H_{i^*}$, we know that the step size therein is monotonically decreasing with respect to the index, in particular,
    \begin{equation*}
        \eta_1 = \frac{1}{64}, \mbox{ and }\eta_N = \sqrt{\frac{1+\log(|X||A|H_{i^*})}{128K}} \leq \eta^*
    \end{equation*}
    Let $j^*$ be the index of base learner in group $i^*$ with step size closest to the $\eta^*$. Then, we consider the base regret of the base learner $B_{i^*,j^*}$.  We consider the following two cases:
    \begin{itemize}
        \item when $\eta^* \leq \frac{1}{64}$, then $\eta_{i^*, j^*} \leq \eta^* \leq 2 \eta_{i^*, j^*}=\eta_{i^*, j^*-1}$, we have
              \begin{align*}
                  \mbox{R.H.S of \eqref{eq:ssp-overall-dynamic}}
                  \leq & \ \frac{8\Pb_K\log{K} +2H_{i^*}(1+\log(|X||A|H_{i^*}))}{\eta^*}  + 32\eta^* B_K + \Theta\left(\frac{H_{i^*}}{\eta^*} \log{\frac{H_{i^*}}{H_1\eta^*}}\right) \\
                  \leq & \ \Ot\left(\sqrt{\left(\Pb_K +H_{*} \right)B_K}\right).
              \end{align*}
        \item when $\eta^* > \frac{1}{64}$, then $\eta_{i^*, j^*} = \frac{1}{64}$, we have
              \begin{align*}
                  \mbox{R.H.S of \eqref{eq:ssp-overall-dynamic}}
                   & \leq  256\left(\Pb_K\log{K} +H_{i^*}(1+\log(|X||A|H_{i^*}))\right)  + \frac{1}{2} B_K + \Theta(H_*) \\
                   & \leq \Ot\left(\Pb_K +H_{*} \right),
              \end{align*}
              where the last inequality holds due to $\sqrt{\frac{H_{i^*}(1+\log(|X||A|H_{i^*}))+4\Pb_K \log K}{32B_K}} \geq \frac{1}{64}$.
    \end{itemize}
    As a result, taking both cases into account yields
    \begin{equation*}
        \sumk\inner{q_{k} - \qkc}{\ell_k} \leq \Ot\left(\sqrt{\left(H_{*}+\Pb_K  \right)(H_{*}+\Pb_K +B_K)}\right).
    \end{equation*}
    This finishes the proof.
\end{proof}

\subsection{Proof of Theorem~\ref{thm:lower-bound-ssp}}
\begin{proof}
    The proof is similar to that of Theorem~\ref{Thm:lower-bound-loop-free}.
    For any $\gamma \in [0,2T]$, we first construct a piecewise-stationary comparator sequence, whose path length is smaller than $\gamma$, then we split the whole time horizon into several pieces, where the comparator is fixed in each piece. By this construction, we can apply the existed minimax static regret lower bound of episodic SSP~\citep{COLT'21:SSP-minimax} in each piece, and finally sum over all pieces to obtain the lower bound for the dynamic regret.

    Denote by $R_K(\Pi, \F, \gamma)$ the minimax dynamic regret, which is defined as
    \begin{equation*}
        R_K(\Pi, \F, \gamma) = \inf_{\pi_1 \in \Pi}\sup_{\ell_1 \in \F} \ldots \inf_{\pi_K \in \Pi}\sup_{\ell_K \in \F} \left(\sumk \inner{q^{\pi_k}}{\ell_k} - \min_{(\pi_1^\C, \ldots, \pi_K^\C) \in \U(\gamma)} \sumk \inner{\qkc}{\ell_k}\right)
    \end{equation*}
    where $\Pi$ denotes the set of all policies, $\F$ denotes the set of loss functions $\ell \in R_{[0,1}^{|X||A|}$ and $\U(\gamma) = \{(\pi_1^\C, \ldots, \pi_K^\C) \mid \forall k \in [K], \pi_k^\C \in \Pi, \mbox{ and } \Pb_K = \sum_{k=2}^K\norm{\qkc-q^{\pi_{k-1}^\C}}_{1} \leq \gamma\}$ is the set of feasible policy sequences with path length $\Pb_K$ of the occupancy measures less than $\gamma$.

    We first consider the case of $\gamma \leq 2(H_*+1)$. From Theorem 3 of~\citet{COLT'21:SSP-minimax}, we ensure for any $D, H_*, K$ with $K\geq D+1$, there exists an SSP instance such that its diameter is $D+2$, the hitting time of the best fixed policy is $H_*+1$ and the expected regret of any policy after $K$ episodes is at least $\Omega(\sqrt{DH_* K})$. Then we can set all compared policies as the best fixed policy and directly utilize this lower bound of the static regret as a natural lower bound of dynamic regret,
    \begin{equation}
        \label{eq:lower-bound-ssp-1}
        R_K(\Pi, \F, \gamma) \geq \Omega(\sqrt{DH_*K}).
    \end{equation}

    We next deal with the case that $\gamma \geq 2(H_*+1)$. Without loss of generality, we assume $L= \ceil{\gamma/2(H_*+1)}$ devides $K$ and split the whole time horizon into $L$ pieces equally. Next, we construct an SSP instance such that its diameter is $D+2$, the hitting time of the best fixed policy is $H_*+1$ and the expected regret of any policy after $K$ episodes is at least $\Omega(\sqrt{DH_* K})$ in each piece. Then, we choose the best fixed policy in each piece as the comparator sequence, whose hitting time are all $H_*+1$. Since the sequence changes at most $L-1\leq\gamma/2(H_*+1)$ times and the variation of the policy sequence at each change point is at most $2(H_*+1)$ (Note that $\norm{q^{\pi_k^\C} - q^{\pi_{k-1}^\C}} \leq \norm{q^{\pi_k^\C}}_1 + \norm{q^{\pi_{k-1}^\C}}_1 = 2(H_*+1), \forall \pi_k^\C \neq \pi_{k-1}^\C)$, the path- $\Pb_K$ does not exceed $\gamma$. As a result,
    \begin{equation}
        \label{eq:lower-bound-ssp-2}
        R_K(\Pi, \F, \gamma) \geq L \Omega(\sqrt{DH_*K/L}) \geq \Omega(\sqrt{DK\gamma}).
    \end{equation}
    Combining~\eqref{eq:lower-bound-ssp-1} and~\eqref{eq:lower-bound-ssp-2}, we have
    \begin{equation*}
        R_K(\Pi, \F, \gamma) \geq \Omega(\sqrt{DH_*K}) + \Omega(\sqrt{DK\gamma}) \geq \Omega(\sqrt{DH_*K(1+\gamma/H_*)}),
    \end{equation*}
    which finishes the proof.
\end{proof}

\subsection{Proof of Lemma~\ref{lem:ssp}}
\begin{proof}
    We assume $m_k \in \R_{[0,1]}^{|X||A|}, \forall k \in [K]$ since $\ell_k \in \R_{[0,1]}^{|X||A|}, \forall k \in [K]$. Since $\eta \leq \frac{1}{64}$, we ensure that $32 \eta \abs{\ell_{k,i}-\mp_{k,i}}\leq 1, \forall i \in [|X||A|]$. Thus, from Lemma~\ref{lem:impossible-tuning}, we have
    \begin{align*}
                & \sumk\inner{q_{k} - \qkc}{\ell_k}                                                                                                                           \\
        \leq {} & \sumk \big(\Dp(\qkc, \qh_k) - \Dp(\qkc, \qh_{k+1})\big) + 32\eta \sumk \inner{\qkc}{(\ell_{k}-\mp_{k})^2} - 16\eta \sumk \inner{q_k}{(\ell_{k}-\mp_{k})^2}.
    \end{align*}
    From the same analyses as that in Section~\ref{sec:proof-base-ssp}, we have
    \begin{equation*}
        \begin{aligned}
            \sumk \big(\Dp(\qkc, \qh_k) - \Dp(\qkc, \qh_{k+1})\big) \leq \frac{H(1+\log(|X||A|H)) + \Pb_K \log(H/\alpha)}{\eta},
        \end{aligned}
    \end{equation*}
    which finishes the proof.
\end{proof}

\subsection{Proof of Theorem~\ref{thm:ssp-adaptivity}}

\begin{proof}
    Similar to the argument in Section~\ref{sec:proof-dynamic-ssp}, suppose $K$ is large enough such that these exists at least a policy $\pi^u$ whose occupancy measure $q^u$ satisfies $q^{\pi^u} \in \Delta(M, K, \frac{1}{K})$. Then, we define $u_k = (1-\frac{1}{K^2})\qkc + \frac{1}{K^2}q^{\pi^u}$ and the corresponding policy $\pi^{u_k}$. For any $k\in[K]$, we ensure that the hitting time $H^{\pi^{u_k}}\leq (1-\frac{1}{K^2}) H_* + \frac{K}{K^2} \leq H_* + 1$ and the occupancy measure $u_k(x,a)\geq \frac{1}{K^3}, \forall x,a$, \ie $u_k \in \Delta(M, H_*+1, \frac{1}{K^3})$. Thus, we have
    \begin{align}
        \E[{\DReg}_K(\pi_{1:K}^\C)] & =  \E \Bigg[\sumk \sum_{i,j} p_k^{i,j} \inner{q_k^{i,j}}{\ell_k} - \sumk \inner{\qkc}{\ell_k}\Bigg] \nonumber                                                                        \\
                                    & = \E \Bigg[\sumk \sum_{i,j} p_k^{i,j} \inner{q_k^{i,j}}{\ell_k} -  \sumk \inner{u_k}{\ell_k}\Bigg] + \frac{1}{K^2} \E \Bigg[\sumk \inner{q^{\pi^u} - \qkc}{\ell_k}\Bigg]\nonumber    \\
                                    & \leq  \E \Bigg[\sumk \sum_{i,j} p_k^{i,j} \inner{q_k^{i,j}}{\ell_k} -  \sumk \inner{u_k}{\ell_k}\Bigg] + 2,\nonumber                                                                 \\
                                    & \leq \underbrace{\E \Bigg[\sumk\inner{p_k - e_{i,j}}{h_k}\Bigg]}_{\meta} +\underbrace{\E \Bigg[\sumk\inner{q_{k}^{i,j}-u_k}{\ell_k}\Bigg]}_{\base} + 2, \label{eq:ssp-decomposition}
    \end{align}
    where the first inequality holds due to $\sum_{x,a} q^u(x,a) \leq K$ and $\sum_{x,a} \qkc(x,a) \leq H_* \leq K$, the last inequality holds due to the definition that $h_k^{i,j} = \inner{q_k^{i,j}}{\ell_k}, \forall i \in [G],j\in [N_i]$ and the decomposition holds for any index $i \in [G],j\in [N_i]$.

    \paragraph{Upper bound of base-regret.} Since the possible range of $H_*$ is $H^{\pi^f} \leq H_* \leq K$. From the construction of horizon length pool $\mathcal{H} = \{H_i = 2^{i-1} \cdot H^{\pi^f} | i \in [G]\}$ where $G = 1 + \ceil{\log((K+1)/H^{\pi^f})}$, we ensure
    \begin{equation*}
        H_1 = H^{\pi^f} \leq H_*+1 \mbox{ and } H_G = K+1 \geq H_*+1.
    \end{equation*}
    So for any unknown $H_*$, there exist an index $i$ for the space pool that satisfies $H_{i^*-1}=\frac{H_{i^*}}{2} \leq H_*+1 \leq H_{i^*}$. Then, we analysis the base-regret of the base learners in group $i^*$. Then, we consider the base-learners in group $i^*$. From the construction of each step size pool, we ensure $\eta_{i,j} \leq \frac{1}{64}$, \ie $32 \eta_{i,j} \abs{\ell_{k,r}- m'_{k,r}}\leq 1, \forall i\in[G],j\in[2N_i], r\in [|X||A|], k \in [K]$. Since $q_k^{i^*, j} \in \Delta(M, H_{i^*}, 1/K^3), \forall k, j$ and $u_k \in \Delta(M, H_{i^*}, 1/K^3), \forall k$, from Lemma~\ref{lem:ssp}, we have
    \begin{align}
             & \base \label{eq:ssp-base}                                                                                                                                                  \\
        \leq & \ \frac{4\sum_{k=2}^K \norm{u_k-u_{k-1}}_1 \log{K} +H_{i^*}(1+\log(|X||A|H_{i^*}))}{\eta_{i^*,j}} \nonumber                                                                \\
             & \hspace{4cm} + 32\eta_{i^*,j} \sumk \inner{u_k}{(\ell_{k}-m'_{k})^2} - 16\eta_{i^*,j} \sumk \inner{q_k^{i^*,j}}{(\ell_{k}-m'_{k})^2} \nonumber                             \\
        \leq & \ \frac{4\Pb_K \log{K} +H_{i^*}(1+\log(|X||A|H_{i^*}))}{\eta_{i^*,j}} + 32\eta_{i^*,j} V'_K - 16\eta_{i^*,j} \sumk \inner{q_k^{i^*,j}}{(\ell_{k}-m'_{k})^2} + 4 \nonumber,
    \end{align}
    where $V'_K = \sumk \inner{\qkc}{(\ell_{k}-m'_{k})^2}$ and $\Pb_K = \sum_{k=2}^K \norm{\qkc - q^{\pi_{k-1}^\C}}_{1}$.

    \paragraph{Upper bound of meta-regret.} Then, we consider the meta-regret with respect to base-learner $\B_{i^*,j}, \forall j\in 2N_{i^*}$. From the construction of the regularizer $\psib(p)$, we have $32 \varepsilon_{i,j} \abs{h_{k}^{i,j} -M_{k}^{i,j}} = 32 \frac{\eta_{i,j}}{2H_i} \abs{\inner{q_k^{i,j}}{\ell_k-m'_k}} \leq 1$ for all $i \in [G], j\in[2 N_i], k\in[K]$. From Lemma~\ref{lem:impossible-tuning} (dropping the negative term), we have
    \begin{align*}
        \meta & \leq \Dpb(e_{i^*,j}, \ph_1) + 32\varepsilon_{i^*,j}\sumk \inner{q_k^{i^*,j}}{\ell_k-m'_k}^2                                                                                                                                                                                                          \\
              & = \frac{1}{\varepsilon_{i^*,j}} \log \frac{1}{\ph_1^{\ i^*,j}} + \sum_{r=1}^G \sum_{s=1}^{2N_i}   \frac{\ph_1^{\ r,s}}{\varepsilon_{r,s}}+32\varepsilon_{i^*,j}\sumk \inner{q_k^{i^*,j}}{\ell_k-m'_k}^2                                                                                              \\
              & = \frac{1}{\varepsilon_{i^*,j}} \log{\frac{\sum_{r=1}^G \sum_{s=1}^{2N_i}\varepsilon_{r,s}^2}{\varepsilon_{i^*,j}^2}} + \frac{\sum_{r=1}^G \sum_{s=1}^{2N_i}\varepsilon_{r,s}}{\sum_{r=1}^G \sum_{s=1}^{2N_i} \varepsilon_{r,s}^2} + 32 \varepsilon_{i^*,j}\sumk \inner{q_k^{i^*,j}}{\ell_k-m'_k}^2,
    \end{align*}
    where the first equality holds due to $\Dpb(p,p')=\sum_{i,j}\frac{1}{\varepsilon_{i,j}}(p_{i,j}\log\frac{p_{i,j}}{p'_{i,j}} - p_{i,j} + p'_{i,j})$ and the last inequality is due to the definition that $\ph_1^{\ i,j} \propto \varepsilon_{i,j}^2$. From the definition of horizon length pool $\mathcal{H} = \{H_i = 2^{i-1} \cdot H^{\pi^f} \mid i \in [G]\}$ where $G = 1 + \ceil{\log((K+1)/H^{\pi^f})}$, the step size pools $\Ecal_i = \left\{ \frac{1}{32 \cdot 2^j} \mid j \in [N_i] \right\}, i\in[G],$ where $N_i = \ceil{\frac{1}{2}\log{(\frac{4 K}{1+\log{(|X||A|H_i)}})}}$ and learning rate $\varepsilon_{i,j} = \frac{\eta_{i,j}}{2H_i}, \forall i\in[G],j\in[2N_i]$, we ensure $\sum_{r=1}^G \sum_{s=1}^{N_i} \varepsilon_{r,s} = \Theta(1/H_1)$ and $\sum_{r=1}^G \sum_{s=1}^{N_i} \varepsilon_{r,s}^2 = \Theta(1/H_1^2)$. Thus,
    \begin{equation}
        \begin{aligned}
            \label{eq:ssp-meta}
            \meta & \leq \Theta\left(\frac{H_{i^*}}{\eta_{i^*,j}} \log{\frac{H_{i^*}}{H_1\eta_{i^*,j}}}\right) + 16 \frac{\eta_{i^*,j}}{H_{i^*}}\sumk \inner{q_k^{i^*,j}}{\ell_k - m'_k}^2 + \Theta(H_1) \\
                  & \leq \Theta\left(\frac{H_{i^*}}{\eta_{i^*,j}} \log{\frac{H_{i^*}}{H_1\eta_{i^*,j}}}\right) + 16 \eta_{i^*,j}\sumk \inner{q_k^{i^*,j}}{(\ell_k - m'_k)^2} + \Theta(H_1),
        \end{aligned}
    \end{equation}
    where the last inequality holds due to Cauchy-Schwarz inequality.

    \paragraph{Upper bound of overall dynamic regret.} Combining~\eqref{eq:ssp-decomposition},~\eqref{eq:ssp-base} and~\eqref{eq:ssp-meta}, we obtain
    \begin{equation}
        \label{eq:ssp-overall}
        \E[{\DReg}_K]\leq \frac{4\Pb_T \log{K} +H_{i^*}(1+\log(|X||A|H_{i^*}))}{\eta_{i^*,j}}+ 32V'_K +  \Theta\left(\frac{H_{i^*}}{\eta_{i^*,j}} \log{\frac{H_{i^*}}{H_1\eta_{i^*,j}}}\right),
    \end{equation}
    which holds for any index $j$. Omit the last term, it is clear that the optimal step size is $\eta^* = \sqrt{\frac{H_{i^*}(1+\log(|X||A|H_{i^*}))+4\Pb_K \log K}{32V'_K}}$. Meanwhile, since $\sum_{x,a} q_k(x,a) \leq H_*, \forall k$, we have $0\leq \Pb_K=\sum_{k=2}^K \norm{\qkc - q^{\pi_{k-1}^\C}}_{1}\leq 2H_*K \leq 2H_{i^*}K$ and $V'_K = \sumk \inner{\qkc}{(\ell_{k}-m'_{k})^2} \leq 4H_*K\leq 4H_{i^*}K$. Therefore, we ensure that
    \begin{equation*}
        \eta^* \geq \sqrt{\frac{1+\log(|X||A|H_{i^*})}{128K}}.
    \end{equation*}
    From the construction of the candidate step size pool $H_{i^*}$, we know that the step size therein is monotonically decreasing with respect to the index, in particular,
    \begin{equation*}
        \eta_1 = \frac{1}{64}, \mbox{ and }\eta_N = \sqrt{\frac{1+\log(|X||A|H_{i^*})}{128K}} \leq \eta^*.
    \end{equation*}

    First we consider the case that $\sumk \inner{\qkc}{\ell_k^2} > \sumk \inner{\qkc}{(\ell_k-m_k)^2}$. Let $j^*$ be the index of base learner in group $i^*$ with step size closest to $\eta^*$ and optimism $m'_k = m_k$. Then, we consider the base regret of the base learner $B_{i^*,j^*}$.  We
    consider the following two cases:
    \begin{itemize}
        \item when $\eta^* \leq \frac{1}{64}$, then $\eta_{i^*, j^*} \leq \eta^* \leq 2 \eta_{i^*, j^*}=\eta_{i^*, j^*-1}$, we have
              \begin{align*}
                          & \mbox{R.H.S of \eqref{eq:ssp-overall}}                                                                                                                                                       \\
                  \leq {} & \frac{8\Pb_K\log{K} +2H_{i^*}(1+\log(|X||A|H_{i^*}))}{\eta^*}  + 32\eta^* \sumk \inner{\qkc}{(\ell_{k}-m_{k})^2} + \Theta\left(\frac{H_{i^*}}{\eta^*} \log{\frac{H_{i^*}}{H_1\eta^*}}\right) \\
                  \leq {} & \Ot\left(\sqrt{\left(\Pb_K +H_{*} \right)\sumk \inner{\qkc}{(\ell_{k}-m_{k})^2}}\right) = \Ot\left(\sqrt{\left(\Pb_K +H_{*} \right)V_K}\right),
              \end{align*}
              where the last equality holds due to $V_K = \min\{\sum_{k=1}^K {\inner{\qkc}{\ell_k^2}}, \sum_{k=1}^K {\inner{\qkc}{(\ell_k-m_k)^2}}\}$ and $\sumk \inner{\qkc}{\ell_k^2} > \sumk \inner{\qkc}{(\ell_k-m_k)^2}$.
        \item when $\eta^* > \frac{1}{64}$, then $\eta_{i^*, j^*} = \frac{1}{64}$, we have
              \begin{align*}
                  \mbox{R.H.S of \eqref{eq:ssp-overall}}
                   & \leq  256\left(\Pb_K\log{K} +H_{i^*}(1+\log(|X||A|H_{i^*}))\right)  + \frac{1}{2}\sumk \inner{\qkc}{(\ell_{k}-m_{k})^2} + \Theta(H_*) \\
                   & \leq \Ot\left(\Pb_K +H_{*} \right),
              \end{align*}
              where the last inequality holds due to $\sqrt{\frac{H_{i^*}(1+\log(|X||A|H_{i^*}))+4\Pb_K \log K}{32\sumk \inner{\qkc}{(\ell_{k}-m_{k})^2}}} \geq \frac{1}{64}$.
    \end{itemize}
    Then, when $\sumk \inner{\qkc}{\ell_k^2} > \sumk \inner{\qkc}{(\ell_k-m_k)^2}$, we can choose the base-learner $B_{i^*,j^*}$ with step size closest to the $\eta^*$ and optimism $m'_k = 0$ to analysis and obtain the same result.
    As a result, taking both cases into account yields
    \begin{equation*}
        \E[{\DReg}_K(\pi_{1:K}^\C)] \leq \Ot\left(\sqrt{\left(H_{*}+\Pb_K  \right)(H_{*}+\Pb_K +V_K)}\right).
    \end{equation*}
    This finishes the proof.
\end{proof}

\subsection{Useful Lemmas}
\label{sec:appendix-lemma-SSP}
we introduce the following lemma which shows the boundedness of the regularizer.
\begin{myLemma}
    \label{lem:bound-of-psi}
    Let $H \geq 1$, it holds that $-H\log(|X||A|) \leq \sum_{x,a}q(x,a)\log{q(x,a)} \leq H \log{H}$ for every $q \in \Delta(M, H)$.
\end{myLemma}
\begin{proof}
    First, we prove the right-hand side of the inequality.
    \begin{align*}
        \sum_{x,a}q(x,a)\log{q(x,a)} = \sum_{x,a}q(x,a)\log{\frac{q(x,a)}{H}} + \sum_{x,a} q(x,a) \log{H} \leq \sum_{s,a} q(x,a) \log{H} \leq H\log{H}.
    \end{align*}
    Then, we prove the left-hand side of the inequality.
    \begin{align*}
        -\sum_{x,a}q(x,a)\log{q(x,a)} & = -\sum_{x,a}q(x,a)\log{\frac{q(x,a)}{H}} - \sum_{x,a} q(x,a) \log{H}          \\
                                      & \leq -H \sum_{s,a}\frac{q(x,a)}{H}  \log{\frac{q(x,a)}{H}} \leq H\log{|X||A|}.
    \end{align*}
    This finishes the proof.
\end{proof}

\section{Proofs for Section~\ref{sec:infinite-horizon} (Infinite-horizon MDPs)}
\label{sec:appendix-infinite-horizon-MDP}
In this section, we first show the relationship between the path length of policies and the path length of occupancy measures. Next, we show the proofs of the reduction to switching-cost expert problem in Section~\ref{sec:reduction}. Then, we give the proofs of the dynamic regret of our algorithm in Section~\ref{sec:dynamic-infinite-horizon} and finally we present the proofs of the impossibility result for switching-cost expert problem in Section~\ref{sec:adaptive-infinite-horizon}.

\subsection{Path Length of Policies and Occupancy Measures}
\label{sec:appendix-infinite-horizon-occupancy-measure}
We introduce the relationship between the path length of policies and the path length of occupancy measures as follows.

\begin{myLemma}
    \label{lem:infinite-horizon-2}
    For any occupancy measure sequence $q_1, \ldots, q_T$ induced by the policy sequence $\pi_1, \ldots, \pi_T$, it holds that
    \begin{equation*}
        \sum_{t=2}^T \norm{q^{\pi_t} - q^{\pi_{t-1}}}_1 \leq (\tau+2) \sum_{t=2}^T \norm{\pi_t - \pi_{t-1}}_{1, \infty}.
    \end{equation*}
\end{myLemma}
\begin{proof}
    Consider any two policies $\pi$ and $\pi'$ with occupancy measure $q^{\pi}$ and $q^{\pi'}$, let $d^{\pi}(x) \triangleq \sum_{x,a} q^{\pi}(x,a), d^{\pi'}(x) \triangleq \sum_{x,a} q^{\pi'}(x,a), \forall x \in X$, we have
    \begin{align*}
        \norm{q^\pi - q^{\pi^\prime}}_1 & = \sum_{x,a} \abs{q^\pi(x,a) - q^{\pi^\prime}(x,a)}                                                                                                \\
                                        & = \sum_{x,a} \abs{d^\pi(x) \pi(a|x) - d^{\pi^\prime}(x) \pi^\prime (a | x)}                                                                        \\
                                        & \leq \sum_{x,a} \abs{d^\pi(x) \pi(a|x) - d^\pi(x) \pi^\prime(a|x)} + \sum_{x,a} \abs{d^\pi(x) \pi^\prime(a|x) - d^{\pi^\prime}(x) \pi^\prime(a|x)} \\
                                        & = \sum_x d^\pi(x) \sum_a \abs{\pi(a|x) - \pi^\prime(a|x)} + \sum_x \abs{d^\pi(x) - d^\prime(x)} \sum_a \pi^{\prime}(a|x)                           \\
                                        & \leq \norm{\pi-\pi^\prime}_{1, \infty} + \norm{d^\pi - d^{\pi^\prime}}_1                                                                           \\
                                        & \leq (\tau+2) \norm{\pi - \pi^\prime}_{1, \infty},
    \end{align*}
    where the first inequality holds due to the triangle inequality and the last inequality holds due to Lemma \ref{lem:infinite-horizon-1}. We finish the proof by summing the inequality over $T$.
\end{proof}

\subsection{Proof of Theorem~\ref{thm:reduction}}
\label{sec:appendix-proof-reduction}
To prove Theorem~\ref{thm:reduction}, we first introduce two lemmas which measure the difference between the sum of average losses and the actual losses of the learner and compared policies. Denote by $\rho_t^\pi$ the \emph{average loss per step} corresponding $\pi$: $ \rho_t^\pi \triangleq \lim_{T \rightarrow \infty} \frac{1}{T} \sum_{t=1}^T \E[\ell_t(x_t,a_t)|P,\pi] = \inner{q^\pi}{\ell_t}$ and the actual cumulative loss suffered by the learner $L_T \triangleq \E[\ell_t(x_t, \pi_t(x_t))|P,\pi]$, where the randomness is over the transition kernel and policy sequence $\pi_{1:T}$. Similarly, the actual cumulative loss suffered by the compared policy sequence $\pi_{1:T}^\C$ is $L_T^\C \triangleq \E[\ell_t(x_t, \pi_t^\C(x_t))|P,\pi]$. Let $d^\pi$ be the stationary state distribution, \ie $d^\pi(x) \triangleq \sum_a q^\pi(x,a), \forall x \in X$. Denote by $\mu_t = \mu_1 P^{\pi_1}\cdots P^{\pi_{t-1}}$ the state distribution after executing $\pi_1, \ldots, \pi_{t-1}$, where $\mu_1$ is the initial distribution, similarly, $\mu_t^\C = \mu_1 P^{\pi_1^\C}\cdots P^{\pi_{t-1}^\C}$.
\begin{myLemma}
    \label{lem:comparator-drift}
    For any compared policy sequence $\pi_1^\C, \ldots, \pi_T^\C$, it holds that
    $
        \sum_{t=1}^T \rho_t^{\pi_t^\C}- L_T^\C \leq (\tau+1)^2 P_T + 2(\tau +1).
    $
\end{myLemma}
\begin{proof}
    From the definition that $\mu_t^\C = \mu_1 P^{\pi_1^\C}\cdots P^{\pi_{t-1}^\C}$, we have
    \begin{align*}
        \sum_{t=1}^T \rho_t^{\pi_t^\C}- L_T^\C & =\sum_{t=1}^T \sum_{x}\left(d^{\pi_t^\C}(x) - \mu_t^\C(x)\right) \sum_a \pi_t^\C(a|x) \ell_t(x,a) \\
                                               & \leq \sumt \norm{d^{\pi_t^\C} - \mu_t^\C}_1                                                       \\
                                               & \leq 2(\tau +1) + (\tau+1) \sum_{t=2}^T \norm{d^{\pi_{t}^\C} - d^{\pi_{t-1}^\C}}_{1}              \\
                                               & \leq 2(\tau +1) + (\tau+1)^2 \sum_{t=2}^T \norm{\pi_t^\C-\pi_{t-1}^\C}_{1, \infty},
    \end{align*}
    where the second inequality holds due to Lemma \ref{lem:infinite-horizon-3} and the last inequality holds due to Lemma \ref{lem:infinite-horizon-1}.
\end{proof}
\begin{myLemma}
    \label{lem:learned-drift}
    For any occupancy measure sequence $q^{\pi_1}, \ldots, q^{\pi_T}$ returned by the learner, it holds that $L_T -  \sum_{t=1}^T \rho_t^{\pi_t}   \leq(\tau+1) \sum_{t=2}^T \norm{q^{\pi_t} - q^{\pi_{t-1}}}_1 + 2(\tau +1).$
\end{myLemma}
\begin{proof}
    From the definition that $\mu_t = \mu_1 P^{\pi_1} \cdots P^{\pi_{t-1}}$, we have
    \begin{align*}
        L_T -  \sum_{t=1}^T \rho_t^{\pi_t} & = \sum_{t=1}^T \sum_{x}\left(\mu_t(x) - d^{\pi_t}(x)\right) \sum_a \pi_t(a|x) \ell_t(x,a)\nonumber \\
                                           & \leq \sumt \norm{\mu_t - d^{\pi_t}}_1\nonumber                                                     \\
                                           & \leq 2(\tau +1) + (\tau+1) \sum_{t=2}^T \norm{d^{\pi_{t}} - d^{\pi_{t-1}}}_{1}  \nonumber          \\
                                           & = 2(\tau +1) + (\tau+1) \sum_{t=2}^T \sum_x \abs{\sum_a q^{\pi_t}(x,a) - q^{\pi_{t-1}}(x,a)}  \nonumber        \\
                                           & \leq 2(\tau +1) + (\tau+1) \sum_{t=2}^T \norm{q^{\pi_t}-q^{\pi_{t-1}}}_1,
    \end{align*}
    where the second inequality holds due to Lemma \ref{lem:infinite-horizon-3}.
\end{proof}
Then, we present the proof of Theorem~\ref{thm:reduction}.
\begin{proof}[Proof of Theorem~\ref{thm:reduction}]
    Note that the dynamic regret for infinite-horizon MDPs is defined as $\E[\DReg(\pi_{1:T}^\C)] = \E[ \sum_{t=1}^{T} \ell_t(x_t, \pi_t(x_t)) - \ell_t(x_t, \pi_t^\C(x_t))].$
    Then it can be written as
    \begin{align*}
        \E[{\DReg}_T(\pi_{1:T}^\C)] & = \E\left[ \sum_{t=1}^{T} \ell_t(x_t, \pi_t(x_t)) - \ell_t(x_t, \pi_t^\C(x_t))\right]                        \\
                                    & = L_T - \sumt \rho_t^{\pi_t} + \sumt (\rho_t^{\pi_t} - \rho_t^{\pi_t^\C}) + \sumt \rho_t^{\pi_t^\C} - L_T^\C \\
                                    & \leq \sumt \inner{q_t - \qtc}{\ell_t} + \sum_{t=2}^T \norm{q_t-q_{t-1}}_1 + (\tau+1)^2 P_T + 4(\tau +1),
    \end{align*}
    where the last inequality holds due to Lemma~\ref{lem:comparator-drift} and Lemma~\ref{lem:learned-drift} and the definition that $P_T = \sum_{t=2}^T \norm{\pi_t - \pi_t^\C}_{1, \infty}$.
\end{proof}

\subsection{Proof of Theorem~\ref{thm:total-regret-general}}
\label{appendix:proof-infinite-horizon}
\begin{proof}
    Similar to the proof in Appendix~\ref{sec:loop-free-ssp-proof-overall}, since the MDP is ergodic according to Definition~\ref{def:mixing}, we assume $T$ is large enough such that there at least exists a policy $\pi^u$ whose occupancy measure $q^u$ satisfies $q^{\pi^u} \in \Delta(M, \frac{1}{T})$, then define $u_t = (1-\frac{1}{T})\qtc + \frac{1}{T}q^{\pi^u} \in \Delta(M,\frac{1}{T^2})$, from the dynamic regret decomposition in~\eqref{eq:mdp-sc}, we have
    \begin{align}
             & \E[{\DReg}_T(\pi_{1:T}^\C)]                                                                                                   \\
        \leq & \sum_{t=1}^T \inner{q_t-\qtc}{\ell_t}+ (\tau+1) \sum_{t=2}^T \norm{q_t - q_{t-1}}_1
        + (\tau+1)^2 P_T  + 4(\tau +1) \nonumber                                                                                             \\
        =    & \sum_{t=1}^T \inner{q_t-u_t}{\ell_t} + \frac{1}{T}\sumt \inner{q^{\pi^u}-q^{\pi_t^\C}}{\ell_t}+ (\tau+1) \sum_{t=2}^T \norm{q_t - q_{t-1}}_1
        + (\tau+1)^2 P_T  + 4(\tau +1) \nonumber                                                                                             \\
        =    & \underbrace{\sum_{t=1}^T \inner{q_t-u_t}{\ell_t} + (\tau+1) \sum_{t=2}^T \norm{q_t - q_{t-1}}_1}_{\term{a}}
        + (\tau+1)^2 P_T  + 4(\tau +1) +2 \label{eq:decomposition-infinite-horizon},
    \end{align}
    where the first equality follows from the definition that $u_t = (1-\frac{1}{T})\qtc + \frac{1}{T}q^{\pi^u}$. We only need to consider $\term{a}$ since the remaining terms are not related to the algorithm. From the decomposition in~\eqref{eq:sc-decomposition}, $\term{a}$ can be written as
    \begin{align*}
        \underbrace{\sum_{t=1}^T \inner{p_t}{h_t} - \sum_{t=1}^T h_{t,i}}_{\meta}  + \underbrace{(\tau+1) \sum_{t=2}^T \norm{p_t - p_{t-1}}_1}_{\metaswitching}  + \underbrace{\sum_{t=1}^T \inner{q_{t,i}-u_t}{\ell_t}}_{\base} + \underbrace{(\tau+1) \sum_{t=2}^T \norm{q_{t,i}-q_{t-1,i}}_1}_{\baseswitching},
    \end{align*}
    which hold for any index $i$. Next, we bound these terms separately.

    \paragraph{Upper bound of base-regret.} From the standard analysis of OMD similar to that in~\eqref{eq:omd-analysis} and \eqref{eq:bregman-difference-1}, we have
    \begin{equation}
        \label{eq:infinite-base-regret}
        \begin{aligned}
            \sumt\inner{q_{t,i} - u_t}{\ell_t}
             & \leq \eta_i \sumt \sum_{x,a} q_{t,i}(x,a) \ell_t^2(x,a) + \frac{1}{\eta_i} \sumt \left(\Dp(u_t, q_{t,i}) - \Dp(u_t, q_{t+1,i})\right) \\
             & \leq \eta_i T + \frac{\log{|X||A|}}{\eta_i} + \frac{2\log{T}}{\eta_i}  \sum_{t=2}^T \norm{u_t - u_{t-1}}_{1}                          \\
             & \leq \eta_i T + \frac{\log{|X||A|} + 2\Pb_T\log{T}}{\eta_i},
        \end{aligned}
    \end{equation}
    which the last inequality holds due to $\sum_{t=2}^T \norm{u_t - u_{t-1}}_{1} \leq \sum_{t=2}^T \norm{\qtc - q^{\pi_{t-1}^\C}}_{1} = \Pb_T$.

    \paragraph{Upper bound of meta-regret.} From the definition that $h_{t,i} = \inner{q_{t,i}}{\ell_t} + (\tau+1) \norm{q_{t,i}-q_{t-1,i}}_1, \forall i \in [N]$, we have $0 \leq h_{t,i} \leq 1 + 2(\tau+1) = 2\tau+3, \forall i \in [N]$. By the standard analysis of Hedge similar to the analysis of meta-regret in Appendix~\ref{sec:loop-free-ssp-proof-overall}, we have
    \begin{equation}
        \label{eq:infinite-meta-regret}
        \sum_{t=1}^T \inner{p_t}{h_t} - \sum_{t=1}^T h_{t,i} \leq \varepsilon \sumt \sumi p_{t,i} h_{t,i}^2 + \frac{\log{N}}{\varepsilon} \leq \varepsilon (2\tau+3)^2 T +\frac{\log{N}}{\varepsilon}.
    \end{equation}
    \paragraph{Upper bound of switching-cost.} From Lemma \ref{lem:omd}, we have
    \begin{equation}
        \label{eq:stability}
        \norm{q_{t,i}-q_{t-1,i}}_1 \leq \eta_i  \norm{\ell_t}_\infty \leq \eta_i, ~~ \mbox{and} ~~ \norm{p_t - p_{t-1}}_1 \leq \varepsilon \norm{h_t}_\infty \leq \varepsilon (2\tau+3), \forall t \geq 2.
    \end{equation}
    \paragraph{Upper bound of overall dynamic regret.} Combining~\eqref{eq:infinite-base-regret},~\eqref{eq:infinite-meta-regret} and~\eqref{eq:stability}, we obtain
    \begin{align*}
        \term{a} & \leq \eta_i (\tau+2) T + \frac{\log{(|X||A|)} + 2\Pb_T\log{T}}{\eta_i}  +  \varepsilon (2\tau+3)^2 T +\frac{\log{N}}{\varepsilon}+\varepsilon (2\tau+3)(\tau+1)T \\
                 & \leq\eta_i (\tau+2) T + \frac{\log{(|X||A|)} + 2\Pb_T\log{T}}{\eta_i}  +  2\varepsilon (2\tau+3)^2 T +\frac{\log{N}}{\varepsilon}.
    \end{align*}
    From the configuration that $\varepsilon = \sqrt{\frac{\log{N}}{2T(2\tau+3)^2}}$, we obtain
    \begin{align*}
        \term{a} \leq \eta_i (\tau+2) T + \frac{\log{(|X||A|)} + 2\Pb_T\log{T}}{\eta_i}  + (4\tau+6) \sqrt{2 T \log{N}}.
    \end{align*}
    It is clear that the the optimal step size is $\eta^* = \sqrt{\frac{\log{|X||A|} + 2 \Pb_T \log{T}}{(\tau+2)T}}$. From the definition of $\Pb_T$, we have $0 \leq \Pb_T = \sum_{t=2}^T \norm{\qtc - q^{\pi_{t-1}^\C}}_{1} \leq 2T$, we ensure the possible range of $\eta^*$ is
    \begin{equation*}
        \sqrt{\frac{\log{|X||A|}}{(\tau+2)T}} \leq \eta^* \leq \sqrt{\frac{\log{(|X||A|)}+4T\log{T}}{(\tau+2)T}}.
    \end{equation*}
    Set the step size pool as $\H = \Big\{2^{i-1}\sqrt{\frac{\log |X||A|}{T}} \mid i \in [N]\Big\}$ where $N=\ceil{\frac{1}{2}\log(1+\frac{4T \log T}{\log|X||A|})}+1$. We can verify that
    \begin{equation*}
        \eta_1 = \sqrt{\frac{\log{|X||A|}}{(\tau+2)T}}\leq \eta^*, \mbox{ and } \eta_N \geq \sqrt{\frac{\log{(|X||A|)}+4T\log{T}}{(\tau+2)T}}= \eta^*.
    \end{equation*}
    Thus, we confirm that there exists a base-learner whose step size satisfies $\eta_{i^*} \leq \eta^* \leq \eta_{i^* +1} = 2\eta_{i^*}$. Then, we choose $i^*$ to analysis to obtain a sharp bound. Thus $\term{a}$ is bounded by
    \begin{align}
        \term{a} & \leq \eta_{i^*} (\tau+2) T + \frac{\log{(|X||A|)} + 2\Pb_T\log{T}}{\eta_{i^*}}  + (4\tau+6) \sqrt{2 T \log{N}} \nonumber \\
                 & \leq \eta^* (\tau+2) T + \frac{2(\log{(|X||A|)} + 2\Pb_T\log{T})}{\eta^*}  + (4\tau+6) \sqrt{2 T \log{N}} \nonumber      \\
                 & \leq 3\sqrt{(\tau+2)T (\log{|X||A| + 2\Pb_T \log{T}})}+ (4\tau+6) \sqrt{2 T \log{N}}\label{eq:term-a-infinite-horizon},
    \end{align}
    where the second inequality holds due to the condition $\eta_{i^*} \leq \eta^* \leq \eta_{i^* +1} = 2\eta_{i^*}$ and the last inequality holds by substituting the optimal step size $\eta^* = \sqrt{\frac{\log{|X||A|} + 2 \Pb_T \log{T}}{(\tau+2)T}}$.
    Therefore, combining~\eqref{eq:decomposition-infinite-horizon} and~\eqref{eq:term-a-infinite-horizon}, we obtain
    \begin{align*}
             & \E[{\DReg}_T(\pi_{1:T}^\C)]                                                                                        \\
        \leq {} & \term{a} + (\tau+1)^2 P_T  + 4\tau+6                                                                               \\
        \leq {} & 3\sqrt{(\tau+2)T (\log{|X||A| + 2\Pb_T \log{T}})}+ (4\tau+6) \sqrt{2 T \log{N}} +  (\tau+1)^2 P_T  + 4\tau+6       \\
        \leq {} & 3\sqrt{(\tau+2)T (\log{|X||A| + 2(\tau+2)P_T \log{T}})}+ (4\tau+6) \sqrt{2 T \log{N}} +  (\tau+1)^2 P_T  + 4\tau+6 \\
        \leq {} & \O(\sqrt{\tau T(\log|X||A| + \tau P_T \log T)} + \tau^2 P_T),
    \end{align*}
    where the third inequality uses $\Pb_T \leq (\tau+2) P_T$ from Lemma~\ref{lem:infinite-horizon-2}. This finishes the proof.
\end{proof}

\subsection{Proof of Theorem~\ref{thm:impossibility}}
\begin{proof}
    The proof is inspired by the proof of Theorem 13 in~\citet{COLT'18:switch-cost}. Let $c \geq 0$ be the constant in Lemma \ref{lem:concentration}. First, we divide the $T$ iterations into $E$ epochs, each with a uniform length $\frac{T}{E}$. For each epoch $e \in [E]$, we assign each expert $i \in [N]$ a loss of $\ell_{e}(i) =0$ with probability $\frac{1}{2}$ and $\ell_{e}(i) =1$ with probability $\frac{1}{2}$ for each interaction in that epoch. Thus by Lemma \ref{lem:concentration}, we have
    \begin{equation}
        \label{eq:sc-1}
        \E \left[\min_{i \in [N]} \sum_{t=1}^T \ell_t(i)\right] \leq \frac{T}{E}\left(\frac{E}{2} - c \sqrt{E\log N}\right) = \frac{T}{2} - cT \sqrt{\frac{\log n}{E}}.
    \end{equation}

    Now let us consider the expected loss of any algorithm $\mathcal{A}$ whose switching budget is at most $B$, \ie $\sum_{t=2}^T \norm{p_t - p_{t-1}}_1 \leq B$. It is easy to verify that the following strategy is optimal: for each epoch, randomly assign the weight on the experts in the first iteration (any strategy in the first iteration is the same since the losses are totally random), then, convert the weight on the bad experts (with loss 1) to the good experts (with loss 0) if we have switching budget remaining else not move. Denote by $W$ the total weight assigned to bad experts in the first iteration of $E$ epochs. Thus $\E[W] = \frac{E}{2}$. Then the random variable $\min(W, \frac{B}{2})$ is equal to the weight that algorithm $\mathcal{A}$ convert from bad experts to good experts ($\frac{B}{2}$ is due to that converting weight $\frac{B}{2}$ will suffer $B$ switching cost). Then, the cumulative loss of algorithm $\mathcal{A}$ is lower bounded as follows:
    \begin{equation}
        \label{eq:sc-2}
        \begin{aligned}
            \mathbb{E}[\mbox{cumulative loss of } \mathcal{A}]
             & = \mathbb{E}[\mbox{weight on bad experts of } \mathcal{A}]                                        \\
             & = \mathbb{E}\left[1 \cdot \min (W, \frac{B}{2})+\frac{T}{E} \cdot(W-\min (W, \frac{B}{2}))\right] \\
             & \geq \frac{T}{E} \mathbb{E}[W-\frac{B}{2}]                                                        \\
             & \geq \frac{T}{E}\left(\frac{E}{2}-\frac{B}{2}\right)                                              \\
             & =\frac{T}{2}-\frac{T B}{2E}
        \end{aligned}
    \end{equation}

    By setting $E = \frac{4B^2}{c^2 \log{N}}$ and combining \eqref{eq:sc-1} and \eqref{eq:sc-2}, we have
    \begin{equation}
        \label{eq:lower-bound}
        {\Reg}_T(\mathcal{A}) \geq cT \sqrt{\frac{\log N}{E}} -\frac{T B}{2E} = \frac{3c^{2} T \log N}{8B},
    \end{equation}
    we ensure that any algorithm $\mathcal{A}$ with switching budget $B$ suffers expected static regret at least $\frac{3c^{2} T \log N}{8B} = \Omega(\frac{T \log{N}}{B})$.

    Consider the case that $B = \Theta(T^{\frac{1}{3}})$ and $E = \frac{4B^2}{c^2 \log{N}} = \Theta(T^{\frac{2}{3}})$. In this case, $\sum_{t=2}^T \norm{\ell_t-\ell_{t-1}}_\infty^2 = E = \Theta(T^{\frac{2}{3}})$, in order to achieving $\O(\sqrt{1+\sum_{t=2}^T \norm{\ell_t-\ell_{t-1}}_\infty^2})$ static regret with switching cost, the algorithm $\mathcal{A}$ is required to keep the static regret and switching cost both not more than $\O(T^{\frac{1}{3}})$. However, from \eqref{eq:lower-bound}, we know that any algorithm with $B = \O(T^{\frac{1}{3}})$ switching budget will suffer $\Omega(\frac{T}{B}) = \Omega(T^{\frac{2}{3}})$ static regret. This completes the proof.
\end{proof}

\subsection{Useful Lemmas}
\label{sec:appendix-lemma-infinite-horizon}
In this part, we present some useful lemmas in infinite-horizon MDPs. Denote by $d^\pi$ the stationary state distribution induced by policy $\pi$ under transition kernel $P$, \ie $d^\pi(x) \triangleq \sum_a q^\pi(x,a), \forall x \in X$. Then we have the following useful lemmas which show the relationships between policy discrepancy and distribution discrepancy.
\begin{myLemma}[Lemma 4 of~\citet{IEEE'Control:bandit-MDP-Neu}]
    \label{lem:infinite-horizon-1}
    For any two policies $\pi$ and $\pi^\prime$, it holds that
    \begin{equation*}
        \norm{d^\pi - d^{\pi^\prime}}_1 \leq (\tau+1)\norm{\pi - \pi^\prime}_{1, \infty}.
    \end{equation*}
\end{myLemma}
\begin{proof}
    The statement follows from solving
    \begin{equation*}
        \norm{d^\pi - d^{\pi^\prime}}_1 \leq \norm{d^\pi P^\pi - d^{\pi^\prime} P^\pi}_1 + \norm{d^{\pi^\prime} P^\pi - d^{\pi^\prime} P^{\pi^\prime}}_1 \leq e^{-1/\tau} \norm{d^\pi - d^{\pi^\prime}}_1 + \norm{\pi - \pi^\prime}_{1,\infty}
    \end{equation*}
    for $\norm{d^\pi - d^{\pi^\prime}}_1$ and using $\frac{1}{1-e^{-1/\tau}} \leq \tau + 1$.
\end{proof}

\begin{myLemma}
    \label{lem:infinite-horizon-3}
    Consider the distribution $\mu_t = \mu_1 P^{\pi_1}\cdot P^{\pi_{t-1}}$, where $\mu_1$ is any distribution over $X$ and $\pi_1, \ldots, \pi_t$ is any policy sequence, it holds that
    \begin{equation*}
        \sum_{t=1}^T \norm{\mu_t - d^{\pi_t}}_1 \leq 2(\tau +1) + (\tau+1) \sum_{t=2}^T \norm{d^{\pi_{t}} - d^{\pi_{t-1}}}_{1}.
    \end{equation*}
\end{myLemma}
\begin{proof}
    It is trivial for $t=1$ since $\norm{\mu_1 - d^{\pi_1}}_1 \leq 2$. Thus, in what follows we only consider the case that $t \geq 2$.By the triangle inequality, we have
    \begin{align*}
        \norm{\mu_t - d^{\pi_t}}_1 & \leq \norm{\mu_t - d^{\pi_{t-1}}}_1 + \norm{d^{\pi_{t-1}} - d^{\pi_t}}_1                                                                                 \\
                                   & = \norm{\mu_{t-1}P^{\pi_{t-1}} - d^{\pi_{t-1}}P^{\pi_{t-1}}}_1 + \norm{d^{\pi_{t-1}} - d^{\pi_t}}_1                                                      \\
                                   & \leq e^{-1/\tau}\norm{\mu_{t-1} - d^{\pi_{t-1}}}_1 + \norm{d^{\pi_{t-1}} - d^{\pi_t}}_1                                                                  \\
                                   & \leq e^{-1/\tau}\left(e^{-1/\tau}\norm{\mu_{t-2} - d^{\pi_{t-2}}}_1 + \norm{d^{\pi_{t-2}} - d^{\pi_{t-1}}}_1\right) + \norm{d^{\pi_{t-1}} - d^{\pi_t}}_1 \\
                                   & \leq \cdots \leq e^{-(t-1)/\tau}\norm{\mu_1 - d^{\pi_1}}_1 + \sum_{n=0}^{t-2} e^{-n/\tau} \norm{d^{\pi_{t-n}} - d^{\pi_{t-n-1}}}_{1}                     \\
                                   & \leq 2e^{-(t-1)/\tau} + \sum_{n=0}^{t-2} e^{-n/\tau} \norm{d^{\pi_{t-n}} - d^{\pi_{t-n-1}}}_{1}.
    \end{align*}
    where the first equality holds since $d^{\pi_{t-1}}$ is the stationary distribution of $\pi_{t-1}$, \ie $d^{\pi_{t-1}} = d^{\pi_{t-1}} P^{\pi_{t-1}}$ and the second inequality holds due to Definition~\ref{def:mixing}.
    Summing above inequality over $t$, we have
    \begin{align*}
            \sum_{t=1}^T \norm{\mu_t - d^{\pi_t}}_1 & \leq 2 + 2\sum_{t=2}^T e^{-(t-1)/\tau} + \sum_{t=2}^T \sum_{n=0}^{t-2} e^{-n/\tau} \norm{d^{\pi_{t-n}} - d^{\pi_{t-n-1}}}_{1} \\
                                                    & \leq 2(\tau +1) +\sum_{t=2}^T (\sum_{n=0}^{T-t} e^{-n/\tau})\norm{d^{\pi_{t}} - d^{\pi_{t-1}}}_{1}                            \\
                                                    & \leq 2(\tau +1) + (\tau+1) \sum_{t=2}^T \norm{d^{\pi_{t}} - d^{\pi_{t-1}}}_{1}.
    \end{align*}
    This finishes the proof.
\end{proof}

We finally introduce the following concentration lemma, which is useful in proving the impossibility result of the switching-cost expert problem.
\begin{myLemma}[Lemma 6 of \citet{JACM'97:doubling-trick}]
    \label{lem:concentration}
    Denote $\sigma_t(i), i\in[N], t\in[T]$ the \iid random variables which take $0$ with probability $\frac{1}{2}$ and take $1$ with probability $\frac{1}{2}$. There exist a constant $c \geq 0$ such that for all positive integer $N$ and $T$,
    \begin{equation*}
        \E \left[\min_{i \in [N]} \sumt \sigma_i(t)\right] \leq \frac{T}{2} - c \sqrt{T\log N}.
    \end{equation*}
\end{myLemma}

\end{document}